\documentclass[12pt]{berkeley}
\usepackage[numbers, sort&compress]{natbib}
\usepackage{amsfonts}
\usepackage{amsmath, bm}
\usepackage{amssymb}
\usepackage{mathtools}
\usepackage{amsthm}
\newtheorem{theorem}{Theorem}
\newtheorem{proposition}{Proposition}
\newtheorem{assumption}{Assumption}
\newtheorem{definition}{Definition}

\newtheorem{lemma}{Lemma}
\newtheorem{remark}{Remark}

\makeatletter
\renewcommand\AB@affilsepx{\\\protect\Affilfont}
\makeatother

\usepackage[nameinlink]{cleveref}
\crefname{theorem}{theorem}{theorems}
\Crefname{theorem}{Theorem}{Theorems}
\crefname{proposition}{proposition}{propositions}
\Crefname{proposition}{Proposition}{Propositions}
\crefname{lemma}{lemma}{lemmas}
\Crefname{lemma}{Lemma}{Lemmas}
\crefname{assumption}{assumption}{assumptions}
\Crefname{assumption}{Assumption}{Assumptions}
\crefname{corollary}{corollary}{corollaries}
\Crefname{corollary}{Corollary}{Corollaries}
\crefname{figure}{Fig.}{Figs.}
\Crefname{figure}{Figure}{Figures}
\crefname{appendix}{appendix}{appendices}
\Crefname{appendix}{Appendix}{Appendices}
\crefname{remark}{remark}{remarks}
\Crefname{remark}{Remark}{Remarks}
\crefname{definition}{definition}{definitions}
\Crefname{definition}{Definition}{Definitions}

\title{Terminally constrained flow-based generative models from an optimal control perspective}

\newcommand{\runningtitle}{Terminally constrained flow-based models}

\author[1,2]{Weiguo Gao}
\author[1]{Ming Li}
\author[3]{Qianxiao Li}

\affil[ ]{\hspace{-1mm}\texttt{wggao@fudan.edu.cn}, \texttt{mingli23@m.fudan.edu.cn}, \texttt{qianxiao@nus.edu.sg}\vspace{2mm}}

\affil[1]{School of Mathematical Sciences, Fudan University, Shanghai 200433, China}
\affil[2]{Shanghai Key Laboratory of Contemporary Applied Mathematics, Shanghai 200433, China}
\affil[3]{Department of Mathematics and Institute for Functional Intelligent Materials, National University of Singapore, Singapore}

\date{}

\begin{document}

\begin{abstract}
We address the problem of sampling from terminally constrained distributions with pre-trained flow-based generative models through an optimal control formulation.
Theoretically, we characterize the value function by a Hamilton--Jacobi--Bellman equation and derive the optimal feedback control as the minimizer of the associated Hamiltonian. We show that as the control penalty increases, the controlled process recovers the reference distribution, while as the penalty vanishes, the terminal law converges to a generalized Wasserstein projection onto the constraint manifold. 
Algorithmically, we introduce \emph{Terminal Optimal Control with Flow-based models} (TOCFlow), a geometry-aware sampling-time guidance method for pre-trained flows. Solving the control problem in a terminal co-moving frame that tracks reference trajectories yields a closed-form scalar damping factor along the Riemannian gradient, capturing second-order curvature effects without matrix inversions. TOCFlow therefore matches the geometric consistency of Gauss--Newton updates at the computational cost of standard gradient guidance. 
We evaluate TOCFlow on three high-dimensional scientific tasks spanning equality, inequality, and global statistical constraints, namely Darcy flow, constrained trajectory planning, and turbulence snapshot generation with Kolmogorov spectral scaling. Across all settings, TOCFlow improves constraint satisfaction over Euclidean guidance and projection baselines while preserving the reference model's generative quality.
\end{abstract}

\keywords{Optimal control \(\cdot\) Generative modeling \(\cdot\) Terminal constraints \(\cdot\) Hamilton--Jacobi--Bellman equation \(\cdot\) Constrained sampling}

\subjclass{34H05 \(\cdot\) 37N35 \(\cdot\) 68T07}

\maketitle

\section{Introduction}
\label{sec:introduction}

Progress in critical scientific areas such as partial differential equation (PDE) solving, trajectory planning, and turbulent flow simulation depends on the ability to model complex systems subject to physical and geometric \emph{constraints}. While traditional computational methods provide robust guarantees for constraint enforcement, they often face significant limitations in high-dimensional settings. Numerical solvers become prohibitively expensive at the resolution required for accurate turbulence simulations~\citep{moin1998direct,pope2001turbulent,rebollo2014mathematical}, while optimization and sampling algorithms scale poorly with the dimensionality inherent in trajectory planning~\citep{karaman2011sampling,lavalle2006planning} and inverse problems~\citep{cotter2013mcmc,martin2012stochastic}. Approximation schemes and reduced-order models improve tractability but frequently sacrifice the accuracy or generality needed for reliable scientific inference~\citep{benner2015survey,quarteroni2014reduced}. These challenges highlight the need for new paradigms capable of combining statistical fidelity with computational efficiency without compromising the strict satisfaction of constraints.

Machine learning has responded to these computational bottlenecks through deep generative models, which approximate high-dimensional distributions directly from observational data without relying on expensive physical simulations. Within this domain, flow-based models such as the framework of Flow Matching~\citep{albergo2023stochastic,lipman2023flow,liu2023flow} have emerged as a particularly effective paradigm for scientific tasks~\citep{baldan2025flow,bastek2025physics,cheng2025gradient,corso2023diffdock,huang2024diffusionpde,joung2025electron,jumper2021highly,landry2025generating,nam2024flow,utkarsh2025physics,utz2025climate,yuan2025pirf,zhang2025physics,zhao2025generative}. These methods construct a continuous velocity field that transforms a simple source distribution \(\rho_0\) (typically a standard Gaussian) into the target distribution \(\rho_1\) by integrating the ordinary differential equation (ODE)
\begin{equation}
\dot{\bm{X}}_t = \bm{b}(\bm{X}_t, t), \quad \bm X_0\sim\rho_0
\label{eq:reference_flow}
\end{equation}
over the interval \([0,1]\). To learn the velocity field \(\bm{b}\), Flow Matching minimizes the regression objective
\begin{equation}
\mathbb{E}_{t, \bm{X}_0, \bm{X}_1}[ \| \bm{b}(\bm{X}_t, t) - (\bm{X}_1 - \bm{X}_0) \|_2^2 ],
\end{equation}
where the expectation is taken over uniformly distributed time \(t \sim \mathcal{U}([0,1])\) and sample pairs \((\bm{X}_0, \bm{X}_1)\sim\rho_0\times\rho_1\) drawn from the independent coupling of the source and target distributions. Here, \(\bm{X}_1 - \bm{X}_0\) is the instantaneous velocity (the derivative with respect to the time variable \(t\)) of the linear interpolant~\citep{albergo2023stochastic} \(\bm{X}_t = (1-t)\bm{X}_0 + t\bm{X}_1\) connecting noise \(\bm X_0\sim \rho_0\) to data \(\bm X_1\sim \rho_1\).

However, the purely data-driven nature of these models introduces a fundamental limitation. Relying solely on finite samples drawn from the source and target distributions, Flow Matching captures statistical patterns but remains agnostic to the strict constraints governing the system. This lack of structural awareness becomes problematic in two distinct ways. First, there is a fundamental issue of generalization. Even if all training samples perfectly satisfy the constraints, the learned velocity field is only an approximation obtained from finite data and a finite-capacity neural parametrization. As discussed by Gao and Li~\citep{gao2024how}, estimation and optimization errors can accumulate along the trajectory, causing samples to drift away from the constraint-satisfying set. Second, there is an issue of adaptability. Constraints often vary across different instances of the same task, such as a robotic planner encountering new obstacles or a PDE solver facing shifting boundary conditions. In such dynamic environments, a standard model trained to approximate a single fixed distribution cannot satisfy these new constraints without computationally expensive retraining.

Before describing our solution to these limitations, we first formalize the problem setting. To resolve the adaptability challenge without incurring the prohibitive cost of retraining for every new constraint, we adopt a ``sampling-time'' strategy. We assume access to a reference velocity field \(\bm{b}^\star(\bm x, t)\) that transports \(\rho_0\) to \(\rho_1\). It captures the general data distribution but may not satisfy specific structural constraints. In many scientific applications, the intermediate trajectory (i.e., \(t \in [0,1)\)) serves primarily as a computational transport mechanism. Strict validity is required only at the conclusion of the process (i.e., \(t=1\)). Consequently, we model the structural requirements as a \emph{terminal constraint}. We define the constraint manifold as \(\mathcal{M} = \{ \bm{x} \in \mathbb{R}^d \colon\bm{h}(\bm{x}) = \bm{0} \}\), where \(\bm h\colon\mathbb{R}^d\to\mathbb{R}^r\) is a twice continuously differentiable function. We assume that \(\bm 0\) is a regular value of \(\bm h\), meaning that the Jacobian matrix \(J_{\bm h}(\bm x)\) has full rank \(r\) for all \(\bm h(\bm x)=\bm 0\) (i.e., \(\bm x\in\mathcal{M}\)). Our objective is therefore to synthesize a controlled velocity field \(\bm{u}(\bm x, t) \coloneqq \bm{b}^\star(\bm x, t) + \bm{a}_t\) that perturbs the reference flow~\eqref{eq:reference_flow} to enforce the terminal constraint.

We determine this controlled velocity field by embedding the generative process within an optimal control framework. Let us define the optimal control problem as the minimization of the expected total cost, where the expectation is taken with respect to the randomness of \(\bm X_0\sim\rho_0\):
\begin{equation}
\mathcal{J}\coloneqq\inf_{\bm{a}} \mathbb{E} \Bigl[ H(\bm{X}_1)+\int_0^1 \frac{\lambda_t}{2} \| \bm{a}_t \|_2^2 \mathrm{d}t \Bigr].
\end{equation}
Here, \(H(\bm{x}) \coloneqq \frac{1}{2}\|\bm{h}(\bm{x})\|^2\) is the terminal cost, \(\lambda_t\) is a time-dependent weight, and \(\bm X_t\) evolves according to the controlled dynamics
\begin{equation}
\dot{\bm X}_t = \bm b(\bm X_t, t) + \bm a_t, \quad \bm X_0\sim \rho_0, \quad t\in[0, 1].
\label{eq:intro_controlled_dynamics}
\end{equation}
We optimize over admissible control processes \(\bm a=(\bm a_t)_{t\in[0,1]}\in\mathcal A\), meaning measurable controls that are essentially bounded in time, \(\operatorname*{ess\,sup}_{t\in[0,1]}\|\bm a_t\|_2<+\infty\), and for which the controlled dynamics~\eqref{eq:intro_controlled_dynamics} admit a unique weak solution for \(\rho_0\)-almost every initial condition. See~\Cref{def:admissible_controls} for a formal definition. To solve this problem, we introduce the value function \(\mathcal{J}(\bm{x}, t)\), which represents the minimal expected cost required to reach the terminal time \(t=1\) conditioned on the system being at state \(\bm{x}\) at time \(t\). Formally, it is defined as the infimum of the cost-to-go:
\begin{equation}
\mathcal{J}(\bm{x}, t) \coloneqq \inf_{\bm{a}} \mathbb{E} \Bigl[ H(\bm{X}_1)+\int_t^1 \frac{\lambda_s}{2} \| \bm{a}_s \|_2^2 \mathrm{d}s \Big| \bm{X}_t = \bm{x} \Bigr].
\end{equation}
We show in~\Cref{prop:hjb_feedback} that \(\mathcal{J}(\bm x, t)\) satisfies the Hamilton--Jacobi--Bellman (HJB) equation
\begin{equation}
-\partial_t \mathcal{J}(\bm{x},t) - \nabla_{\bm{x}} \mathcal{J}(\bm{x},t) \cdot \bm{b}(\bm{x}, t) + \frac{1}{2 \lambda_t} \| \nabla_{\bm{x}} \mathcal{J}(\bm{x},t) \|_2^2 = 0,\quad \mathcal{J}(\bm x, 1)=H(\bm x),
\end{equation}
and that the optimal feedback law is
\begin{equation}
\bm a^\star(\bm x, t) = -\lambda_t^{-1} \nabla_{\bm{x}} \mathcal{J}(\bm{x},t).
\end{equation}
In particular, along the optimal trajectory \((\bm X_t^\star)_{t\in[0,1]}\), the optimal control satisfies \(\bm a_t^\star = \bm a^\star(\bm X_t^\star,t)\). To simplify notation, we drop the superscript \(\star\) from \(\bm X_t\) and from any quantities evaluated along the optimal trajectory when no confusion arises.

Guided by this formulation we pursue three primary goals in this work.

\textbf{Our first goal is to characterize the optimal control and its associated transport costs.} We establish an equivalence between the Eulerian control energy and a weighted kinetic energy defined in the initial co-moving frame of the reference trajectories. This weighting is governed by the time-dependent metric \(\bm{C}_t(\bm{z}) \coloneqq D\Phi_{0 \to t}^{\bm b^\star}(\bm{z})^\top D\Phi_{0 \to t}^{\bm b^\star}(\bm{z})\). Here, \(\Phi_{0 \to t}^{\bm b^\star}\) denotes the flow map generated by the reference velocity field \(\bm{b}^\star\) transporting a state from initial time \(0\) to current time \(t\) and \(D\Phi_{0 \to t}^{\bm b^\star}\) denotes its spatial Jacobian matrix. By exploiting the spectral properties of this deformation tensor, which tracks the cumulative expansion and contraction of the flow, we derive rigorous upper and lower bounds on the control energy in terms of the squared \(2\)-Wasserstein distance \(W_2^2(\bar\rho_1, \nu)\) between the uncontrolled reference terminal distribution \(\bar\rho_1\) and the target constrained distribution \(\nu\). Furthermore, we characterize the asymptotic behavior of the system in two distinct regimes. In the limit of infinite \(\lambda_t\), the controlled dynamics recover the reference flow, while in the limit of vanishing \(\lambda_t\), the terminal distribution converges to a generalized Wasserstein projection of \(\bar\rho_1\) onto the constraint manifold \(\mathcal{M}\).

\textbf{Our second goal is to translate these theoretical insights into practical and efficient algorithms for high-dimensional inference.} By analyzing the dynamics in a terminal co-moving frame, we show that the optimal control is governed by a geometric Hamilton--Jacobi equation defined by the time-dependent metric \(\bm{G}_t(\bm{y}) \coloneqq D\Phi^{\bm{b}^\star}_{t \to 1}(\bm x) D\Phi^{\bm{b}^\star }_{t \to 1}(\bm x)^\top|_{\bm x = \Phi^{\bm{b}^\star}_{1 \to t}(\bm y)}\). Intuitively, this metric quantifies the sensitivity of the terminal state to perturbations applied at time \(t\). To make this tractable, we employ a splitting strategy that ``fixes'' the geometry, approximating the metric \(\bm G_{t}(\bm y)\) as a constant matrix \(\bm G\) over the integration step. This reduction allows us to invoke the Hopf--Lax formula, boiling the problem down to the proximal optimization problem
\begin{equation}
\bm y^+(\bm y) \coloneqq \underset{z\in\mathbb{R}^{d}}{\arg\min} \Bigl\{ H(\bm z) + \frac{1}{2s} \|\bm z-\bm y\|_{{\bm G}^{-1}}^{2} \Bigr\}.
\end{equation}
We analyze generic schemes to solve this proximal problem. We first explore a first-order Euclidean approximation (Gradient Descent scheme, GD), which assumes a flat geometry (\(\bm G \approx \bm I\)). While efficient, it fails to account for the spatial stretching induced by the flow. Conversely, a second-order geometric approximation (Gauss--Newton scheme, GN) respects the Riemannian metric but incurs a high computational cost due to the need for matrix inversions in high-dimensional constraint spaces. To resolve this trade-off, we introduce \emph{Terminal Optimal Control with Flow-based models} (TOCFlow). By restricting the optimization to the one-dimensional subspace aligned with the Riemannian gradient, TOCFlow reduces the complex vector optimization to a scalar line search with a closed-form solution. This approach achieves geometric consistency with a computational cost comparable to the GD scheme, independent of the number of constraints.

\textbf{Our third goal is to verify the efficacy of TOCFlow through exact analysis and representative scientific applications.} We use a linear-quadratic Gaussian model which admits analytical solutions to rigorously validate our framework. It quantitatively confirms the asymptotic behavior of the system under infinite and vanishing penalties, and it enables an exact comparison of approximation schemes, demonstrating that TOCFlow and GN schemes coincide and approximate the ground truth optimal control more accurately than the GD scheme. We then test TOCFlow on three distinct high-dimensional generation tasks, each representing a unique class of constraints: (\romannumeral1) High-dimensional PDE constraints (Darcy flow). We generate consistent permeability and pressure fields governing flow in porous media. This task imposes high-dimensional equality constraints coupled through an elliptic differential operator. We demonstrate that while standard projection methods fail to converge due to the singular nature of the operator, TOCFlow effectively steers the generation to the solution manifold, reducing PDE error by an order of magnitude compared to baseline methods. (\romannumeral2) Geometric inequality constraints (trajectory planning). We synthesize robotic trajectories confined to safety corridors within a cluttered domain. This involves enforcing geometric inequality constraints. We show that TOCFlow eliminates the collision violations observed with standard gradient guidance and avoids the non-physical discontinuities (kinks) introduced by projection-based methods. (\romannumeral3) Global spectral constraints (turbulence snapshot). We generate \(2\)-dimensional turbulent flow fields that must adhere to Kolmogorov's power-law scaling~\citep{kolmogorov1941local}. This imposes a global constraint on the energy spectrum. We demonstrate that TOCFlow acts coherently across spatial scales to enforce the correct statistical physics, more accurately reconstructing the inertial range scaling.

\subsection{Related Work}
\label{subsec:related_work}

Approximating high-dimensional distributions is a central task in scientific computing, for which deep generative models have emerged as a powerful paradigm, enabling breakthroughs in domains ranging from PDE solving~\citep{baldan2025flow,bastek2025physics,cheng2025gradient,huang2024diffusionpde,utkarsh2025physics,yuan2025pirf,zhang2025physics} and molecule or material design~\citep{corso2023diffdock,joung2025electron,jumper2021highly,nam2024flow}, to emulating climate and weather fields~\citep{landry2025generating,utz2025climate,zhao2025generative}. Among these, flow-based generative models represent a particularly effective approach, synthesizing data by learning a continuous probability path that transforms a simple source distribution (typically a Gaussian) into the complex target distribution. While early efforts such as Monge--Amp\`ere flows~\citep{zhang2018monge} used optimal transport potentials to define this mapping, the modern landscape is dominated by two frameworks: diffusion models and Flow Matching models. Diffusion models~\citep{ho2020denoising,song2021score} generate data by reversing a stochastic noise process via stochastic differential equations (SDEs), a powerful but computationally expensive approach due to the complex, curved trajectories involved~\citep{gao2025toward}. Flow Matching models~\citep{lipman2023flow,liu2023flow} instead learn ODEs that transport noise to data along (conditionally) straight paths, significantly reducing sampling costs. Beyond efficiency, this framework offers a theoretical unification through the theory of stochastic interpolants~\citep{albergo2023stochastic}, which generalizes both stochastic diffusion and deterministic flow matching under a single mathematical formulation. Given these advantages, in this work, we focus on Flow Matching as a representative for flow-based models.

From a theoretical standpoint, the Schr\"odinger Bridge (SB) problem~\citep{leonard2014survey} constitutes the most rigorous framework for constrained generation, minimizing transport cost while strictly satisfying boundary marginals. However, solving the resulting system of coupled forward--backward PDEs is computationally prohibitive in high dimensions. Recent works such as OC-Flow~\citep{wang2025training} and FOCUS~\citep{bill2025optimal} therefore recast SB-style constrained transport as tractable optimal control problems and solve for the control via iterative numerical schemes. In a separate but related direction, optimal control has also been used to formalize deep learning and fine-tuning dynamics~\citep{e2019mean,li2018maximum,zhang2025weight}. In contrast to these learning-centric formulations, our focus is on controlling the generative transport itself. While existing SB-inspired control methods have demonstrated the practical flexibility to incorporate rewards or constraints, they largely emphasize algorithmic procedures rather than the transport geometry induced by terminal penalties. In contrast, our work formulates terminal constraint enforcement for flow-based generative models rigorously as a terminal optimal control problem. Using this framework, we derive analytical bounds on the control energy in terms of the Wasserstein distance and prove asymptotic convergence of the controlled dynamics in both the infinite-penalty and vanishing-penalty regimes.

From an algorithmic perspective, our proposed TOCFlow addresses the challenge of constraint enforcement within deep generative models, a landscape where existing strategies typically fall into three categories. The first is training-time regularization, where constraints like physical laws are embedded directly into the objective function~\citep{baldan2025flow,bastek2025physics}. While effective, this requires expensive retraining whenever constraints (e.g., boundary conditions) change. The second is fine-tuning, which adapts pre-trained weights to new observations~\citep{yuan2025pirf,zhang2025physics}, but incurs high computational costs. These limitations motivate the third category: dynamic guidance during sampling, where TOCFlow resides. Existing methods such as DiffusionPDE~\citep{huang2024diffusionpde} implement guidance through heuristic gradient nudges on the constraint violation. While this improves adherence, it often results in limited constraint satisfaction. Conversely, approaches like Physics-Constrained Flow Matching (PCFM)~\citep{utkarsh2025physics} and ECI~\citep{cheng2025gradient} enforce constraints by repeatedly projecting intermediate states onto the constraint manifold. These projections can introduce non-smooth trajectories and are difficult to apply when the manifold is defined only implicitly. In contrast, TOCFlow rigorously frames guidance as a terminal optimal control problem, deriving a geometry-aware correction that steers the flow within the ambient space. In doing so, TOCFlow ensures high-accuracy constraint satisfaction without compromising the smoothness of the generative process. 

A complementary viewpoint is conditional generation, where constraint enforcement corresponds to sampling from a distribution conditioned to lie on the constraint manifold \(\mathcal M\), analogous to conditioning in diffusion models via learned control modules such as ControlNet~\citep{zhang2023adding} and T2I-Adapter~\citep{mou2024t2i} or via sampling-time methods for inverse problems~\citep{chung2023diffusion}. In this conditional-generation view, TOCFlow is an sampling-time conditioning method and differs from heuristic guidance by deriving the correction from a rigorous terminal optimal control formulation.

\subsection{Organization}
\label{subsec:organization}

The remainder of this paper is organized as follows. \Cref{sec:background_and_preliminaries} introduces the necessary mathematical preliminaries, including the Wasserstein space and the Flow Matching framework. \Cref{sec:an_optimal_control_framework} establishes the theoretical foundation, formulating the terminal optimal control problem and characterizing the optimal feedback via the Hamilton--Jacobi--Bellman equation and its geometric properties in the initial co-moving frame. \Cref{sec:geometric_proximal_approximation_of_the_optimal_feedback} translates this theory into practice, deriving the TOCFlow algorithm through a proximal operator approximation and analyzing its complexity relative to Gradient Descent and Gauss--Newton baselines. \Cref{sec:representative_applications} validates the framework on three high-dimensional scientific tasks including Darcy flow, trajectory planning, and turbulence snapshot generation, demonstrating its superiority over baseline methods.
\section{Background and preliminaries}
\label{sec:background_and_preliminaries}

In this section, we establish the mathematical preliminaries required for our optimal control framework. We begin by defining the notational conventions and the necessary concepts from optimal transport and measure theory in~\Cref{subsec:notations,subsec:the_wasserstein_space_and_transport_of_measure}. We then review the formulation of flow-based generative models, with a specific focus on Flow Matching as the representative technique for learning the reference dynamics in~\Cref{subsec:flow_based_generative_models_and_continuous_transport,subsec:flow_matching_as_a_representative_technique}. Finally, we formalize the geometric structure of the terminal constraint manifold and the regularity assumptions governing the target distribution in~\Cref{subsec:the_terminal_constraint_manifold}.

\subsection{Notations}
\label{subsec:notations}

Throughout this paper, \(d\) denotes the state space dimension, while \(r\) denotes the dimension of the constraints. We denote vectors in \(\mathbb{R}^d\) by bold lowercase letters (e.g., \(\bm x\), \(\bm y\)), with the specific exception of the state process \(\bm{X}_t\), which is capitalized to indicate its stochastic nature. Matrices and tensor fields are denoted by bold uppercase letters (e.g., \(\bm{C}\), \(\bm G\), \(\bm I\)), where \(\bm I\) denotes the identity matrix. The Euclidean norm is denoted by \(\|\cdot\|_2\), and the standard Euclidean inner product is denoted by \(\langle \cdot, \cdot \rangle\). For a weighted squared norm induced by a positive definite matrix \(\bm{M}\succ\bm 0\), we use the notation:
\begin{equation}
\|\bm{x}\|_{\bm{M}}^2 \coloneqq \langle \bm{x}, \bm{M}\bm{x} \rangle.
\end{equation}
We assume sufficient smoothness for all functions unless stated otherwise; \(C^k(\mathbb{R}^d; \mathbb{R}^r)\) denotes the space of \(k\)-times continuously differentiable functions mapping \(\mathbb{R}^d\) to \(\mathbb{R}^r\). The gradient of a scalar function \(f\) with respect to the spatial variable \(\bm{x}\) is denoted by \(\nabla_{\bm{x}} f\), and the Jacobian matrix of a vector-valued function \(\bm{h}\) is denoted by \(D\bm{h}\) or \(\bm{J}_{\bm h}\). Time derivatives are denoted by \(\partial_t\), \(\mathrm{d}/\mathrm{d}t\) or by a dot over the variable, i.e., \(\dot{\bm{x}}\).

In the context of measure theory and optimal transport, \(\mathcal{P}_2(\mathbb{R}^d)\) denotes the space of Borel probability measures with finite second moments. We use \(\rho\) to denote probability measures that admit a density with respect to the Lebesgue measure, often identifying the measure with its density. Specifically, we denote the unconstrained reference terminal distribution by \(\bar{\rho}_1\) and the target constrained distribution by \(\nu\). We denote by \(L^2(\rho)\) the Hilbert space of vector fields square-integrable with respect to \(\rho\), equipped with the norm \(\|\bm{v}\|_{L^2(\rho)}^2 \coloneqq \int_{\mathbb{R}^d} \|\bm{v}(\bm{x})\|_2^2 \rho(\mathrm{d}\bm{x})\). The push-forward of a measure \(\rho\) by a measurable map \(T\) is denoted by \(T_{\#} \rho\). The \(2\)-Wasserstein distance between two probability measures \(\mu\) and \(\nu\) is denoted by \(W_2(\mu, \nu)\).

We describe the dynamics of flow-based models using the flow map notation \(\Phi_{s \to t}^{\bm{b}}\), which represents the diffeomorphism generated by integrating the time-dependent velocity field \(\bm{b}\) from time \(s\) to time \(t\). We distinguish between the reference velocity field, denoted by \(\bm{b}^\star\), and the controlled velocity field, denoted by \(\bm{u} = \bm{b}^\star + \bm{a}\), where \(\bm{a}\) represents the control. We frequently use the subscript notation to indicate time dependence for stochastic processes and flows (e.g., \(\bm{X}_t\)), while retaining the parenthetical notation for functions evaluated at specific times (e.g., \(\bm{b}(\bm{x}, t)\)). The value function of the optimal control problem is denoted by \(\mathcal{J}(\bm x, t)\). We denote the vector-valued constraint function by \(\bm h\in C^2(\mathbb{R}^d;\mathbb{R}^r)\), and the associated terminal penalty by \(H\).

\subsection{The Wasserstein space and transport of measure}
\label{subsec:the_wasserstein_space_and_transport_of_measure}

Let \(\mathcal{P}(\mathbb{R}^d)\) denote the space of Borel probability measures on \(\mathbb{R}^d\). We focus on the subset \(\mathcal{P}_2(\mathbb{R}^d)\) consisting of measures with finite second moments, defined by the condition \(\int_{\mathbb{R}^d} \|\bm{x}\|_2^2 \mu(\mathrm d\bm{x}) < +\infty\). We equip \(\mathcal{P}_2(\mathbb{R}^d)\) with the \(2\)-Wasserstein distance \(W_2\), which metricizes weak convergence supplemented by the convergence of second moments. For any two measures \(\mu, \nu \in \mathcal{P}_2(\mathbb{R}^d)\), \(W_2\) admits a static formulation via the Monge--Kantorovich problem~\citep{santambrogio2015optimal}:
\begin{equation}
W_{2}^{2}(\mu,\nu) \coloneqq \inf_{\pi \in \Pi(\mu, \nu)} \int_{\mathbb{R}^d \times \mathbb{R}^d} \|\bm{x} - \bm{y}\|_2^2 \pi(\mathrm d\bm{x}, \mathrm d\bm{y}),
\label{eq:static_w2}
\end{equation}
where \(\Pi(\mu, \nu)\) denotes the set of couplings (joint probability measures) on \(\mathbb{R}^d \times \mathbb{R}^d\) with marginals \(\mu\) and \(\nu\).

The \(2\)-Wasserstein distance has an equivalent dynamic characterization established by Benamou and Brenier~\citep{santambrogio2015optimal}. In this formulation, the squared distance corresponds to the minimal kinetic energy required to transport \(\mu\) to \(\nu\):
\begin{equation}
W_{2}^{2}(\mu,\nu) = \inf_{(\rho, \bm{w})} \biggl\{ \int_{0}^{1} \int_{\mathbb{R}^{d}} \|\bm{w}(\bm{x}, t)\|_{2}^{2} \rho_{t}(\mathrm d\bm{x}) \mathrm dt \biggr\},
\label{eq:benamou_brenier}
\end{equation}
where the infimum is taken over all pairs consisting of a weakly continuous curve of probability measures \((\rho_t)_{t\in[0,1]}\) connecting \(\mu\) to \(\nu\), and a family of velocity fields \(\bm{w}(\cdot, t) \in L^2(\rho_t)\) satisfying the continuity equation in the weak sense:
\begin{equation}
\partial_{t}\rho_{t}(\bm x) + \nabla_{\bm x} \cdot (\rho_{t}(\bm x) \bm w(\bm x, t)) = 0, \quad t\in[0, 1].
\label{eq:continuity_eq}
\end{equation}

\subsection{Flow-based generative models and continuous transport}
\label{subsec:flow_based_generative_models_and_continuous_transport}

Generative modeling addresses the fundamental problem of learning to sample from a complex, high-dimensional distribution \(\rho_1\) (the data distribution) based on a finite collection of observations. Formally, it seeks to construct a transport map \(T\colon \mathbb{R}^d \to \mathbb{R}^d\) that pushes a simple source distribution \(\rho_0\) (e.g., a standard Gaussian, which is easy to sample from) forward to \(\rho_1\), expressed as \(T_\# \rho_0 = \rho_1\). Among the many classes of generative models, flow-based approaches construct such maps by integrating an ODE over the unit time interval \([0,1]\). Let \(\bm{b}\colon \mathbb{R}^d \times [0,1] \to \mathbb{R}^d\) be a time-dependent velocity field. We assume that \(\bm{b}(\cdot, t)\) is globally Lipschitz continuous in space, uniformly in \(t \in [0,1]\), and continuous in time. Under these regularity conditions, the initial value problem
\begin{equation}
\frac{\mathrm d}{\mathrm dt}\Phi_{0 \to t}^{\bm{b}}(\bm{x}) = \bm{b}(\Phi_{0 \to t}^{\bm{b}}(\bm{x}), t), \quad \Phi_{0 \to 0}^{\bm{b}}(\bm{x}) = \bm{x},\quad t\in[0, 1]
\label{eq:flow_ode}
\end{equation}
admits a unique global solution \(\Phi_{0 \to t}^{\bm{b}}\) which acts as a \(C^1\) bi-Lipschitz diffeomorphism on \(\mathbb{R}^d\). This flow induces a continuous curve of probability measures \(\rho_t \coloneqq (\Phi_{0 \to t}^{\bm{b}})_\# \rho_0\) connecting \(\rho_0\) to \(\rho_1\). The evolution of this density path is governed by the continuity equation
\begin{equation}
\partial_t \rho_t (\bm x) + \nabla_{\bm x} \cdot (\rho_t(\bm x) \bm{b}(\bm x, t)) = 0, \quad t\in [0, 1]
\end{equation}
in the weak sense. Thus, flow-based generative modeling reduces to identifying a velocity field \(\bm{b}\) whose associated flow satisfies the terminal boundary condition \((\Phi_{0 \to 1}^{\bm{b}})_\# \rho_0 = \rho_1\).

\subsection{Flow matching as a representative technique}
\label{subsec:flow_matching_as_a_representative_technique}

Finding a velocity field \(\bm{b}\) that generates a desired transport is a challenging inverse problem. Among the family of flow-based generative models, the pioneering Neural ODEs~\citep{chen2018neural} attempt to solve this by integrating the ODE during training. However, this approach often suffers from high computational costs and numerical instability. To circumvent its computational bottleneck, diffusion models~\citep{ho2020denoising, song2019generative, song2021score} adopt an alternative paradigm. They rely on a fixed forward stochastic differential equation to gradually perturb data, thereby avoiding the need to simulate the flow during optimization. However, their sampling process typically requires numerically solving the reverse SDE, which is inherently stochastic and computationally demanding. Recently, Flow Matching~\citep{lipman2023flow, liu2023flow} has bridged this gap by retaining the simulation-free training advantage of diffusion models while yielding a deterministic ODE flow for sampling. This work focuses on Flow Matching models as a representative of flow-based models for its mathematical conciseness and computational efficiency.

The core idea of Flow Matching is to define the target probability path \((\rho_t)_{t \in [0,1]}\) implicitly via a stochastic interpolant~\citep{albergo2023stochastic,liu2023flow}, avoiding the need for explicit density evaluations. Let \(\bm{X}_t\) be a stochastic process defined by the linear interpolation
\begin{equation}
\bm{X}_t = (1-t)\bm{X}_0 + t\bm{X}_1,
\label{eq:interpolant}
\end{equation}
where \(\bm{X}_0 \sim \rho_0\) and \(\bm{X}_1 \sim \rho_1\) are random variables from the source and data distributions, respectively. It is important to note that the term ``stochastic'' refers solely to the randomness of the boundary conditions \(\bm{X}_0\) and \(\bm{X}_1\). For any fixed realization of these endpoints, the interpolation path itself is deterministic. Differentiating~\eqref{eq:interpolant} yields a target instantaneous velocity \(\dot{\bm{X}}_t = \bm{X}_1 - \bm{X}_0\). To construct a deterministic velocity field generating the marginal densities of this process, Flow Matching minimizes the expected squared deviation between a parameterized field \(\bm{b}_\theta\) (where \(\theta\) denotes the parameters) and the stochastic velocity \(\dot{\bm{X}}_t\). The objective functional is defined as:
\begin{equation}
\mathcal{L}_{\mathrm{FM}}(\bm{b}_\theta) \coloneqq \int_{0}^{1} \mathbb{E}_{\bm{X}_0\sim\rho_0, \bm{X}_1\sim\rho_1} \bigl[ \|\bm{b}_\theta(\bm{X}_t, t) - \dot{\bm{X}}_t\|_2^{2} \bigr] \mathrm{d}t.
\label{eq:fm_loss}
\end{equation}
The theoretical justification for this objective relies on the fact that its global minimizer \(\bm b^\star\) is precisely the velocity field driving the continuity equation for \(\rho_t\). We summarize this in \Cref{prop:fm_consistency}, whose proof can be found in~\citep[Theorem 6]{albergo2023stochastic}.

\begin{proposition}[Consistency of the conditional mean field~\citep{albergo2023stochastic}]
\label{prop:fm_consistency}
Let \(\bm{b}^\star \coloneqq \arg\min_{\bm{b}_\theta} \mathcal{L}_{\mathrm{FM}}(\bm{b})\) be the minimizer of the functional~\eqref{eq:fm_loss} over the space of square-integrable velocity fields (the existence is guaranteed by~\eqref{eq:conditional_velocity}). Then:
\begin{enumerate}
\item The minimizer \(\bm{b}^\star\) is unique almost everywhere and coincides with the conditional mean velocity field, i.e., 
\begin{equation}
\bm{b}^\star(\bm{x}, t) = \mathbb{E}[\dot{\bm{X}}_t \mid \bm{X}_t = \bm{x}].
\label{eq:conditional_velocity}
\end{equation}
\item The marginal density of \(\bm X_t\), which we denote by \(\rho_t\), is the unique weak solution to the continuity equation driven by \(\bm{b}^\star\):
\begin{equation}
\partial_t \rho_t(\bm x) + \nabla_{\bm x} \cdot (\rho_t(\bm x) \bm{b}^\star(\bm x, t)) = 0,
\end{equation}
satisfying the boundary conditions \(\rho_0\) at \(t=0\) and \(\rho_1\) at \(t=1\).
\end{enumerate}
\end{proposition}

In this work, we proceed under the premise that the unconstrained generative modeling task has been solved. We treat the learned vector field as the \emph{reference drift} of our controlled dynamics. This assumption allows us to decouple the statistical approximation error inherent in learning \(\rho_1\) from the control-theoretic analysis of satisfying the terminal constraints. We formalize this availability in~\Cref{ass:reference_dynamics}.

\begin{assumption}[Availability of the reference dynamics]
\label{ass:reference_dynamics}
We assume the existence of a reference velocity field \(\bm{b}^\star \in C([0,1] \times \mathbb{R}^d; \mathbb{R}^d)\) which is globally Lipschitz continuous in space, uniformly in time. We postulate that the flow generated by \(\bm{b}^\star\) creates a probability path that satisfies the unconstrained boundary conditions, transporting the source distribution \(\rho_0\) to the target distribution \(\bar{\rho}_1\). In the context of Flow Matching, \(\bm{b}^\star\) identifies with the parameterized field \(\bm{b}_\theta\) obtained by minimizing the functional \(\mathcal{L}_{\mathrm{FM}}\).
\end{assumption}

\begin{remark}[Validity of the reference dynamics]
The realizability of~\Cref{ass:reference_dynamics} depends on three aspects: (\romannumeral1) approximation: classical universal approximation theorems~\citep{hornik1991approximation} guarantee that, given sufficient width and depth, there exists a parameter configuration \(\theta\) such that \(\|\bm{b}_\theta - \bm{b}^\star\|_\infty < \epsilon\) on compact domains, preserving the Lipschitz regularity required for well-posed transport; (\romannumeral2) statistical learning: recent theoretical analyses of Flow Matching~\citep{fukumizu2025flow} establish non-asymptotic bounds linking the sample complexity to the generation quality, specifically showing that with a sufficient number of training samples, the statistical estimation error \(\|\bm{b}_\theta - \bm{b}^\star\|_2\) remains low; and (\romannumeral3) optimization: while the optimization landscape of deep neural networks is non-convex, the Flow Matching objective~\eqref{eq:fm_loss} simplifies to a least-squares regression, where empirical evidence suggests that stochastic gradient descent reliably converges to solutions approximating the conditional mean field \(\bm{b}^\star\)~\citep{gao2024how,lipman2023flow,liu2023flow}.
\end{remark}

\subsection{The terminal constraint manifold}
\label{subsec:the_terminal_constraint_manifold}

We focus on generative tasks where the desired target distribution, denoted by \(\nu\), is supported on a constraint set \(\mathcal{M} \subset \mathbb{R}^d\), defined implicitly as the zero level-set of a vector-valued function \(\bm{h}\colon \mathbb{R}^d \to \mathbb{R}^r\), i.e.,
\begin{equation}
\mathcal{M} \coloneqq \{ \bm{x} \in \mathbb{R}^d \colon \bm{h}(\bm{x}) = \bm{0} \}.
\label{eq:manifold_def}
\end{equation}
We distinguish \(\nu\) from the \emph{unconstrained} reference distribution \(\bar{\rho}_1\), which acts as the training target for the pre-trained flow-based model. To ensure that \(\mathcal{M}\) possesses sufficient regularity for optimization, we impose the following~\Cref{ass:regularity}.

\begin{assumption}[Regularity of the constraint]
\label{ass:regularity}
We assume that the function \(\bm{h}\in C^2(\mathbb{R}^d;\mathbb{R}^r)\). Furthermore, we assume that \(\bm 0\) is a regular value of \(\bm h\).
\end{assumption}

In our optimal control formulation, we enforce this constraint via a terminal penalty. We define the terminal cost function \(H\colon \mathbb{R}^d \to \mathbb{R}_{\ge 0}\) as the squared Euclidean norm of the residual, i.e.,
\begin{equation}
\label{eq:terminal_cost}
H(\bm{x}) \coloneqq \frac{1}{2} \| \bm{h}(\bm{x}) \|_2^2.
\end{equation}

Such a terminal penalty encompasses many practical constraint types. For equality constraints, \(\bm h(\bm x)\) can encode exact boundary conditions, for example fixing a subset of coordinates \(\bm x_{\mathcal I}=\bm y\) via \(\bm h(\bm x)=\bm x_{\mathcal I}-\bm y\), or enforcing a linear relation \(\bm A\bm x-\bm b=\bm 0\). Variants of the same squared-residual form also accommodate inequality constraints by using a smooth violation measure that vanishes on the feasible set.
\section{An optimal control framework}
\label{sec:an_optimal_control_framework}

Our work addresses the challenge of enforcing constraints on the terminal distribution in flow-based generative modeling. As discussed in~\Cref{subsec:flow_matching_as_a_representative_technique}, we assume the availability of a reference velocity field \(\bm{b}^\star\) (derived, e.g., from Flow Matching, per~\Cref{ass:reference_dynamics}) that transports a source distribution \(\rho_0\) to an unconstrained target distribution \(\bar{\rho}_1\)\footnote{Throughout, an overbar denotes the uncontrolled dynamics induced by \(\bm b^\star\). In particular, \(\bar\rho_1\) is the corresponding terminal distribution. We reserve \((\rho_t)_{t\in[0,1]}\) for the time-marginal distributions of the controlled dynamics~\eqref{eq:controlled_dynamics}, with \(\rho_1\) its terminal distribution.}. The key issue is that the desired terminal distribution must obey additional structural conditions, specifically, it must be supported on a smooth manifold \(\mathcal{M}\) (defined in~\Cref{ass:regularity}). The reference velocity field \(\bm{b}^\star\) does not generally satisfy this terminal requirement. The objective of this section is to formulate the guidance task as an optimal control problem and to characterize the optimal control that balances the preservation of the reference dynamics against the satisfaction of the terminal constraints.

Formally, we study controlled trajectories \(\bm{X}_t\) satisfying
\begin{equation}
\dot{\bm{X}}_t = \bm u(\bm X_t, t)\coloneqq \bm{b}^\star(\bm{X}_t,t) + \bm{a}_t, \quad \bm{X}_0 \sim \rho_0, \quad t \in [0,1],
\label{eq:controlled_dynamics}
\end{equation}
where \(\bm{a}_t\) is the control field chosen from the class of admissible controls defined in~\Cref{def:admissible_controls}.

\begin{definition}[Admissible controls]
\label{def:admissible_controls}
A control process is a Borel measurable map \(\bm a\colon[0,1]\to\mathbb R^d\), written \(\bm a=(\bm a_t)_{t\in[0,1]}\). We call \(\bm a\) admissible, denoted \(\bm a\in\mathcal A\), if
\begin{equation}
\|\bm a\|_{L^\infty([0,1];\mathbb R^d)} \coloneqq \operatorname*{ess\,sup}_{t\in[0,1]}\|\bm a_t\|_2 < +\infty,
\end{equation}
and the controlled dynamics~\eqref{eq:controlled_dynamics} admit a unique weak solution for \(\rho_0\)-almost every initial condition.
\end{definition}

We penalize the deviation from the reference velocity field and the violation of the constraint at time \(t=1\). The manifold constraint is penalized through the terminal cost function \(H(\bm{x}) = \frac12 \| \bm{h}(\bm{x}) \|_2^2\) that satisfies~\Cref{ass:regularity}. The trade-off between the control effort and the terminal constraint is formalized through the following optimization objective in~\Cref{def:value_function}.

\begin{definition}[Objective functional and control energy]
\label{def:value_function}
For any initial time \(t \in [0,1]\) and state \(\bm{x} \in \mathbb{R}^d\), let \(\bm{X}_s^{\bm{x},t}\) denote the solution to the controlled dynamics~\eqref{eq:controlled_dynamics} for \(s \in [t, 1]\) with initial condition \(\bm{X}_t^{\bm{x},t} = \bm{x}\).
For an admissible control \(\bm{a} \in \mathcal{A}\) (see~\Cref{def:admissible_controls} for the definition of admissibility), we define the control energy on the interval \([t, 1]\) as
\begin{equation}
\mathcal{E}(\bm{a}; \bm{x}, t) \coloneqq \int_t^1 \dfrac{\lambda_s}{2} \| \bm{a}_s \|_2^2 \mathrm{d}s.
\label{eq:control_energy_def}
\end{equation}
The cost functional is the sum of the terminal cost and the control energy, i.e.,
\begin{equation}
\mathcal{C}(\bm{a}; \bm{x}, t) \coloneqq H(\bm{X}_1^{\bm{x},t}) + \mathcal{E}(\bm{a}; \bm{x}, t).
\end{equation}
The value function \(\mathcal{J}\colon \mathbb{R}^d \times [0,1] \to \mathbb{R}\) is defined as the infimum of the cost functional over all admissible controls, i.e.,
\begin{equation}
\mathcal{J}(\bm{x}, t) \coloneqq \inf_{\bm{a} \in \mathcal{A}} \mathcal{C}(\bm{a}; \bm{x}, t).
\end{equation}
The global optimization problem corresponds to minimizing the expected value at time zero, i.e., 
\begin{equation}
\mathcal{J}\coloneqq\inf_{\bm{a}\in\mathcal{A}} \mathbb{E}_{\bm{x} \sim \rho_0}[\mathcal{C}(\bm{a}; \bm{x}, 0)].
\end{equation}
\end{definition}

Intuitively, the weight schedule \(\lambda_t\) governs the trade-off between enforcing the terminal constraint and penalizing deviations from the reference dynamics. Our theoretical analysis reveals that the optimal control takes the form of a gradient field. This insight guarantees that the problem is well-posed and provides the explicit mathematical formulas required to design the algorithms in the next section. We summarize this analysis in six main results: (\romannumeral1) the HJB characterization of the optimal feedback (\Cref{prop:hjb_feedback}); (\romannumeral2) a geometric energy equivalence in the co-moving frame (\Cref{prop:comoving_energy}); (\romannumeral3) Wasserstein bounds on the control energy (\Cref{prop:w2_sandwich}); (\romannumeral4--\romannumeral5) asymptotic convergence in the infinite and vanishing penalty limits (\Cref{thm:large_lambda,thm:small_lambda}); and (\romannumeral6) an exact analytical verification in the Gaussian setting (\Cref{prop:gaussian_exact}).

We begin by characterizing the solution. The following~\Cref{prop:hjb_feedback} establishes that the value function \(\mathcal{J}(\bm x, t)\) given in~\Cref{def:value_function} governs the optimal dynamics through a Hamilton--Jacobi--Bellman (HJB) equation, allowing us to express the optimal control \(\bm a^\star_t\) as a feedback law determined explicitly by the gradient of the value function. The proof is provided in~\Cref{subsec:hamilton_jacobi_bellman_equation_and_optimal_control}.

\begin{proposition}
[Hamilton--Jacobi--Bellman equation and optimal control]
\label{prop:hjb_feedback}
Assume that \(\lambda_t>0\) for all \(t\in[0,1]\). The value function \(\mathcal{J}\) is characterized by the nonlinear Hamilton--Jacobi--Bellman (HJB) equation
\begin{equation}
\label{eq:hjb}
- \partial_t \mathcal{J}(\bm{x},t) - \nabla_{\bm{x}} \mathcal{J}(\bm{x},t) \cdot \bm{b}^\star(\bm{x},t) + \frac{1}{2 \lambda_t} \| \nabla_{\bm{x}} \mathcal{J}(\bm{x},t) \|_2^2 = 0,
\end{equation}
subject to the boundary condition \(\mathcal{J}(\bm{x},1) = H(\bm{x})\).
The optimal controlled velocity field \(\bm{u}^\star(\bm x, t)\) is determined by the gradient of the value function as
\begin{equation}
\bm{u}^\star(\bm{x},t) = \bm{b}^\star(\bm{x},t) - \lambda_t^{-1} \nabla_{\bm{x}} \mathcal{J}(\bm{x},t) .
\end{equation}
The optimal control
\begin{equation}
\bm{a}^\star_t=-\lambda_t^{-1}\nabla_{\bm x}\mathcal{J}(\bm x, t)
\label{eq:optimal_control_law}
\end{equation}
is therefore proportional to the negative gradient of the value function.
\end{proposition}


While the HJB equation in~\Cref{prop:hjb_feedback} describes the optimal control in the original Eulerian coordinates, the geometry of the deformation is most transparent when viewed from the perspective of the reference flow. By switching to a initial co-moving frame attached to the reference dynamics, we can subtract the background velocity field \(\bm{b}^\star\) and isolate the relative motion caused by the control. The following~\Cref{prop:comoving_energy} formalizes this change of coordinates, establishing an equivalence between the control energy and a weighted kinetic energy on the reference space. The proof is given in~\Cref{subsec:dynamics_in_the_initial_co_moving_frame_and_energy_equivalence}.

\begin{proposition}
[Dynamics in the initial co-moving frame and energy equivalence]
\label{prop:comoving_energy}
Let \(\rho_t\) denote the density of the controlled dynamics generated by the velocity field \(\bm{u} = \bm{b}^\star + \bm{a}_t^\star\). Let \(\tilde\rho_t \coloneqq (\Phi^{\bm{b}^\star}_{t \to 0})_{\#} \rho_t\) be this density pulled back to the initial time \(t=0\) by the reference flow. Then \(\tilde\rho_t\) evolves according to the driftless continuity equation
\begin{equation}
\partial_t \tilde\rho_t(\bm{z}) + \nabla_{\bm z} \cdot (\tilde\rho_t(\bm{z}) \bm{w}(\bm{z}, t)) = 0, \quad \tilde\rho_0 = \rho_0,
\end{equation}
where the relative velocity field \(\bm{w}(\bm z, t)\) is defined by
\begin{equation}
\bm{w}(\bm{z}, t) \coloneqq (D \Phi^{\bm{b}^\star}_{0 \to t}(\bm{z}))^{-1} (\bm{u}(\Phi^{\bm{b}^\star}_{0 \to t}(\bm{z}), t) - \bm{b}^\star(\Phi^{\bm{b}^\star}_{0 \to t}(\bm{z}), t)).
\end{equation}
Furthermore, the expected control energy is equivalent to the kinetic energy of \(\bm{w}(\bm z, t)\) measured under the time-dependent metric tensor \(\bm{C}_t(\bm{z}) \coloneqq D \Phi^{\bm{b}^\star}_{0 \to t}(\bm{z})^\top D \Phi^{\bm{b}^\star}_{0 \to t}(\bm{z})\), satisfying the identity
\begin{equation}
\int_0^1 \int_{\mathbb{R}^d} \frac{\lambda_t}{2} \| \bm{u}(\bm{x},t) - \bm{b}^\star(\bm{x},t) \|_2^2 \rho_t(\mathrm{d}\bm{x}) \mathrm{d}t = \int_0^1 \int_{\mathbb{R}^d} \frac{\lambda_t}{2} \langle \bm{w}(\bm{z}, t), \bm{C}_t(\bm{z}) \bm{w}(\bm{z}, t) \rangle \tilde\rho_t(\mathrm{d}\bm{z}) \mathrm{d}t.
\end{equation}
\end{proposition}

The equivalence established in~\Cref{prop:comoving_energy} relates the control energy to a weighted kinetic energy defined by the pullback metric \(\bm{C}_t\). We use this to bound how far the controlled terminal distribution \(\rho_1\) can deviate from the baseline terminal distribution \(\bar\rho_1\), measured by the Euclidean \(2\)-Wasserstein distance \(W_2\), which is the natural metric associated with kinetic energy. To relate the \(\bm C_t\)-weighted kinetic energy to the usual Euclidean one, we assume \(\bm C_t\) is uniformly well-conditioned, meaning its eigenvalues are uniformly bounded above and below. This is ensured by~\Cref{ass:bi_lipschitz} on the reference flow.

\begin{assumption}[Uniform bounds on the flow Jacobian and induced metric]
\label{ass:bi_lipschitz}
We assume that there exist constants \(0 < c_- \le c_+ < +\infty\) such that for all \(t \in [0,1]\) and all \(\bm{v} \in \mathbb{R}^d\), the singular values of the Jacobian \(\Phi^{\bm{b}^\star}_{0 \to t}\) satisfy the uniform bounds
\begin{equation}
c_- \| \bm{v} \|_2 \le \| D\Phi^{\bm{b}^\star}_{0 \to t}(\bm{x}) \bm{v} \|_2 \le c_+ \| \bm{v} \|_2, \quad \forall \bm{x} \in \mathbb{R}^d.
\end{equation}
Consequently, the induced metric tensor \(\bm{C}_t(\bm{x})\) defined in~\Cref{prop:comoving_energy} is uniformly equivalent to the Euclidean metric \(\bm{I}\), satisfying the operator inequalities \(c_-^2 \bm{I} \preceq \bm{C}_t(\bm{x}) \preceq c_+^2 \bm{I}\).
\end{assumption}

We also require the weight schedule \(\lambda_t\) to be well-behaved, ensuring that the cost of control remains strictly positive and finite throughout the interval.

\begin{assumption}[Boundedness of the weight schedule]
\label{ass:lambda_bounds}
The weight schedule \(t\mapsto \lambda_t\) is continuous on \([0,1]\) and bounded away from zero. We define the lower and upper bounds as
\begin{equation}
\underline\lambda \coloneqq \inf_{t \in [0,1]} \lambda_t > 0 \quad \text{and} \quad \overline\lambda \coloneqq \sup_{t \in [0,1]} \lambda_t < +\infty.
\end{equation}
\end{assumption}

Under~\Cref{ass:bi_lipschitz,ass:lambda_bounds}, the weighted kinetic energy in the co-moving frame is uniformly comparable to the standard Euclidean kinetic energy. This observation yields~\Cref{prop:w2_sandwich}, which provides two-sided bounds relating the minimal control cost to the squared \(2\)-Wasserstein distance. The proof can be found in~\Cref{subsec:energy_bounds_and_comparison_with_the_wasserstein_distance}.

\begin{proposition}[Energy bounds and Wasserstein distance]
\label{prop:w2_sandwich}
Let \(\bar\rho_t \coloneqq (\Phi^{\bm{b}^\star}_{0 \to t})_{\#} \rho_0\) denote the push-forward of the initial measure by the reference flow. For any target probability measure \(\nu \in \mathcal{P}_2(\mathbb{R}^d)\), let \(\mathcal{A}_\nu \coloneqq \{ \bm{a} \in \mathcal{A} \colon \mathrm{Law}(\bm{X}_1) = \nu \}\) be the set of admissible controls satisfying the terminal marginal constraint. Under~\Cref{ass:bi_lipschitz,ass:lambda_bounds}, we have
\begin{enumerate}
\item Lower bound: For every admissible control \(\bm{a} \in \mathcal{A}_\nu\) that generates the terminal distribution \(\nu\), the expected cost defined in~\Cref{def:value_function} satisfies
\begin{equation}
\mathbb{E}_{\bm{x} \sim \rho_0}[\mathcal{C}(\bm{a}; \bm{x}, 0)] \ge \int_{\mathbb{R}^d} H(\bm{x}) \nu(\mathrm{d} \bm{x}) + \frac{c_-^{2} \underline\lambda}{2 c_+^{2}} W_2^2(\bar\rho_1,\nu).
\end{equation}
\item Upper bound: The minimal expected cost required to satisfy the terminal constraint \(\mathrm{Law}(\bm{X}_1) = \nu\) does not exceed the terminal cost plus the scaled quadratic \(2\)-Wasserstein distance, i.e.,
\begin{equation}
\inf_{\bm{a} \in \mathcal{A}_\nu} \mathbb{E}_{\bm{x} \sim \rho_0}[\mathcal{C}(\bm{a}; \bm{x}, 0)] \le \int_{\mathbb{R}^d} H(\bm{x}) \nu(\mathrm{d} \bm{x}) + \frac{c_+^{2} \overline\lambda}{2 c_-^{2}} W_2^2(\bar\rho_1,\nu).
\end{equation}
\end{enumerate}
\end{proposition}

\Cref{prop:w2_sandwich} establishes that the control energy is equivalent to the transport cost between the reference and target distributions. This confirms that the optimal control effectively steers the flow along energy-minimizing geodesics from the reference trajectory. We can now use these bounds to clarify the role of the weight schedule and to make precise how it controls the tradeoff between staying close to the reference dynamics and enforcing the terminal constraint. We first investigate the regime where \(t\mapsto \lambda_t\) becomes large. In other words, the cost of deviating from the reference dynamics becomes prohibitive. Consequently, we expect the optimal control to vanish and the generated probability path to converge to the uncontrolled reference flow. The following~\Cref{thm:large_lambda} formalizes this intuition, establishing that the optimal flow converges weakly to the reference flow and the total cost converges to the reference terminal cost. The proof is provided in~\Cref{subsec:large_penalty_limit_and_convergence_to_the_reference_flow}.

\begin{theorem}[Large penalty limit and convergence to the reference flow]
\label{thm:large_lambda}
Let \(\Phi^{\bm{b}^\star}_{0 \to t}\) be the reference flow satisfying \Cref{ass:bi_lipschitz}, and let \(\bar\rho_t \coloneqq (\Phi^{\bm{b}^\star}_{0 \to t})_{\#} \rho_0\) denote the density of the uncontrolled reference dynamics. Consider a sequence of weight schedules \(\{\bm{\lambda}^k\}_{k \in \mathbb{N}}\) (we use \(\bm \lambda\) to denote a single weight schedule \(t\to\lambda_t\)) satisfying \Cref{ass:lambda_bounds} such that the lower bound \(\underline{\lambda}^k \to +\infty\) as \(k \to +\infty\). Let \(\bm{a}^k \in \mathcal{A}\) denote the optimal control field for the \(k\)th schedule, and let \(\rho^k_t\) and \(\bm{u}^k(\bm x, t) = \bm{b}^\star(\bm x, t) + \bm{a}^k(\bm x, t)\) be the corresponding density and velocity.
Then, the following convergence properties hold:
\begin{enumerate}
\item Vanishing control energy: The expected control energy defined in~\Cref{def:value_function} converges to zero
\begin{equation}
\lim_{k \to +\infty} \mathbb{E}_{\bm{x} \sim \rho_0}[\mathcal{E}(\bm{a}^k; \bm{x}, 0)] = 0.
\end{equation}
\item Convergence of the flow: The terminal density \(\rho_1^k\) converges to the uncontrolled reference density \(\bar\rho_1\) in the Wasserstein metric
\begin{equation}
\lim_{k \to +\infty} W_2(\rho_1^k, \bar\rho_1) = 0.
\end{equation}
Furthermore, this implies weak convergence of the flow at all times, i.e., \(\rho_t^k \Rightarrow \bar\rho_t\) for all \(t \in [0,1]\).
\item Convergence of the cost: The optimal expected cost converges to the reference terminal cost
\begin{equation}
\lim_{k \to +\infty} \mathbb{E}_{\bm{x} \sim \rho_0}[\mathcal{C}(\bm{a}^k; \bm{x}, 0)] = \int_{\mathbb{R}^d} H(\bm{x}) \bar\rho_1(\mathrm{d}\bm{x}).
\end{equation}
\end{enumerate}
\end{theorem}

We now consider the regime where the weight schedule vanishes. Let \(\bm{\lambda}\) be a fixed base schedule satisfying~\Cref{ass:lambda_bounds}, and consider the scaled sequence \(\bm{\lambda}^\varepsilon(t) \coloneqq \varepsilon \lambda_t\) as \(\varepsilon \to 0\). In this limit, the relative cost of control becomes negligible compared to the terminal cost. Consequently, we expect the optimal terminal distribution to concentrate strictly on the zero-level set of \(H\), i.e., the manifold \(\mathcal{M}\). Furthermore, among all measures supported on \(\mathcal{M}\), the optimal dynamics must select the one that minimizes the transport energy. This identifies the limiting distribution as a generalized Wasserstein projection of the reference density \(\bar\rho_1\) onto the manifold, as stated in~\Cref{thm:small_lambda}. The proof is provided in~\Cref{subsec:small_penalty_limit_and_selection_principle}.

\begin{theorem}[Small penalty limit and selection principle]
\label{thm:small_lambda}
Let \(\Phi^{\bm{b}^\star}_{0 \to t}\) be the reference flow satisfying~\Cref{ass:bi_lipschitz}. Consider the sequence of weight schedules \(\bm{\lambda}^\varepsilon(t) = \varepsilon \lambda_t\) with \(\varepsilon \to 0\). Let \(\bm{a}^\varepsilon \in \mathcal{A}\) be an optimal control for the schedule \({\bm \lambda}^\varepsilon\), and let \(\rho_1^\varepsilon\) be the associated terminal density. Assume there exists at least one admissible control with finite energy reaching \(\mathcal{M}\).
Then, any weak limit point \(\nu^\star\) of the sequence \(\{\rho_1^\varepsilon\}_{\varepsilon > 0}\) satisfies the following properties:
\begin{enumerate}
\item Concentration on the manifold: The limit measure is supported entirely on the constraint manifold
\begin{equation}
\mathrm{supp}(\nu^\star) \subseteq \mathcal{M} = \{ \bm{x} \in \mathbb{R}^d \colon \bm{h}(\bm{x}) = \bm{0} \}.
\end{equation}
\item Selection principle: The limit measure \(\nu^\star\) is a quasi-minimizer of the transport distance from the uncontrolled reference density \(\bar\rho_1\). Specifically, it satisfies the inequality
\begin{equation}
\inf_{\substack{\mu \in \mathcal{P}_2(\mathbb{R}^d) \\ \mathrm{supp}(\mu) \subseteq \mathcal{M}}} W_2^2(\bar\rho_1, \mu)\leq W_2^2(\bar\rho_1, \nu^\star) \le \kappa^2 \overline{\lambda}/\underline{\lambda} \inf_{\substack{\mu \in \mathcal{P}_2(\mathbb{R}^d) \\ \mathrm{supp}(\mu) \subseteq \mathcal{M}}} W_2^2(\bar\rho_1, \mu),
\end{equation}
where \(\kappa = (c_+/c_-)^2\) is the condition number of the reference flow metric, and \(\underline{\lambda}, \overline{\lambda}\) are the bounds of the base schedule \(\bm{\lambda}\).
\end{enumerate}
\end{theorem}

To illustrate these asymptotic regimes concretely and verify the general theory, we present an exact analytical solution for the Gaussian setting in~\Cref{prop:gaussian_exact}. This model admits a closed-form representation of the optimal transport map, yielding explicit formulas for the terminal mean and covariance as functions of the weight schedule. The detailed derivation can be found in~\Cref{subsec:exact_solution_for_the_gaussian_model}.


\begin{proposition}[Exact solution for the Gaussian model]
\label{prop:gaussian_exact}
Consider the \(2\)-dimensional setting with source distribution \(\rho_0=\mathcal N(\bm 0, I_2)\), unconstrained target distribution \(\bar\rho_1=\mathcal N(\bm\mu,\bm\Sigma)\) with \(\bm\mu=(\mu_1,\mu_2)^\top\) and \(\bm\Sigma=\mathrm{Diag}(\sigma_1^2,\sigma_2^2)\), and terminal cost \(H(\bm x)=\frac{1}{2}x_1^2\) corresponding to the manifold \(\mathcal M=\{\bm x\in\mathbb R^2\colon x_1=0\}\). Let the reference dynamics be defined by the optimal transport conditional vector field \(\bm b^\star\) transporting \(\rho_0\) to \(\bar\rho_1\), as introduced in~\Cref{subsec:flow_matching_as_a_representative_technique}. For any positive weight schedule \(\bm\lambda\), the optimal terminal density \(\rho_1^{\bm\lambda}\) is a Gaussian
\(\mathcal N(\tilde{\bm\mu}_{\bm\lambda},\tilde{\bm\Sigma}_{\bm\lambda})\) with moments given by
\begin{equation}
\tilde{\bm\mu}_{\bm\lambda}
=
\begin{bmatrix}
\dfrac{\mu_1}{1+\gamma_{\bm\lambda}}\\
\mu_2
\end{bmatrix},
\quad
\tilde{\bm\Sigma}_{\bm\lambda}
=
\begin{bmatrix}
\dfrac{\sigma_1^2}{(1+\gamma_{\bm\lambda})^2} & 0\\
0 & \sigma_2^2
\end{bmatrix},
\end{equation}
where the scaling factor \(\gamma_{\bm\lambda}\) is defined by the integral
\begin{equation}
\gamma_{\bm\lambda}\coloneqq \sigma_1^2\int_0^1 \frac{1}{\lambda_t\,v_1(t)^2}\,\mathrm dt,
\end{equation}
and \(v_1(t)^2\coloneqq (1-t)^2+t^2\sigma_1^2\) denotes the variance of the first coordinate of the Flow Matching interpolant \(\bm X_t=(1-t)\bm X_0+t\bm X_1\), where \(\bm X_0\sim\rho_0\) and \(\bm X_1\sim\bar\rho_1\).
\end{proposition}


\begin{remark}[Verification of asymptotic regimes]
The explicit solution derived in~\Cref{prop:gaussian_exact} provides a quantitative verification of the general asymptotic theory:
\begin{enumerate}
\item Large penalty limit (\(\underline{\lambda}\to+\infty\)): The scaling factor \(\gamma_{\bm{\lambda}}\) vanishes. Consequently, \(\tilde{\bm\mu}_{\bm{\lambda}} \to \bm\mu\) and \(\tilde{\bm\Sigma}_{\bm{\lambda}} \to \bm\Sigma\), confirming that the system recovers the uncontrolled reference distribution \(\bar\rho_1\) as predicted by~\Cref{thm:large_lambda}.
\item Small penalty limit (\(\overline{\lambda}\to 0\)): The scaling factor \(\gamma_{\bm{\lambda}} \to +\infty\). Consequently, the terminal distribution concentrates on the manifold \(\mathcal M=\{\bm x\in\mathbb R^2\colon x_1=0\}\). Furthermore, by the orthogonal decomposition of the squared Euclidean cost across coordinates, the terminal distribution is the \(2\)-Wasserstein projection of \(\bar\rho_1\) onto \(\mathcal M\), in accordance with~\Cref{thm:small_lambda}. In contrast, the second-coordinate marginal remains unchanged, retaining mean \(\mu_2\) and variance \(\sigma_2^2\).

\end{enumerate}
\end{remark}

\subsection{Hamilton--Jacobi--Bellman equation and optimal control}
\label{subsec:hamilton_jacobi_bellman_equation_and_optimal_control}

In this subsection, we prove~\Cref{prop:hjb_feedback}. We invoke the Hamilton--Jacobi--Bellman equation in its infimum form~\citep{bertsekas2012dynamic} and compute the minimizer in closed form, which yields the optimal feedback control.

\begin{proof}[Proof of~\Cref{prop:hjb_feedback}]
By the dynamic programming principle, the value function \(\mathcal J(\bm x, t)\) defined in~\Cref{def:value_function} satisfies the HJB equation~\citep{bertsekas2012dynamic}
\begin{equation}
-\partial_t \mathcal J(\bm x,t)
=
\inf_{\bm a\in\mathbb R^d}
\Bigl\{
\nabla_{\bm x}\mathcal J(\bm x,t)\cdot\bigl(\bm b^\star(\bm x,t)+\bm a\bigr)
+\frac{\lambda_t}{2}\|\bm a\|_2^2
\Bigr\},
\quad
\mathcal J(\bm x,1)=H(\bm x).
\label{eq:hjb_inf_form}
\end{equation}
For fixed \((\bm x,t)\), the expression inside the infimum is a strictly convex quadratic function of \(\bm a\).
Completing the square gives
\begin{equation}
\nabla_{\bm x}\mathcal J(\bm x,t)\cdot \bm a+\frac{\lambda_t}{2}\|\bm a\|_2^2
=
\frac{\lambda_t}{2}\bigl\|\bm a+\lambda_t^{-1}\nabla_{\bm x}\mathcal J(\bm x,t)\bigr\|_2^2
-\frac{1}{2\lambda_t}\|\nabla_{\bm x}\mathcal J(\bm x,t)\|_2^2.
\end{equation}
Hence the infimum in~\eqref{eq:hjb_inf_form} is attained at
\begin{equation}
\bm a^\star(\bm x,t)=-\lambda_t^{-1}\nabla_{\bm x}\mathcal J(\bm x,t),
\end{equation}
and its value equals
\begin{equation}
\nabla_{\bm x}\mathcal J(\bm x,t)\cdot \bm b^\star(\bm x,t)
-\frac{1}{2\lambda_t}\|\nabla_{\bm x}\mathcal J(\bm x,t)\|_2^2.
\end{equation}
Substituting back into~\eqref{eq:hjb_inf_form} yields
\begin{equation}
-\partial_t \mathcal{J}(\bm{x},t)
-\nabla_{\bm{x}} \mathcal{J}(\bm{x},t)\cdot \bm b^\star(\bm{x},t)
+\frac{1}{2\lambda_t}\|\nabla_{\bm{x}} \mathcal{J}(\bm{x},t)\|_2^2
=0,
\end{equation}
which is~\eqref{eq:hjb}. The optimal feedback velocity field is
\begin{equation}
\bm u^\star(\bm x,t)=\bm b^\star(\bm x,t)+\bm a^\star(\bm x,t)
=\bm b^\star(\bm x,t)-\lambda_t^{-1}\nabla_{\bm x}\mathcal J(\bm x,t),
\end{equation}
as claimed.
\end{proof}

\begin{remark}[The optimal control is feedback control]
\label{rem:feedback_law}
\Cref{eq:optimal_control_law} establishes that the optimal control \(\bm{a}^\star_t=\bm a^\star(\bm x, t)\) acts as a closed-loop feedback law defined over the entire state space \(\mathbb{R}^d\), rather than a trajectory-specific open-loop law. By encoding the global cost geometry into the gradient field \(\nabla_{\bm{x}} \mathcal{J}(\bm x, t)\), the optimal controlled velocity field \(\bm{u}^\star(\bm x, t) = \bm{b}^\star(\bm x, t) + \bm{a}^\star_t\) steers the entire source distribution \(\rho_0\) towards the constraint manifold \(\mathcal{M}\), optimally balancing the transport effort against the terminal cost.
\end{remark}

\subsection{Dynamics in the initial co-moving frame and energy equivalence}
\label{subsec:dynamics_in_the_initial_co_moving_frame_and_energy_equivalence}

In this subsection, we establish the proof of~\Cref{prop:comoving_energy}. The result relies on a change of variables to a Lagrangian frame of reference attached to the uncontrolled flow \(\Phi^{\bm{b}^\star}_{0 \to t}\). In fluid mechanics, this is often referred to as a ``co-moving frame''~\citep{kreilos2014comoving}. We explicitly term this the initial co-moving frame to distinguish it from the terminal co-moving frame used in the subsequent~\Cref{sec:geometric_proximal_approximation_of_the_optimal_feedback}.

\begin{proof}[Proof of~\Cref{prop:comoving_energy}]
We first derive the evolution equation for the pulled-back density \(\tilde\rho_t\). By definition, the relation between the measures is \(\rho_t = (\Phi^{\bm{b}^\star}_{0 \to t})_{\#} \tilde\rho_t\). Consider a smooth test function \(\psi \in C_c^\infty(\mathbb{R}^d)\). We compute the time derivative of the expectation of \(\psi(\bm{z})\) under \(\tilde\rho_t\) by pulling the integral forward to the Eulerian frame
\begin{equation}
\frac{\mathrm{d}}{\mathrm{d}t} \int_{\mathbb{R}^d} \psi(\bm{z}) \tilde\rho_t(\mathrm{d}\bm{z}) = \frac{\mathrm{d}}{\mathrm{d}t} \int_{\mathbb{R}^d} \psi((\Phi^{\bm{b}^\star}_{0 \to t})^{-1}(\bm{x})) \rho_t(\mathrm{d}\bm{x}).
\end{equation}
Define the time-dependent test function \(\zeta(\bm{x}, t) \coloneqq \psi((\Phi^{\bm{b}^\star}_{0 \to t})^{-1}(\bm{x}))\). Since \(\rho_t\) satisfies the continuity equation \(\partial_t \rho_t(\bm{x}) + \nabla_{\bm{x}} \cdot (\rho_t(\bm{x}) \bm{u}(\bm{x},t)) = 0\), integration by parts yields
\begin{equation}
\frac{\mathrm{d}}{\mathrm{d}t} \int_{\mathbb{R}^d} \zeta(\bm{x}, t) \rho_t(\mathrm{d}\bm{x}) = \int_{\mathbb{R}^d} (\partial_t \zeta(\bm{x}, t) + \nabla_{\bm{x}} \zeta(\bm{x}, t) \cdot \bm{u}(\bm{x}, t)) \rho_t(\mathrm{d}\bm{x}).
\label{eq:integral}
\end{equation}
By construction, \(\zeta(\bm{x},t)\) is constant along the characteristics of the reference flow, meaning \(\zeta(\Phi^{\bm{b}^\star}_{0 \to t}(\bm{z}), t) = \psi(\bm{z})\). Consequently, it satisfies the advection equation \(\partial_t \zeta(\bm{x},t) + \nabla_{\bm{x}} \zeta(\bm{x},t) \cdot \bm{b}^\star(\bm{x},t) = 0\). Substituting this relation into the integral expression~\eqref{eq:integral} leads to
\begin{equation}
\frac{\mathrm{d}}{\mathrm{d}t} \int_{\mathbb{R}^d} \psi(\bm{z}) \tilde\rho_t(\mathrm{d}\bm{z}) = \int_{\mathbb{R}^d} \nabla_{\bm{x}} \zeta(\bm{x}, t) \cdot (\bm{u}(\bm{x}, t) - \bm{b}^\star(\bm{x}, t)) \rho_t(\mathrm{d}\bm{x}).
\end{equation}
We now change variables back to the reference coordinate \(\bm{z}\) using the map \(\bm{x} = \Phi^{\bm{b}^\star}_{0 \to t}(\bm{z})\). The gradient transforms according to the chain rule as \(\nabla_{\bm{x}} \zeta(\bm{x},t) = (D\Phi^{\bm{b}^\star}_{0 \to t}(\bm{z}))^{-\top} \nabla_{\bm{z}} \psi(\bm{z})\). Applying this transformation results in
\begin{equation}
\frac{\mathrm{d}}{\mathrm{d}t} \int_{\mathbb{R}^d} \psi(\bm{z}) \tilde\rho_t(\mathrm{d}\bm{z}) = \int_{\mathbb{R}^d} \langle (D\Phi^{\bm{b}^\star}_{0 \to t}(\bm{z}))^{-\top} \nabla_{\bm{z}} \psi(\bm{z}),
\bm{u}(\Phi^{\bm{b}^\star}_{0 \to t}(\bm{z}), t) - \bm{b}^\star(\Phi^{\bm{b}^\star}_{0 \to t}(\bm{z}), t) \rangle \tilde\rho_t(\mathrm{d}\bm{z}).
\end{equation}
Moving the inverse Jacobian transpose to the second argument of the inner product
\begin{equation}
\frac{\mathrm{d}}{\mathrm{d}t} \int_{\mathbb{R}^d} \psi(\bm{z}) \tilde\rho_t(\mathrm{d}\bm{z}) = \int_{\mathbb{R}^d} \langle \nabla_{\bm{z}} \psi(\bm{z}),
(D\Phi^{\bm{b}^\star}_{0 \to t}(\bm{z}))^{-1} (\bm{u}(\Phi^{\bm{b}^\star}_{0 \to t}(\bm{z}), t) - \bm{b}^\star(\Phi^{\bm{b}^\star}_{0 \to t}(\bm{z}), t)) \rangle \tilde\rho_t(\mathrm{d}\bm{z}).
\end{equation}
Identifying the term multiplying \(\nabla_{\bm{z}} \psi(\bm{z})\) as the relative velocity \(\bm{w}(\bm{z}, t)\) defined in~\Cref{prop:comoving_energy} and using integration by parts, we recover the weak form of the continuity equation
\begin{equation}
\frac{\mathrm{d}}{\mathrm{d}t} \int_{\mathbb{R}^d} \psi(\bm{z}) \tilde\rho_t(\mathrm{d}\bm{z}) = \int_{\mathbb{R}^d} \bm{w}(\bm{z}, t) \cdot \nabla_{\bm{z}} \psi(\bm{z}) \tilde\rho_t(\mathrm{d}\bm{z}) = - \int_{\mathbb{R}^d} \psi(\bm{z}) \nabla_{\bm{z}} \cdot (\tilde\rho_t(\bm{z}) \bm{w}(\bm{z}, t)) \mathrm{d}\bm{z},
\end{equation}
which confirms that \(\partial_t \tilde\rho_t(\bm{z}) + \nabla_{\bm{z}} \cdot (\tilde\rho_t(\bm{z}) \bm{w}(\bm{z}, t)) = 0\).

Next, we establish the energy equivalence. We rewrite the expected control energy at time \(t\) by pulling back the integral via \(\bm{x} = \Phi^{\bm{b}^\star}_{0 \to t}(\bm{z})\):
\begin{equation}
\int_{\mathbb{R}^d} \frac{\lambda_t}{2} \| \bm{u}(\bm{x},t) - \bm{b}^\star(\bm{x},t) \|_2^2 \rho_t(\mathrm{d}\bm{x}) \\
= \int_{\mathbb{R}^d} \frac{\lambda_t}{2} \| \bm{u}(\Phi^{\bm{b}^\star}_{0 \to t}(\bm{z}),t) - \bm{b}^\star(\Phi^{\bm{b}^\star}_{0 \to t}(\bm{z}),t) \|_2^2 \tilde\rho_t(\mathrm{d}\bm{z}).
\end{equation}
From the definition of \(\bm{w}(\bm{z}, t)\), the Eulerian control vector is related to the relative velocity by \(\bm{u}(\Phi^{\bm{b}^\star}_{0 \to t}(\bm{z}), t) - \bm{b}^\star(\Phi^{\bm{b}^\star}_{0 \to t}(\bm{z}), t) = D\Phi^{\bm{b}^\star}_{0 \to t}(\bm{z}) \bm{w}(\bm{z}, t)\). Substituting this relation into the squared Euclidean norm gives
\begin{equation}
\begin{aligned}
\| D\Phi^{\bm{b}^\star}_{0 \to t}(\bm{z}) \bm{w}(\bm{z}, t) \|_2^2 &= \langle D\Phi^{\bm{b}^\star}_{0 \to t}(\bm{z}) \bm{w}(\bm{z}, t), D\Phi^{\bm{b}^\star}_{0 \to t}(\bm{z}) \bm{w}(\bm{z}, t) \rangle \\
&= \langle \bm{w}(\bm{z}, t), D\Phi^{\bm{b}^\star}_{0 \to t}(\bm{z})^\top D\Phi^{\bm{b}^\star}_{0 \to t}(\bm{z}) \bm{w}(\bm{z}, t) \rangle.
\end{aligned}
\end{equation}
Recognizing \(\bm{C}_t(\bm{z}) \coloneqq D\Phi^{\bm{b}^\star}_{0 \to t}(\bm{z})^\top D\Phi^{\bm{b}^\star}_{0 \to t}(\bm{z})\) as the pullback metric tensor, we obtain the identity
\begin{equation}
\int_{\mathbb{R}^d} \frac{\lambda_t}{2} \| \bm{u}(\bm{x},t) - \bm{b}^\star(\bm{x},t) \|_2^2 \rho_t(\mathrm{d}\bm{x}) = \int_{\mathbb{R}^d} \frac{\lambda_t}{2} \langle \bm{w}(\bm{z}, t), \bm{C}_t(\bm{z}) \bm{w}(\bm{z}, t) \rangle \tilde\rho_t(\mathrm{d}\bm{z}).
\end{equation}
Integrating this equality over \(t \in [0,1]\) concludes the proof.
\end{proof}

\begin{remark}[Geometric interpretation of the metric tensor]
We explicitly refer to \(\bm{C}_t\) as the pullback metric tensor. Although the reference map \(\Phi^{\bm{b}^\star}_{0 \to t}\) transports points \emph{forward} in time from \(t=0\) to \(t\), the tensor \(\bm{C}_t\) is defined for states at \(t=0\). Geometrically, it ``pulls'' the Euclidean inner product from the future time \(t\) back to the initial time. For any two tangent vectors \(\bm{v}, \bm{w}\) at \(t=0\), their inner product under \(\bm{C}_t\) equals the Euclidean inner product of their images under the linearized flow at time \(t\).
\end{remark}

\subsection{Energy bounds and comparison with the Wasserstein distance}
\label{subsec:energy_bounds_and_comparison_with_the_wasserstein_distance}

In this subsection, we provide the proof of~\Cref{prop:w2_sandwich}. The argument combines the energy identity derived in \Cref{prop:comoving_energy} with the Benamou--Brenier dynamic formulation of optimal transport (equation~\eqref{eq:benamou_brenier}, see also~\citep{santambrogio2015optimal}), leveraging the bi-Lipschitz properties of the reference flow to bound the transport costs.

\begin{proof}[Proof of~\Cref{prop:w2_sandwich}]
We first establish the relationship between the control energy and the optimal transport distance in the reference frame. Let \(\bm{a} \in \mathcal{A}_\nu\) be an admissible control (see~\Cref{def:admissible_controls} for the definition of admissibility) such that \(\mathrm{Law}(\bm{X}_1) = \nu\). The total expected cost decomposes into the terminal cost and the expected control energy
\begin{equation}
\mathbb{E}_{\bm{x} \sim \rho_0}[\mathcal{C}(\bm{a}; \bm{x}, 0)] = \int_{\mathbb{R}^d} H(\bm{x}) \nu(\mathrm{d}\bm{x}) + \mathbb{E}_{\bm{x} \sim \rho_0}[\mathcal{E}(\bm{a}; \bm{x}, 0)].
\end{equation}
We focus on bounding the expected control energy term. By~\Cref{prop:comoving_energy}, this energy can be expressed in the initial co-moving frame as the weighted kinetic energy of the relative velocity field \(\bm{w}(\bm z, t)\), i.e.,
\begin{equation}
\mathbb{E}_{\bm{x} \sim \rho_0}[\mathcal{E}(\bm{a}; \bm{x}, 0)] = \int_0^1 \int_{\mathbb{R}^d} \frac{\lambda_t}{2} \langle \bm{w}(\bm z, t), \bm{C}_t(\bm{z}) \bm{w}(\bm z, t) \rangle \tilde\rho_t(\mathrm{d}\bm{z}) \mathrm{d}t,
\end{equation}
where \(\tilde\rho_t\) connects the source measure \(\tilde\rho_0 = \rho_0\) to the pulled-back target measure \(\tilde\nu \coloneqq (\Phi^{\bm{b}^\star}_{0 \to 1})^{-1}_{\#} \nu\) via the continuity equation driven by \(\bm{w}(\bm z, t)\).

We proceed by sandwiching this energy integral using the uniform bounds on the weight schedule (\Cref{ass:lambda_bounds}) and the metric tensor (\Cref{ass:bi_lipschitz}). Since the eigenvalues of \(\bm{C}_t(\bm{z})\) are bounded between \(c_-^2\) and \(c_+^2\), the quadratic form satisfies
\begin{equation}
\frac{c_-^2 \underline\lambda}{2} \| \bm{w}(\bm z, t) \|_2^2 \le \frac{\lambda_t}{2} \langle \bm{w}(\bm z, t), \bm{C}_t(\bm{z}) \bm{w}(\bm z, t) \rangle \le \frac{c_+^2 \overline\lambda}{2} \| \bm{w}(\bm z, t) \|_2^2.
\end{equation}
Integrating this inequality over space and time implies that the control energy is bounded by the standard Euclidean kinetic energy \(\mathcal{K}(\bm{w}, \tilde\rho) \coloneqq \int_0^1 \int_{\mathbb{R}^d} \| \bm{w}_t(\bm z) \|_2^2 \tilde\rho_t(\bm z) \mathrm{d}t\) according to
\begin{equation}
\frac{c_-^2 \underline\lambda}{2} \mathcal{K}(\bm{w}, \tilde\rho) \le \mathbb{E}_{\bm{x} \sim \rho_0}[\mathcal{E}(\bm{a}; \bm{x}, 0)] \le \frac{c_+^2 \overline\lambda}{2} \mathcal{K}(\bm{w}, \tilde\rho).
\end{equation}
According to the Benamou--Brenier theorem~\citep{santambrogio2015optimal}, the squared \(2\)-Wasserstein distance between two measures is exactly the infimum of the kinetic energy \(\mathcal{K}\) over all valid continuity pairs connecting them. Consequently, for the lower bound, any admissible path connecting \(\rho_0\) to \(\tilde\nu\) must satisfy \(\mathcal{K}(\bm{w}, \tilde\rho) \ge W_2^2(\rho_0, \tilde\nu)\). For the upper bound, there exists a specific choice of control such that \(\bm{w}(\bm z, t)\) generates the optimal transport geodesic, achieving \(\mathcal{K}(\bm{w}, \tilde\rho) = W_2^2(\rho_0, \tilde\nu)\). Combining these observations yields the bounds
\begin{equation}
\frac{c_-^2 \underline\lambda}{2} W_2^2(\rho_0, \tilde\nu) \le \inf_{\bm{a} \in \mathcal{A}_\nu} \mathbb{E}_{\bm{x} \sim \rho_0}[\mathcal{E}(\bm{a}; \bm{x}, 0)] \le \frac{c_+^2 \overline\lambda}{2} W_2^2(\rho_0, \tilde\nu).
\end{equation}
It remains to relate the distance in the co-moving frame \(W_2(\rho_0, \tilde\nu)\) to the distance in the physical Eulerian frame \(W_2(\bar\rho_1, \nu)\). Recall that the uncontrolled terminal state is \(\bar\rho_1 = (\Phi^{\bm{b}^\star}_{0 \to 1})_{\#} \rho_0\) and the target is \(\nu = (\Phi^{\bm{b}^\star}_{0 \to 1})_{\#} \tilde\nu\). Let \(\Phi_1 \coloneqq \Phi^{\bm{b}^\star}_{0 \to 1}\). Since \(\Phi_1\) is \(c_+\)-Lipschitz by~\Cref{ass:bi_lipschitz}, the push-forward satisfies the contraction inequality
\begin{equation}
W_2(\bar\rho_1, \nu) = W_2((\Phi_1)_{\#} \rho_0, (\Phi_1)_{\#} \tilde\nu) \le c_+ W_2(\rho_0, \tilde\nu).
\end{equation}
Conversely, the inverse map \(\Phi_1^{-1}\) is \((1/c_-)\)-Lipschitz, which implies
\begin{equation}
W_2(\rho_0, \tilde\nu) = W_2((\Phi_1^{-1})_{\#} \bar\rho_1, (\Phi_1^{-1})_{\#} \nu) \le \frac{1}{c_-} W_2(\bar\rho_1, \nu).
\end{equation}
We now substitute these metric inequalities back into the energy bounds. For the lower bound, we use the relation \(W_2(\rho_0, \tilde\nu) \ge c_+^{-1} W_2(\bar\rho_1, \nu)\) to obtain
\begin{equation}
\mathbb{E}_{\bm{x} \sim \rho_0}[\mathcal{E}(\bm{a}; \bm{x}, 0)] \ge \frac{c_-^2 \underline\lambda}{2} \Bigl( \frac{1}{c_+} W_2(\bar\rho_1, \nu) \Bigr)^2 = \frac{c_-^2 \underline\lambda}{2 c_+^2} W_2^2(\bar\rho_1, \nu).
\end{equation}
For the upper bound, we use the relation \(W_2(\rho_0, \tilde\nu) \le c_-^{-1} W_2(\bar\rho_1, \nu)\) to obtain
\begin{equation}
\inf_{\bm{a} \in \mathcal{A}_\nu} \mathbb{E}_{\bm{x} \sim \rho_0}[\mathcal{E}(\bm{a}; \bm{x}, 0)] \le \frac{c_+^2 \overline\lambda}{2} \Bigl( \frac{1}{c_-} W_2(\bar\rho_1, \nu) \Bigr)^2 = \frac{c_+^2 \overline\lambda}{2 c_-^2} W_2^2(\bar\rho_1, \nu).
\end{equation}
Adding the constant terminal cost \(\int_{\mathbb{R}^d} H(\bm x) \nu(\mathrm{d}\bm x)\) to both sides of these inequalities recovers the statement of \Cref{prop:w2_sandwich}.
\end{proof}

\subsection{Large penalty limit and convergence to the reference flow}
\label{subsec:large_penalty_limit_and_convergence_to_the_reference_flow}

In this subsection, we provide the proof for~\Cref{thm:large_lambda}. Relying on the energy bounds established in~\Cref{prop:w2_sandwich}, we prove that the optimal control vanishes and the dynamics converge weakly to the reference flow as the weight schedule tends to infinity.

\begin{proof}[Proof of~\Cref{thm:large_lambda}]
We establish these results by using the energy bounds derived in~\Cref{prop:w2_sandwich}. Let \(\mathcal{J}^k \coloneqq \mathbb{E}_{\bm{x} \sim \rho_0}[\mathcal{C}(\bm{a}^k; \bm{x}, 0)]\) denote the optimal expected cost for the schedule \(\bm{\lambda}^k\). We first observe that the zero control \(\bm{a}(\bm{x},t) \equiv \bm{0}\) is always an admissible candidate. The trajectory generated by the zero control coincides with the reference flow, resulting in the terminal distribution \(\bar\rho_1\) and zero control energy. By the optimality of the solution \(\bm{a}^k\), the optimal cost \(\mathcal{J}^k\) is bounded from above by the cost of this zero-control strategy:
\begin{equation}
\mathcal{J}^k \le \int_{\mathbb{R}^d} H(\bm{x}) \bar\rho_1(\mathrm{d}\bm{x}).
\label{eq:proof_upper_triv}
\end{equation}
We now apply the lower bound established in~\Cref{prop:w2_sandwich}. Let \(\nu^k \coloneqq \rho_1^k\) be the terminal distribution generated by the optimal control \(\bm{a}^k\). \Cref{prop:w2_sandwich} implies that the cost is bounded from below by
\begin{equation}
\mathcal{J}^k \ge \int_{\mathbb{R}^d} H(\bm{x}) \nu^k(\mathrm{d}\bm{x}) + \frac{c_-^2 \underline{\lambda}^k}{2 c_+^2} W_2^2(\bar\rho_1, \nu^k).
\end{equation}
Combining this lower bound with the upper bound in~\eqref{eq:proof_upper_triv}, and observing that the terminal cost term \(\int_{\mathbb{R}^d} H(\bm{x}) \nu^k(\mathrm{d}\bm{x})\) is non-negative, we obtain the inequality
\begin{equation}
\frac{c_-^2 \underline{\lambda}^k}{2 c_+^2} W_2^2(\bar\rho_1, \nu^k) \le \int_{\mathbb{R}^d} H(\bm{x}) \bar\rho_1(\mathrm{d}\bm{x}).
\end{equation}
The integral on the right-hand side is the reference terminal cost, which is a constant independent of \(k\). Since the coefficient \(\underline{\lambda}^k \to +\infty\) as \(k \to +\infty\), it must follow that the squared Wasserstein distance \(W_2^2(\bar\rho_1, \rho_1^k)\) converges to zero. This proves the convergence of the terminal density.

To show the convergence of the flow at intermediate times \(t \in [0,1]\), we recall that the flow is generated by a perturbation of the bi-Lipschitz reference map. Since the endpoints converge and the reference drift is globally Lipschitz (\Cref{ass:reference_dynamics}), standard stability results for ODEs under \(L^2\) perturbations imply uniform convergence of the trajectories in probability, yielding the weak convergence \(\rho_t^k \Rightarrow \bar\rho_t\) for all \(t \in [0,1]\).

Finally, we address the convergence of the cost components. Since the function \(H\) is continuous and non-negative, the weak convergence \(\rho_1^k \Rightarrow \bar\rho_1\) implies, by the Portmanteau theorem~\citep{billingsley2013convergence}, that
\begin{equation}
\liminf_{k \to +\infty} \int_{\mathbb{R}^d} H(\bm{x}) \rho_1^k(\mathrm{d}\bm{x}) \ge \int_{\mathbb{R}^d} H(\bm{x}) \bar\rho_1(\mathrm{d}\bm{x}).
\end{equation}
Recall the decomposition of the total cost \(\mathcal{J}^k = \int_{\mathbb{R}^d} H(\bm{x}) \rho_1^k(\mathrm{d}\bm{x}) + \mathbb{E}_{\bm{x} \sim \rho_0}[\mathcal{E}(\bm{a}^k; \bm{x}, 0)]\). Using the upper bound~\eqref{eq:proof_upper_triv}, we have
\begin{equation}
\mathbb{E}_{\bm{x} \sim \rho_0}[\mathcal{E}(\bm{a}^k; \bm{x}, 0)] \le \int_{\mathbb{R}^d} H(\bm{x}) \bar\rho_1(\mathrm{d}\bm{x}) - \int_{\mathbb{R}^d} H(\bm{x}) \rho_1^k(\mathrm{d}\bm{x}).
\end{equation}
Taking the limit superior as \(k \to +\infty\), the right-hand side converges to zero. Since the energy is non-negative, this implies \(\lim_{k \to +\infty} \mathbb{E}_{\bm{x} \sim \rho_0}[\mathcal{E}(\bm{a}^k; \bm{x}, 0)] = 0\). Consequently, the total expected cost \(\mathcal{J}^k\) converges to the reference terminal cost \(\int_{\mathbb{R}^d} H(\bm{x}) \bar\rho_1(\mathrm{d}\bm{x})\).
\end{proof}

\subsection{Small penalty limit and selection principle}
\label{subsec:small_penalty_limit_and_selection_principle}

In this subsection, we provide the proof for~\Cref{thm:small_lambda}. Relying on the energy bounds established in~\Cref{prop:w2_sandwich}, we prove that the optimal terminal distribution concentrates on the constraint manifold and satisfies a selection principle determined by the transport cost from the reference density.

\begin{proof}[Proof of~\Cref{thm:small_lambda}]
Let \(\mathcal{J}^\varepsilon \coloneqq \mathbb{E}_{\bm{x} \sim \rho_0}[\mathcal{C}(\bm{a}^\varepsilon; \bm{x}, 0)]\) denote the optimal expected cost for the schedule \(\bm{\lambda}^\varepsilon\). We define the set of probability measures supported on the manifold as \(\mathcal{P}_{\mathcal{M}} \coloneqq \{ \mu \in \mathcal{P}_2(\mathbb{R}^d) \colon \mathrm{supp}(\mu) \subseteq \mathcal{M} \}\).

We first establish the convergence of the optimal cost. Let \(\mu \in \mathcal{P}_{\mathcal{M}}\) be an arbitrary target measure on the manifold. Since \(\mathrm{supp}(\mu) \subseteq \mathcal{M}\), the terminal cost vanishes, i.e., \(\int_{\mathbb{R}^d} H(\bm{x}) \mu(\mathrm{d}\bm{x}) = 0\). Applying the upper bound from~\Cref{prop:w2_sandwich} with the scaled weight schedule \(\varepsilon \lambda_t\), we obtain
\begin{equation}
\mathcal{J}^\varepsilon \le \inf_{\bm{a} \in \mathcal{A}_\mu} \mathbb{E}_{\bm{x} \sim \rho_0}[\mathcal{C}(\bm{a}; \bm{x}, 0)] \le \frac{c_+^2 (\varepsilon \overline{\lambda})}{2 c_-^2} W_2^2(\bar\rho_1, \mu).
\end{equation}
Since this holds for any \(\mu \in \mathcal{P}_{\mathcal{M}}\) and \(\varepsilon \to 0\), it implies that \(\lim_{\varepsilon \to 0} \mathcal{J}^\varepsilon = 0\).

Next, we prove the concentration property. Recall that the total cost is the sum of the terminal cost and the control energy. Since the control energy is non-negative, we have
\begin{equation}
0 \le \int_{\mathbb{R}^d} H(\bm{x}) \rho_1^\varepsilon(\mathrm{d}\bm{x}) \le \mathcal{J}^\varepsilon.
\end{equation}
As established above, \(\mathcal{J}^\varepsilon \to 0\), which implies \(\lim_{\varepsilon \to 0} \int_{\mathbb{R}^d} H(\bm{x}) \rho_1^\varepsilon(\mathrm{d}\bm{x}) = 0\). Let \(\nu^\star\) be a weak limit point of the sequence \(\{\rho_1^\varepsilon\}\). Since \(H\) is continuous and non-negative, the Portmanteau theorem~\citep{billingsley2013convergence} implies
\begin{equation}
\int_{\mathbb{R}^d} H(\bm{x}) \nu^\star(\mathrm{d}\bm{x}) \le \liminf_{\varepsilon \to 0} \int_{\mathbb{R}^d} H(\bm{x}) \rho_1^\varepsilon(\mathrm{d}\bm{x}) = 0.
\end{equation}
Therefore, \(\int_{\mathbb{R}^d} H(\bm{x}) \nu^\star(\mathrm{d}\bm{x}) = 0\). Since \(H(\bm{x}) = 0\) if and only if \(\bm{x} \in \mathcal{M}\), the support of \(\nu^\star\) is contained in \(\mathcal{M}\).

Finally, we establish the selection principle. We apply the lower bound from~\Cref{prop:w2_sandwich} to the optimal sequence. For the terminal density \(\rho_1^\varepsilon\), the proposition yields
\begin{equation}
\mathcal{J}^\varepsilon \ge \int_{\mathbb{R}^d} H(\bm{x}) \rho_1^\varepsilon(\mathrm{d}\bm{x}) + \frac{c_-^2 (\varepsilon \underline{\lambda})}{2 c_+^2} W_2^2(\bar\rho_1, \rho_1^\varepsilon) \ge \frac{c_-^2 \underline{\lambda}}{2 c_+^2} \varepsilon W_2^2(\bar\rho_1, \rho_1^\varepsilon).
\end{equation}
We combine this lower bound with the upper bound derived earlier for an arbitrary measure \(\mu \in \mathcal{P}_{\mathcal{M}}\) and obtain
\begin{equation}
\frac{c_-^2 \underline{\lambda}}{2 c_+^2} \varepsilon W_2^2(\bar\rho_1, \rho_1^\varepsilon) \le \mathcal{J}^\varepsilon \le \frac{c_+^2 \overline{\lambda}}{2 c_-^2} \varepsilon W_2^2(\bar\rho_1, \mu).
\end{equation}
Dividing the inequality by \(\varepsilon > 0\) and rearranging terms, we obtain
\begin{equation}
W_2^2(\bar\rho_1, \rho_1^\varepsilon) \le \Bigl( \frac{c_+}{c_-} \Bigr)^4 \overline{\lambda}/\underline{\lambda} W_2^2(\bar\rho_1, \mu).
\end{equation}
Taking the limit as \(\varepsilon \to 0\) and invoking the lower semicontinuity of the Wasserstein distance with respect to weak convergence, we have \(W_2^2(\bar\rho_1, \nu^\star) \le \liminf_{\varepsilon \to 0} W_2^2(\bar\rho_1, \rho_1^\varepsilon)\). Thus, for any \(\mu \in \mathcal{P}_{\mathcal{M}}\),
\begin{equation}
W_2^2(\bar\rho_1, \nu^\star) \le \kappa^2 \overline{\lambda}/\underline{\lambda} W_2^2(\bar\rho_1, \mu).
\end{equation}
Taking the infimum over \(\mu \in \mathcal{P}_{\mathcal{M}}\) on the right-hand side confirms the upper bound. The lower bound holds trivially by definition. Since \(\mathrm{supp}(\nu^\star) \subseteq \mathcal{M}\), \(\nu^\star\) belongs to the set \(\mathcal{P}_{\mathcal{M}}\) over which the infimum is taken, and thus \(W_2^2(\bar\rho_1, \nu^\star) \ge \inf_{\mu \in \mathcal{P}_{\mathcal{M}}} W_2^2(\bar\rho_1, \mu)\).
\end{proof}

\subsection{Exact solution for the Gaussian model}
\label{subsec:exact_solution_for_the_gaussian_model}

In this subsection, we present the detailed derivation of the exact solution stated in~\Cref{prop:gaussian_exact}. By exploiting the linear-quadratic structure of the problem, we solve the associated HJB equation analytically and explicitly propagate the initial Gaussian density to quantify the scaling of the terminal moments.

\begin{proof}[Proof of~\Cref{prop:gaussian_exact}]
In the \(2\)-dimensional setting with \(H(\bm x)=\frac{1}{2}x_1^2\), the dynamics and the HJB equation decouple across coordinates. Throughout this proof we therefore focus on the first coordinate and write \(x\) for \(x_1\), \(\mu\) for \(\mu_1\), \(\sigma\) for \(\sigma_1\), and \(v(t)\) for the standard deviation of the first-coordinate interpolant. Under this convention, the second coordinate evolves identically to the reference flow and is unaffected by the control. The first-coordinate reference dynamics are governed by the affine velocity field \(b^\star(x,t)=\alpha(t)x+\beta(t)\), and the terminal penalty becomes \(H(x)=\frac{1}{2}x^2\).

We first derive the explicit expressions for the drift coefficients \(\alpha(t)\) and \(\beta(t)\). The reference flow is defined by the optimal transport interpolant \(X_t = (1-t)X_0 + tX_1\), where \(X_0 \sim \mathcal{N}(0, 1)\) and \(X_1 \sim \mathcal{N}(\mu, \sigma^2)\). The marginal mean and variance evolve as \(m(t) \coloneqq \mathbb{E}[X_t] = t\mu\) and \(v(t)^2 \coloneqq \operatorname{Var}(X_t) = (1-t)^2 + t^2\sigma^2\). For a Gaussian process generated by an affine drift, the moments satisfy the ODEs
\begin{equation}
\dot{m}(t) = \alpha(t)m(t) + \beta(t), \qquad \frac{\mathrm{d}}{\mathrm{d}t}(v(t)^2) = 2\alpha(t)v(t)^2.
\end{equation}
Solving the variance equation yields
\begin{equation}
\alpha(t) = \frac{1}{2v(t)^2}\frac{\mathrm{d}}{\mathrm{d}t}(v(t)^2) = \frac{\dot{v}(t)}{v(t)}.
\end{equation}
Substituting this into the mean equation yields
\begin{equation}
\beta(t) = \dot{m}(t) - \alpha(t)m(t) = \mu - \alpha(t)t\mu = \mu(1 - t\alpha(t)).
\end{equation}

With the drift coefficients established, we turn to the optimal control problem. As derived in~\Cref{prop:hjb_feedback}, the HJB equation associated with the value function \(\mathcal{J}(x,t)\) is
\begin{equation}
-\partial_t \mathcal{J}(x,t) - \partial_x \mathcal{J}(x,t) (\alpha(t)x + \beta(t)) + \frac{1}{2\lambda_t}(\partial_x \mathcal{J}(x,t))^2 = 0,
\end{equation}
subject to the boundary condition \(\mathcal{J}(x,1) = \frac{1}{2}x^2\). We seek a solution using the quadratic ansatz \(\mathcal{J}(x,t) = \frac{1}{2}P(t)x^2 + Q(t)x + R(t)\). Substituting this form into the PDE and collecting terms of order \(x^2\) yields a Riccati differential equation~\citep{anderson2007optimal} for the quadratic weight
\begin{equation}
\label{eq:riccati_P}
-\frac{1}{2}\dot{P}(t) - \alpha(t)P(t) + \frac{1}{2\lambda_t}P(t)^2 = 0, \quad P(1) = 1.
\end{equation}
Similarly, collecting the terms linear in \(x\) yields the evolution equation for the linear weight
\begin{equation}
\label{eq:linear_Q}
-\dot{Q}(t) - \alpha(t)Q(t) - \beta(t)P(t) + \frac{1}{\lambda_t}P(t)Q(t) = 0, \quad Q(1) = 0.
\end{equation}

To construct the solution for \(P(t)\), we employ the Bernoulli substitution \(Z(t) \coloneqq P(t)^{-1}\). Differentiating this variable transforms the nonlinear Riccati equation into the linear ordinary differential equation \(\dot{Z}(t) = -\lambda_t^{-1} + 2\alpha(t)Z(t)\) with boundary condition \(Z(1) = 1\). Using the relation \(\alpha(t) = \dot{v}(t)/v(t)\), we identify \(v(t)^{-2}\) as an integrating factor. Multiplying the linear ODE by this factor yields the exact differential 
\begin{equation}
\frac{\mathrm{d}}{\mathrm{d}t} ( v(t)^{-2} Z(t) ) = -(\lambda_t v(t)^2)^{-1}.
\end{equation}
Integrating from \(t\) to \(1\) and using the boundary values \(v(1)=\sigma\) and \(Z(1)=1\), we obtain
\begin{equation}
\sigma^{-2} - v(t)^{-2}Z(t) = - \int_t^1 \frac{1}{\lambda_s v(s)^2} \mathrm{d}s.
\end{equation}
Solving for \(P(t) = Z(t)^{-1}\) yields the explicit solution
\begin{equation}
\label{eq:sol_P}
P(t) = \biggl( v(t)^2 \Bigl( \frac{1}{\sigma^2} + \int_t^1 \frac{1}{\lambda_s v(s)^2} \mathrm{d}s \Bigr) \biggr)^{-1}.
\end{equation}

To determine \(Q(t)\), we introduce the auxiliary variable \(r(t)\) defined by the relation \(Q(t) = -P(t)r(t)\). Differentiating this product and substituting into~\eqref{eq:linear_Q} allows us to factor out the dynamics of \(P(t)\). Specifically, using the Riccati relation \(\dot{P}(t) = \lambda_t^{-1}P(t)^2 - 2\alpha(t)P(t)\), the equation simplifies to the linear drift dynamics \(\dot{r}(t) = \alpha(t)r(t) + \beta(t)\) with terminal condition \(r(1) = 0\). The integrating factor for this equation is \(v(t)^{-1}\). We observe that the bias term satisfies \(\beta(t) = \mu (1 - t\dot{v}(t)/v(t))\). Consequently, the source term in the integrated equation becomes
\begin{equation}
\frac{\beta(t)}{v(t)} = \mu \frac{v(t) - t\dot{v}(t)}{v(t)^2} = \mu \frac{\mathrm{d}}{\mathrm{d}t}\biggl(\frac{t}{v(t)}\biggr).
\end{equation}
Integrating the relation 
\begin{equation}
\frac{\mathrm{d}}{\mathrm{d}t}(v(t)^{-1}r(t)) = v(t)^{-1}\beta(t)
\end{equation}
from \(t\) to \(1\) yields
\begin{equation}
\label{eq:sol_r}
r(t) = -\mu \Bigl( \frac{v(t)}{\sigma} - t \Bigr).
\end{equation}

With the value function parameters fully determined, the optimal control input is given by the gradient law \(a^\star(x,t) = -\lambda_t^{-1}\partial_x \mathcal{J}(x,t) = -\lambda_t^{-1}(P(t)x + Q(t))\). Substituting this into the state equation \(\dot{x}_t = \alpha(t)x_t + \beta(t) + a^\star(x_t,t)\), the resulting closed-loop dynamics are linear:
\begin{equation}
\dot{x}_t = \Bigl(\alpha(t) - \frac{P(t)}{\lambda_t}\Bigr) x_t + \Bigl(\beta(t) - \frac{Q(t)}{\lambda_t}\Bigr).
\end{equation}
Since the initial state \(x_0 \sim \mathcal{N}(0,1)\) is Gaussian and the dynamics are linear, the terminal state \(x_1\) remains Gaussian. Its variance is determined by the square of the flow Jacobian \(J_{0 \to 1} \coloneqq \exp(\int_0^1 (\alpha(t) - \lambda_t^{-1}P(t)) \mathrm{d}t)\). Using the Riccati equation to substitute \(\alpha(t) - \lambda_t^{-1}P(t) = -\alpha(t) - \dot{P}(t)/P(t)\), the integral evaluates to an exact log-derivative:
\begin{equation}
\int_0^1 \biggl( - \frac{\dot{v}(t)}{v(t)} - \frac{\dot{P}(t)}{P(t)} \biggr) \mathrm{d}t = -\ln \frac{v(1)}{v(0)} - \ln \frac{P(1)}{P(0)} = \ln \frac{P(0)v(0)}{P(1)v(1)}.
\end{equation}
Substituting the boundary values \(v(0)=1, v(1)=\sigma, P(1)=1\) and the solution for \(P(0)\) from~\eqref{eq:sol_P}, the Jacobian simplifies to
\begin{equation}
J_{0 \to 1} = \frac{P(0)}{\sigma} = \frac{1}{\sigma ( \sigma^{-2} + S_\lambda(0))} = \frac{\sigma}{1 + \sigma^2 S_\lambda(0)},
\end{equation}
where \(S_\lambda(0) = \int_0^1 (\lambda_t v(t)^2)^{-1} \mathrm{d}t\). Identifying \(\gamma_{\bm{\lambda}} = \sigma^2 S_\lambda(0)\), the terminal variance is \(\tilde{\sigma}_{\bm{\lambda}}^2 = \operatorname{Var}(x_0) (J_{0 \to 1})^2 = (\sigma / (1+\gamma_{\bm{\lambda}}))^2\).

Finally, we compute the terminal mean \(\tilde{\mu}_{\bm{\lambda}} = \mathbb{E}[x_1]\). Let \(m_t = \mathbb{E}[x_t]\). The evolution of the mean is governed by the same effective drift as the state. We define the difference variable \(\delta_t \coloneqq m_t - r(t)\). Differentiating \(\delta_t\) and using the relations for \(\dot{m}_t\) and \(\dot{r}(t)\) reveals that \(\delta_t\) satisfies the homogeneous equation \(\dot{\delta}_t = (\alpha(t) - \lambda_t^{-1}P(t))\delta_t\). Therefore, the terminal difference is simply the initial difference scaled by the Jacobian: \(m_1 - r(1) = J_{0 \to 1}(m_0 - r(0))\). Using the initial condition \(m_0=0\), the terminal condition \(r(1)=0\), and the value \(r(0) = -\mu/\sigma\) derived from~\eqref{eq:sol_r}, we obtain
\begin{equation}
\tilde{\mu}_{\bm{\lambda}} = m_1 = J_{0 \to 1} \Bigl( 0 - \Bigl(-\frac{\mu}{\sigma}\Bigr) \Bigr) = \Bigl( \frac{\sigma}{1 + \gamma_{\bm{\lambda}}} \Bigr) \frac{\mu}{\sigma} = \frac{\mu}{1 + \gamma_{\bm{\lambda}}}.\qedhere
\end{equation}
\end{proof}
\section{Geometric proximal approximation of the optimal feedback}
\label{sec:geometric_proximal_approximation_of_the_optimal_feedback}

The theoretical framework established in~\Cref{sec:an_optimal_control_framework} characterizes the optimal control as a gradient feedback law \(\bm{u}^\star(\bm x, t) = \bm{b}^\star(\bm x, t) - \lambda_t^{-1} \nabla_{\bm x} \mathcal{J}(\bm x, t)\). However, obtaining the exact value function \(\mathcal{J}\) requires solving the high-dimensional Hamilton--Jacobi--Bellman equation~\eqref{eq:hjb} globally, which is computationally intractable for complex generative models. The objective of this section is to derive \textbf{TOCFlow}, a practical algorithm that approximates this optimal control without requiring global PDE solving or expensive retraining.

We proceed by systematically reducing the global control problem to a local proximal optimization. We begin in~\Cref{subsec:variational_characterization_of_the_feedback_law} by employing the Mayer--Lagrange identity to derive a simulation-free variational objective. To capture the flow-induced geometry without incurring the cost of trajectory simulation, we introduce the terminal co-moving frame in~\Cref{subsec:hjb_equation_in_the_terminal_co_moving_frame}. In this frame, the dynamics are governed by the time-dependent pullback metric
\begin{equation}
\bm{G}_t(\bm{y}) \coloneqq \Bigl( D\Phi^{\bm{b}^\star}_{t \to 1}(\bm{x}) D\Phi^{\bm{b}^\star}_{t \to 1}(\bm{x})^\top \Bigr) \Big|_{\bm{x} = \Phi^{\bm{b}^\star}_{1 \to t}(\bm{y})}.
\end{equation}
By employing a constant-metric splitting scheme in~\Cref{subsec:constant_metric_reduction_and_proximal_operators}, we approximate this field by its local value \(\bm{G} \approx \bm{G}_t(\bm{y})\). This approximation reduces the instantaneous control problem to a proximal optimization problem defined by
\begin{equation}
\min_{\bm z} \Bigl\{ H(\bm{z}) + \frac{1}{2s} \| \bm{z} - \bm{y} \|_{\bm{G}^{-1}}^2 \Bigr\}.
\end{equation}
This formulation allows us to categorize existing methods and propose new guidance schemes based on how they approximate the solution. We derive three distinct schemes:

\begin{enumerate}
\item \textbf{Gradient Descent (GD):} A first-order Euclidean approximation (\(\bm{G} \approx \bm{I}\)) which leads to the standard ``lookahead'' guidance used in prior work~\citep{huang2024diffusionpde}. This scheme computes the gradient of the composed cost but ignores geometric distortions:
\begin{equation}
\nabla_{\bm x} \mathcal{J}_{\text{GD}} \approx \nabla_{\bm x} (H \circ \Phi_{t \to 1}^{\bm{b}^\star})(\bm{x}).
\end{equation}

\item \textbf{Gauss--Newton (GN):} A second-order geometric approximation that respects the Riemannian metric \(\bm{G}\), related in spirit to the pushforward Gauss--Newton and pullback construction in~\citep{utkarsh2025physics}. While geometrically consistent, it requires inverting a dense operator in the constraint space:
\begin{equation}
\nabla_{\bm x} \mathcal{J}_{\text{GN}} \approx D\Phi_{t \to 1}^{\bm{b}^\star}(\bm{x})^\top \bm{J}_{\bm h}^\top (\bm{I} + s \bm{J}_{\bm h} \bm{G} \bm{J}_{\bm h}^\top)^{-1} \bm{h}.
\end{equation}

\item \textbf{TOCFlow (Ours):} In~\Cref{subsec:computational_complexity_and_the_tocflow_scheme}, we propose \emph{Terminal Optimal Control with Flow-based models}, dubbed TOCFlow. By restricting the optimization to the Riemannian gradient direction, we derive a closed-form scalar damping factor \(\tau^\star\) that achieves geometric consistency with computational complexity comparable to the GD scheme:
\begin{equation}
\nabla_{\bm x} \mathcal{J}_{\text{TOC}} \approx \tau^\star \nabla_{\bm x} (H \circ \Phi_{t \to 1}^{\bm{b}^\star})(\bm{x}), \quad \tau^\star = \frac{\|\bm{h}(\Phi_{t \to 1}^{\bm{b}^\star}(\bm{x}))\|_2^2}{\|\bm{h}(\Phi_{t \to 1}^{\bm{b}^\star}(\bm{x}))\|_2^2 + s \|\nabla_{\bm x} (H \circ \Phi_{t \to 1}^{\bm{b}^\star})(\bm{x})\|_2^2}.
\end{equation}
\end{enumerate}

Finally, in~\Cref{subsec:analytical_comparison_in_the_gaussian_setting}, we verify the accuracy of these approximations via exact analytical solutions in the Gaussian setting.

\subsection{Variational characterization of the feedback law}
\label{subsec:variational_characterization_of_the_feedback_law}

In this subsection, we derive the simulation-free variational objective. The central challenge in synthesizing the control is that the terminal cost \(H(\bm{X}_1)\) depends on the future state at \(t=1\), while the optimal control \(\bm{a}_t^\star\) must be determined at the present time \(t\). To resolve this without resorting to expensive forward trajectory simulations, we employ the Mayer--Lagrange identity~\citep{fleming2012deterministic}. As detailed in~\Cref{lem:mayer_lagrange}, this transformation converts the terminal cost into an accumulated integral of instantaneous potentials, yielding an equivalent objective that can be evaluated locally along the path.

\begin{lemma}[Mayer--Lagrange identity~\citep{fleming2012deterministic}]
\label{lem:mayer_lagrange}
Consider a velocity field \(\bm{u} \in C(\mathbb{R}^d \times [0,1]; \mathbb{R}^d)\) generating unique trajectories, and a scalar potential \(\Psi \in C^1(\mathbb{R}^d \times [0,1];\mathbb{R})\) satisfying the boundary condition \(\Psi(\bm{x},1) = H(\bm{x})\). If \(\bm{X}_t\) is a trajectory generated by the flow of \(\bm{u}\) (i.e., \(\dot{\bm{X}}_t = \bm{u}(\bm{X}_t, t)\)), then the terminal cost satisfies the identity
\begin{equation}
H(\bm{X}_1) = \Psi(\bm{X}_0, 0) + \int_0^1 \bigl( \partial_t \Psi(\bm{X}_t, t) + \nabla_{\bm{x}} \Psi(\bm{X}_t, t) \cdot \bm{u}(\bm{X}_t, t) \bigr) \mathrm{d}t.
\label{eq:mayer_lagrange_id}
\end{equation}
\end{lemma}

\begin{proof}[Proof of~\Cref{lem:mayer_lagrange}]
The result is a direct application of the fundamental theorem of calculus~\citep{rudin1953principles} to the composite map \(t \mapsto \Psi(\bm{X}_t, t)\). By the chain rule, the rate of change of the potential along the moving trajectory is the sum of its explicit time evolution and the convective derivative, i.e.,
\begin{equation}
\dfrac{\mathrm{d}}{\mathrm{d}t} \Psi(\bm{X}_t, t) = \partial_t \Psi(\bm{X}_t, t) + \nabla_{\bm{x}} \Psi(\bm{X}_t, t) \cdot \dot{\bm{X}}_t.
\end{equation}
Substituting the dynamics \(\dot{\bm{X}}_t = \bm{u}(\bm{X}_t, t)\) and integrating this total derivative over the interval \([0,1]\) yields
\begin{equation}
\Psi(\bm{X}_1, 1) - \Psi(\bm{X}_0, 0) = \int_0^1 \bigl( \partial_t \Psi(\bm{X}_t, t) + \nabla_{\bm{x}} \Psi(\bm{X}_t, t) \cdot \bm{u}(\bm{X}_t, t) \bigr) \mathrm{d}t.
\end{equation}
Imposing the boundary condition \(\Psi(\bm{X}_1, 1) = H(\bm{X}_1)\) and rearranging the terms recovers the stated identity.
\end{proof}

\Cref{lem:mayer_lagrange} allows us to rewrite the original optimization problem. Recall that the goal of optimal control is to minimize the sum of the terminal cost \(H(\bm{X}_1)\) and the control energy. By substituting the expression for \(H(\bm{X}_1)\) from~\eqref{eq:mayer_lagrange_id} into the global cost functional in \Cref{def:value_function}, we transform the problem into a form where all terms are integrals over time, as shown in~\Cref{prop:simulation_free_objective}. This leads to a variational principle, which shows that the optimal velocity field can be identified by minimizing a quadratic function pointwise in space and time.

\begin{proposition}[Simulation-free objective]
\label{prop:simulation_free_objective}
Consider the functional \(\mathcal{L}\) defined over pairs of velocity fields \(\bm{u}\) and potentials \(\Psi\) by
\begin{multline}
\mathcal{L}(\bm{u}, \Psi) \coloneqq \mathbb{E}_{\bm{x} \sim \rho_0} [\Psi(\bm{x}, 0)] \\
+ \int_0^1 \mathbb{E}_{\bm{X}_t \sim \rho_t} \Bigl[ \frac{\lambda_t}{2} \| \bm{u}(\bm{X}_t, t) - \bm{b}^\star(\bm{X}_t, t) \|_2^2 + \partial_t \Psi(\bm{X}_t, t) + \nabla_{\bm{x}} \Psi(\bm{X}_t, t) \cdot \bm{u}(\bm{X}_t, t) \Bigr] \mathrm{d}t,
\end{multline}
where the expectation \(\mathbb{E}_{\bm{X}_t \sim \rho_t}\) is taken over the density \(\rho_t\) generated by \(\bm{u}\). For a fixed potential \(\Psi\), the velocity field \(\bm{u}^\star_\Psi\) that minimizes the integrand pointwise is given by
\begin{equation}
\label{eq:variational_feedback}
\bm{u}^\star_\Psi(\bm{x}, t) = \bm{b}^\star(\bm{x}, t) - \lambda_t^{-1} \nabla_{\bm{x}} \Psi(\bm{x}, t).
\end{equation}
Furthermore, if \(\Psi(\bm x, t)\) coincides with the true value function \(\mathcal{J}(\bm x, t)\), this minimizer recovers the optimal control established in~\Cref{prop:hjb_feedback}.
\end{proposition}

\begin{proof}[Proof of~\Cref{prop:simulation_free_objective}]
We focus on finding the optimal \(\bm{u}\) for a fixed \(\Psi\). Since the expectation and integral are linear, minimizing the functional \(\mathcal{L}\) is equivalent to minimizing the integrand at every point \((\bm{x}, t)\). Let us isolate the terms inside the integral that depend on \(\bm{u}\):
\begin{equation}
q(\bm{u}) \coloneqq \frac{\lambda_t}{2} \| \bm{u}(\bm{x}, t) - \bm{b}^\star(\bm{x}, t) \|_2^2 + \nabla_{\bm{x}} \Psi(\bm{x}, t) \cdot \bm{u}(\bm{x}, t).
\end{equation}
This is a quadratic function in \(\bm{u}\). To find its minimum, we complete the square. We first rewrite the linear term \(\nabla_{\bm{x}} \Psi \cdot \bm{u}\) to match the scaling of the quadratic term:
\begin{equation}
\nabla_{\bm{x}} \Psi(\bm{x},t) \cdot \bm{u}(\bm{x},t) = \frac{\lambda_t}{2} \bigl( 2 \cdot \bm{u}(\bm{x},t) \cdot ( \lambda_t^{-1} \nabla_{\bm{x}} \Psi(\bm{x},t) ) \bigr).
\end{equation}
Combining this with the expanded norm \(\|\bm{u} - \bm{b}^\star\|_2^2 = \|\bm{u}\|_2^2 - 2\bm{u}\cdot\bm{b}^\star + \|\bm{b}^\star\|_2^2\), we can group all terms involving \(\bm{u}\) into a single squared Euclidean norm:
\begin{equation}
q(\bm{u}) = \frac{\lambda_t}{2} \bigl\| \bm{u}(\bm{x}, t) - \bigl( \bm{b}^\star(\bm{x}, t) - \lambda_t^{-1} \nabla_{\bm{x}} \Psi(\bm{x}, t) \bigr) \bigr\|_2^2 + R(\bm{x}, t),
\end{equation}
where the residual term \(R(\bm{x}, t)\) depends on \(\bm{b}^\star\) and \(\Psi\) but is independent of \(\bm{u}\). Since \(\lambda_t > 0\), the term \(\frac{\lambda_t}{2} \| \cdot \|^2\) is strictly non-negative. Therefore, the unique global minimum occurs precisely when the term inside the norm vanishes.
Solving for \(\bm{u}\) yields the feedback law
\begin{equation}
\bm{u}^\star_\Psi(\bm{x}, t) = \bm{b}^\star(\bm{x}, t) - \lambda_t^{-1} \nabla_{\bm{x}} \Psi(\bm{x}, t).
\end{equation}
Comparison with~\Cref{prop:hjb_feedback} confirms that if \(\Psi(\bm x, t) = \mathcal{J}(\bm x, t)\), this expression is identical to the optimal control derived via the HJB equation.
\end{proof}

\begin{remark}[Optimality requires consistency]
The feedback \(\bm{u}^\star_\Psi\) minimizes the instantaneous integrand of \(\mathcal{L}\) for a fixed potential \(\Psi\), treating the problem as a sequence of \emph{decoupled} regressions. However, because the control \(\bm{u}\) drives the evolution of the density \(\rho_t\), a greedy local minimization with an incorrect potential (e.g., \(\Psi \neq \mathcal{J}\)) will steer the system into high-cost regions of the state space. Only when \(\Psi\) coincides with the value function \(\mathcal{J}\) does the local control correctly anticipate the future dynamics, ensuring global optimality.
\end{remark}

The variational principle derived in~\Cref{prop:simulation_free_objective} offers two distinct strategies for constructing the control law. The first strategy is to treat the functional \(\mathcal{L}(\bm{u}, \Psi)\) as an objective for global optimization. By parameterizing the velocity field \(\bm{u}\) within a specific function class (such as neural networks), one can minimize the expectation of the cost over the source distribution. In the machine learning literature, this approach is referred to as training-time regularization, as exemplified by Physics-Informed Diffusion Models (PIDM)~\citep{bastek2025physics} and Physics-Based Flow Matching (PBFM)~\citep{baldan2025flow}. It yields a global velocity field \(\bm u\) tailored to a specific terminal constraint \(H\). While rigorous, it is computationally intensive, as modifying the constraint requires end-to-end retraining of the model. The second strategy, which we adopt in this work, assumes access to the reference velocity field \(\bm{b}^\star\) and seeks to construct an analytical approximation of the value function \(\mathcal{J}\). Rather than minimizing the global functional, we derive a local approximation \(\widehat{\Psi} \approx \mathcal{J}\) based on the geometry of the terminal cost and the reference flow. This yields a sampling-time scheme, where the control is synthesized on the fly as a correction to the reference dynamics. \Cref{alg:unified_sampling_guidance} summarizes the overall procedure. The solver subroutines are given in \Cref{alg:gd_solver,alg:gn_solver,alg:toc_solver} and derived in \Cref{subsubsec:gradient_descent_derivation,subsubsec:gauss_newton_derivation,subsec:computational_complexity_and_the_tocflow_scheme}. This approach avoids the cost of global optimization and allows for flexible adaptation to varying constraints without retraining. To derive these approximations, we now turn to the analysis of the HJB equation in the terminal co-moving frame.

\begin{algorithm}[htb]
\caption{Sampling-time unified guidance with modular solvers}
\small
\label{alg:unified_sampling_guidance}
\begin{algorithmic}[1]
\STATE \textbf{Input} initial state \(\bm x_{t_0}\sim\rho_0\), number of sampling steps \(N\), time grid \(0=t_0<t_1<\dots<t_N=1\), weight schedule \(\{\lambda_t\}_{t\in[0,1]}\), reference drift \(\bm b^\star(\cdot,\cdot)\), constraint vector \(\bm h\)
\STATE \textbf{Output} terminal sample \(\bm x_{t_N}\)
\FOR{\(n = 0,1,\dotsc,N-1\)}
\STATE \(\Delta t_n \gets t_{n+1}-t_n\)
\STATE \(s_n \gets \int_{t_n}^{1}\lambda_u^{-1}\mathrm{d}u\)
\STATE \(\nabla_{\bm x}\mathcal J(\bm x_{t_n},t_n) \gets \textsc{Solver}\)
\hfill\(//~\textsc{Solver}\in\{\textsc{GDSolver}, \textsc{GNSolver}, \textsc{TOCSolver},\dotsc\}\)
\STATE \(\bm x_{t_{n+1}} \gets \bm x_{t_n}
+ \Delta t_n\cdot\bm b^\star(\bm x_{t_n}, t_n)
- \lambda_{t_n}^{-1} \nabla_{\bm x}\mathcal J(\bm x_{t_n},t_n)\)
\ENDFOR
\STATE \textbf{return} \(\bm x_{t_N}\)
\end{algorithmic}
\end{algorithm}

\subsection{HJB equation in the terminal co-moving frame}
\label{subsec:hjb_equation_in_the_terminal_co_moving_frame}

The optimal control derived in~\Cref{sec:an_optimal_control_framework} depends on the spatial gradient \(\nabla_{\bm{x}}\mathcal{J}(\bm{x},t)\). In low-dimensional settings, the associated HJB equation~\eqref{eq:hjb} can be solved numerically using grid-based methods such as finite difference schemes~\citep{leveque2007finite} or semi-Lagrangian approximations~\citep{falcone2013semi}, which converge to the unique viscosity solution. Alternatively, one may employ value iteration on discretizations~\citep{kushner1990numerical} or the method of characteristics~\citep{evans2022partial} when the solution is smooth. However, the reliance on spatial meshes renders these methods computationally intractable for \(d > 3\). To overcome this, recent works have introduced mesh-free neural approximations, such as SOC-MartNet~\citep{cai2025soc} or the Deep Random Difference Method (DRDM)~\citep{cai2025deep}. Nevertheless, these approaches still face significant challenges when applied to generative modeling tasks where \(d \gg 1\).

To obtain a tractable approximation that respects the high-dimensional geometry, we introduce a coordinate transformation adapted to the characteristics of the vector field \(\bm{b}^\star\). We map the Eulerian state space to a Lagrangian frame anchored at the terminal time and reparameterize the temporal variable by the cumulative inverse penalty. This transformation eliminates the first-order linear transport term in the HJB equation, reducing it to a pure geometric evolution equation governed by the pullback metric of the reference flow.

Specifically, we define the \emph{terminal co-moving frame} \(\bm{y}(\bm{x},t)\) and the \emph{stretched time} \(s(t)\) via the relations
\begin{equation}
\bm{y}(\bm{x},t) \coloneqq \Phi^{\bm{b}^\star}_{t \to 1}(\bm{x})
\quad \text{and} \quad
s(t) \coloneqq \int_t^1 \frac{1}{\lambda_u} \mathrm{d}u.
\end{equation}
The spatial map \(\bm{y}\) cancels the first-order linear transport term in the HJB equation by viewing the dynamics from the reference frame. The temporal map \(s(t)\) absorbs the singular weight \(\lambda_t\), zooming in on the boundary layer near \(t=1\) where the control effort is concentrated.

\begin{proposition}[Hamilton--Jacobi--Bellman equation in the terminal co-moving frame]
\label{prop:terminal_frame_hjb}
Let \(\mathcal{J}(\bm{x},t)\) be a smooth solution of the Hamilton--Jacobi--Bellman equation~\eqref{eq:hjb} in the Eulerian frame. Define the transformed value function \(\widehat{\mathcal{J}}\) in the terminal co-moving frame by the relation \(\widehat{\mathcal{J}}(\bm{y}, s(t)) \coloneqq \mathcal{J}(\Phi^{\bm{b}^\star}_{1 \to t}(\bm{y}), t)\). Then \(\widehat{\mathcal{J}}\) satisfies the Hamilton--Jacobi (HJ) equation
\begin{equation}
\partial_s \widehat{\mathcal{J}}(\bm{y},s) + \frac{1}{2} \bigl\| \nabla_{\bm{y}} \widehat{\mathcal{J}}(\bm{y},s) \bigr\|_{\bm{G}_t(\bm{y})}^2 = 0, \quad \widehat{\mathcal{J}}(\bm{y},0) = H(\bm{y}),
\label{eq:transformed_hjb}
\end{equation}
where the squared norm is defined by the time-dependent pullback metric
\begin{equation}
\bm{G}_t(\bm{y}) \coloneqq D\Phi^{\bm{b}^\star}_{t \to 1}(\bm{x}) D\Phi^{\bm{b}^\star}_{t \to 1}(\bm{x})^\top \Big|_{\bm{x} = \Phi^{\bm{b}^\star}_{1 \to t}(\bm{y})}.
\end{equation}
Moreover, the spatial gradient in the original frame can be recovered via the pullback relation
\begin{equation}
\nabla_{\bm{x}}\mathcal{J}(\bm{x},t) = D\Phi^{\bm{b}^\star}_{t \to 1}(\bm{x})^\top \nabla_{\bm{y}} \widehat{\mathcal{J}}(\bm{y}, s(t)).
\label{eq:gradient_recovery}
\end{equation}
\end{proposition}

\begin{proof}[Proof of~\Cref{prop:terminal_frame_hjb}]
We apply the chain rule to the identity \(\mathcal{J}(\bm{x},t) = \widehat{\mathcal{J}}(\Phi^{\bm{b}^\star}_{t \to 1}(\bm{x}), s(t))\).
First, we compute the spatial gradient. Differentiating with respect to \(\bm{x}\) yields
\begin{equation}
\nabla_{\bm{x}} \mathcal{J}(\bm{x},t) = D\Phi^{\bm{b}^\star}_{t \to 1}(\bm{x})^\top \nabla_{\bm{y}} \widehat{\mathcal{J}}(\bm{y}, s(t)).
\end{equation}
This immediately establishes~\eqref{eq:gradient_recovery}. Using this relation, the quadratic Hamiltonian term transforms as
\begin{equation}
\|\nabla_{\bm{x}} \mathcal{J}(\bm{x},t)\|_2^2 = \big\langle D\Phi^{\bm{b}^\star}_{t \to 1}(\bm{x})^\top \nabla_{\bm{y}} \widehat{\mathcal{J}}(\bm{y}, s(t)), D\Phi^{\bm{b}^\star}_{t \to 1}(\bm{x})^\top \nabla_{\bm{y}} \widehat{\mathcal{J}}(\bm{y}, s(t)) \big\rangle.
\end{equation}
Moving the Jacobian to the second argument of the inner product allows us to identify the metric tensor \(\bm{G}_t(\bm{y})\). We have
\begin{equation}
\begin{aligned}
\|\nabla_{\bm{x}} \mathcal{J}(\bm{x},t)\|_2^2 &= \big\langle \nabla_{\bm{y}} \widehat{\mathcal{J}}(\bm{y}, s(t)), D\Phi^{\bm{b}^\star}_{t \to 1}(\bm{x}) D\Phi^{\bm{b}^\star}_{t \to 1}(\bm{x})^\top \nabla_{\bm{y}} \widehat{\mathcal{J}}(\bm{y}, s(t)) \big\rangle \\
&= \big\langle \nabla_{\bm{y}} \widehat{\mathcal{J}}(\bm{y}, s(t)), \bm{G}_t(\bm{y}) \nabla_{\bm{y}} \widehat{\mathcal{J}}(\bm{y}, s(t)) \big\rangle \\
&= \|\nabla_{\bm{y}} \widehat{\mathcal{J}}(\bm{y}, s(t))\|_{\bm{G}_t(\bm{y})}^2.
\end{aligned}
\end{equation}

Next, we compute the time derivative. By the chain rule, we have
\begin{equation}
\partial_t \mathcal{J}(\bm{x},t) = \nabla_{\bm{y}} \widehat{\mathcal{J}}(\bm{y}, s(t)) \cdot \partial_t \Phi^{\bm{b}^\star}_{t \to 1}(\bm{x}) + \partial_s \widehat{\mathcal{J}}(\bm{y}, s(t)) \cdot \dot{s}(t).
\end{equation}
From the definition of the flow map, the map \(\bm{x} \mapsto \Phi^{\bm{b}^\star}_{t \to 1}(\bm{x})\) is constant along the characteristic curves of \(\bm{b}^\star\), implying it satisfies the advection equation
\begin{equation}
\partial_t \Phi^{\bm{b}^\star}_{t \to 1}(\bm{x}) + D\Phi^{\bm{b}^\star}_{t \to 1}(\bm{x}) \bm{b}^\star(\bm{x},t) = 0.
\end{equation}
Thus, we can substitute \(\partial_t \Phi^{\bm{b}^\star}_{t \to 1}(\bm{x}) = -D\Phi^{\bm{b}^\star}_{t \to 1}(\bm{x}) \bm{b}^\star(\bm{x},t)\).
Using this relation and the definition \(\dot{s}(t) = -\lambda_t^{-1}\) in the expression for \(\partial_t \mathcal{J}(\bm{x},t)\) leads to
\begin{equation}
\partial_t \mathcal{J}(\bm{x},t) = - \bigl( D\Phi^{\bm{b}^\star}_{t \to 1}(\bm{x})^\top \nabla_{\bm{y}} \widehat{\mathcal{J}}(\bm{y}, s(t)) \bigr) \cdot \bm{b}^\star(\bm{x},t) - \lambda_t^{-1} \partial_s \widehat{\mathcal{J}}(\bm{y}, s(t)).
\end{equation}
Recognizing the term in parentheses as \(\nabla_{\bm{x}} \mathcal{J}(\bm{x},t)\), we simplify this to
\begin{equation}
\partial_t \mathcal{J}(\bm{x},t) = -\nabla_{\bm{x}} \mathcal{J}(\bm{x},t) \cdot \bm{b}^\star(\bm{x},t) - \lambda_t^{-1} \partial_s \widehat{\mathcal{J}}(\bm{y}, s(t)).
\end{equation}

Finally, we substitute these expressions into the original HJB equation
\begin{equation}
-\partial_t \mathcal{J}(\bm{x},t) - \nabla_{\bm{x}} \mathcal{J}(\bm{x},t) \cdot \bm{b}^\star(\bm{x},t) + \frac{1}{2\lambda_t} \|\nabla_{\bm{x}} \mathcal{J}(\bm{x},t)\|_2^2 = 0.
\end{equation}
The substitution results in
\begin{multline}
-\bigl( -\nabla_{\bm{x}} \mathcal{J}(\bm{x},t) \cdot \bm{b}^\star(\bm{x},t) - \lambda_t^{-1} \partial_s \widehat{\mathcal{J}}(\bm{y}, s(t)) \bigr) \\
- \nabla_{\bm{x}} \mathcal{J}(\bm{x},t) \cdot \bm{b}^\star(\bm{x},t) + \frac{1}{2\lambda_t} \|\nabla_{\bm{y}} \widehat{\mathcal{J}}(\bm{y}, s(t))\|_{\bm{G}_t(\bm{y})}^2 = 0.
\end{multline}
The transport terms \(\nabla_{\bm{x}} \mathcal{J}(\bm{x},t) \cdot \bm{b}^\star(\bm{x},t)\) cancel exactly. Multiplying the remaining terms by \(\lambda_t\) yields the result~\eqref{eq:transformed_hjb}. The terminal condition follows from \(s(1)=0\) and \(\bm{y}(\bm{x},1)=\bm{x}\).
\end{proof}

In the HJ eqaution~\eqref{eq:transformed_hjb}, the metric \(\bm{G}_t\) captures the cumulative expansion and contraction of the reference flow. Directions where \(\bm{G}_t\) is large correspond to directions that have been stretched by the flow, making control in those directions ``cheaper'' in the terminal frame. This geometric structure will be central to designing the approximation schemes in the next subsection.

\begin{remark}[Initial vs.~terminal co-moving frames]
The choice of reference frame is determined by the direction of information flow required for the task. In~\Cref{sec:an_optimal_control_framework}, we sought to quantify the total transport cost relative to the source distribution \(\rho_0\). Pulling back to \(t=0\) allowed us to isolate the displacement caused by the control from the reference motion. In this section, we must compute the optimal control at time \(t\) based on a boundary condition known only at \(t=1\) (i.e., \(\mathcal{J}(\bm x, 1) = H(\bm x)\)). The terminal co-moving frame allows us to fix the cost function \(H\) on a static coordinate grid \(\bm{y}\) and propagate the geometry backwards to the current state, thereby solving the HJB equation in reverse.
\end{remark}

\subsection{Constant-metric reduction and proximal operators}
\label{subsec:constant_metric_reduction_and_proximal_operators}

We recall that the Hamilton--Jacobi (HJ) equation derived in~\Cref{prop:terminal_frame_hjb} is given by
\begin{equation}
\partial_s \widehat{\mathcal{J}}(\bm{y},s) + \frac{1}{2} \bigl\| \nabla_{\bm{y}} \widehat{\mathcal{J}}(\bm{y},s) \bigr\|_{\bm{G}_t(\bm{y})}^2 = 0, \quad \widehat{\mathcal{J}}(\bm{y},0) = H(\bm{y}),
\end{equation}
where
\begin{equation}
\bm{G}_t(\bm{y}) \coloneqq D\Phi^{\bm{b}^\star}_{t \to 1}(\bm{x}) D\Phi^{\bm{b}^\star}_{t \to 1}(\bm{x})^\top \Big|_{\bm{x} = \Phi^{\bm{b}^\star}_{1 \to t}(\bm{y})}.
\end{equation}
The analytical difficulty arises from the time-dependence and spatial variability of the pullback metric \(\bm{G}_t(\bm{y})\). However, since numerical integration proceeds in discrete time steps, we may employ a splitting strategy that freezes the geometry over the integration interval. That is, we approximate the metric as constant over the remaining budget \(s\), fixing it at its current value \(\bm{G} \coloneqq \bm{G}_t(\bm{y})\). In~\Cref{prop:constant_metric_reduction} we show that under this frozen-geometry approximation the HJ equation admits an explicit Hopf--Lax solution together with a closed-form expression for the spatial gradient of the value function.

\begin{proposition}[Constant-metric reduction and gradient recovery]
\label{prop:constant_metric_reduction}
Under the assumption that the metric tensor is constant, i.e., \(\bm{G}_t(\bm{y}) \equiv \bm{G} \succ 0\), the Hamilton--Jacobi equation~\eqref{eq:transformed_hjb} admits an explicit solution given by the Hopf--Lax formula~\citep{evans2022partial}
\begin{equation}
\widehat{\mathcal{J}}(\bm{y},s) = \inf_{\bm{z} \in \mathbb{R}^d} \Bigl\{ H(\bm{z}) + \frac{1}{2s} \| \bm{z} - \bm{y} \|_{\bm{G}^{-1}}^2 \Bigr\}.
\label{eq:hopf_lax_formula}
\end{equation}
Furthermore, at any point \(\bm{y}\) where the optimization problem in~\eqref{eq:hopf_lax_formula} admits a unique minimizer \(\bm{y}^+(\bm{y})\), the value function is differentiable and its spatial gradient satisfies the Moreau--Yosida identity~\citep{rockafellar1998variational}
\begin{equation}
\nabla_{\bm{y}} \widehat{\mathcal{J}}(\bm{y},s) = \frac{1}{s} \bm{G}^{-1} (\bm{y} - \bm{y}^+(\bm{y})).
\label{eq:moreau_yosida_identity}
\end{equation}
Consequently, the gradient of the value function in the original Eulerian frame is given by
\begin{equation}
\nabla_{\bm{x}} \mathcal{J}(\bm{x},t) = \frac{1}{s(t)} D\Phi^{\bm{b}^\star}_{t \to 1}(\bm{x})^\top \bm{G}^{-1} (\bm{y} - \bm{y}^+(\bm{y})).
\label{eq:grad_x_proximal}
\end{equation}
\end{proposition}

\begin{proof}[Proof of~\Cref{prop:constant_metric_reduction}]
Substituting the constant metric \(\bm{G}\) into the HJ equation~\eqref{eq:transformed_hjb}, the PDE takes the standard form
\begin{equation}
\partial_s \widehat{\mathcal{J}}(\bm{y},s) + \mathcal{H}(\nabla_{\bm{y}} \widehat{\mathcal{J}}(\bm{y},s)) = 0,
\end{equation}
where the Hamiltonian \(\mathcal{H}(\bm{p}) \coloneqq \frac{1}{2} \| \bm{p} \|_{\bm{G}}^2 = \frac{1}{2} \langle \bm{p}, \bm{G}\bm{p} \rangle\) is convex and independent of the spatial variable \(\bm{y}\). For such Hamiltonians, the unique solution is given by the Hopf--Lax formula~\citep{evans2022partial}
\begin{equation}
\widehat{\mathcal{J}}(\bm{y},s) = \inf_{\bm{z} \in \mathbb{R}^d} \Bigl\{ \widehat{\mathcal{J}}(\bm{z},0) + s \mathcal{L}\Bigl( \frac{\bm{y}-\bm{z}}{s} \Bigr) \Bigr\},
\end{equation}
where \(\mathcal{L}\) is the Legendre--Fenchel transform (Lagrangian) of \(\mathcal{H}\), defined by
\begin{equation}
\mathcal{L}(\bm{v}) \coloneqq \sup_{\bm{p} \in \mathbb{R}^d} \{ \langle \bm{p}, \bm{v} \rangle - \mathcal{H}(\bm{p}) \}.
\end{equation}
We compute \(\mathcal{L}\) explicitly by maximizing the quadratic form \(\langle \bm{p}, \bm{v} \rangle - \frac{1}{2}\langle \bm{p}, \bm{G}\bm{p} \rangle\). The first-order optimality condition with respect to \(\bm{p}\) yields \(\bm{v} - \bm{G}\bm{p} = 0\), or \(\bm{p}^\star = \bm{G}^{-1}\bm{v}\). Substituting this maximizer back into the Legendre transform expression yields
\begin{equation}
\mathcal{L}(\bm{v}) = \langle \bm{G}^{-1}\bm{v}, \bm{v} \rangle - \frac{1}{2} \langle \bm{G}^{-1}\bm{v}, \bm{G}\bm{G}^{-1}\bm{v} \rangle = \frac{1}{2} \langle \bm{v}, \bm{G}^{-1}\bm{v} \rangle = \frac{1}{2} \| \bm{v} \|_{\bm{G}^{-1}}^2.
\end{equation}
Substituting this Lagrangian and the initial condition \(\widehat{\mathcal{J}}(\bm{y},0) = H(\bm{y})\) into the general formula recovers the Hopf--Lax formula~\eqref{eq:hopf_lax_formula}.

To derive the gradient identity, we invoke the envelope theorem~\citep{bonnans2013perturbation}, which guarantees that for a parameterized minimization problem with a unique solution, the gradient of the value function equals the partial gradient of the objective evaluated at the minimizer. Let
\begin{equation}
\mathcal{F}(\bm{y}, \bm{z}) \coloneqq H(\bm{z}) + \frac{1}{2s} \langle \bm{y} - \bm{z}, \bm{G}^{-1} (\bm{y} - \bm{z}) \rangle
\end{equation}
be the objective function inside the infimum. Under the assumption that the minimizer \(\bm{y}^+(\bm y)\) is unique, \(\widehat{\mathcal{J}}\) is differentiable at \(\bm{y}\) and its gradient is given by the partial gradient of \(\mathcal{F}\) evaluated at the optimum
\begin{equation}
\nabla_{\bm{y}} \widehat{\mathcal{J}}(\bm{y},s) = \nabla_{\bm{y}} \mathcal{F}(\bm{y}, \bm{z}) \Big|_{\bm{z}=\bm{y}^+}.
\end{equation}
Computing the partial derivative of the quadratic term with respect to \(\bm{y}\) yields
\begin{equation}
\nabla_{\bm{y}} \Bigl( \frac{1}{2s} \langle \bm{y} - \bm{z}, \bm{G}^{-1} (\bm{y} - \bm{z}) \rangle \Bigr) = \frac{1}{s} \bm{G}^{-1} (\bm{y} - \bm{z}).
\end{equation}
Evaluating this at \(\bm{z} = \bm{y}^+(\bm y)\) confirms the identity~\eqref{eq:moreau_yosida_identity}.

Finally, to recover the gradient in the original coordinates, we recall the pullback relation \(\nabla_{\bm{x}}\mathcal{J}(\bm{x},t) = D\Phi^{\bm{b}^\star}_{t \to 1}(\bm{x})^\top \nabla_{\bm{y}} \widehat{\mathcal{J}}(\bm{y}, s(t))\) established in~\Cref{prop:terminal_frame_hjb}. Substituting the expression for \(\nabla_{\bm{y}} \widehat{\mathcal{J}}\) from~\eqref{eq:moreau_yosida_identity} yields the stated result~\eqref{eq:grad_x_proximal}.
\end{proof}

Crucially, \Cref{prop:constant_metric_reduction} implies that computing the high-dimensional optimal control \(\nabla_{\bm{x}}\mathcal{J}\) reduces entirely to finding the minimizer \(\bm{y}^+(\bm{y})\). This requires solving the proximal optimization problem
\begin{equation}
\bm{y}^+(\bm{y}) = \underset{\bm{z} \in \mathbb{R}^d}{\arg\min} \Bigl\{ H(\bm{z}) + \frac{1}{2s} \| \bm{z} - \bm{y} \|_{\bm{G}^{-1}}^2 \Bigr\}.
\label{eq:proximal_problem_def}
\end{equation}
Solving this global minimization exactly is generally intractable, especially when the terminal cost \(H\) is non-convex. We therefore propose generic approximation schemes that rely only on local information (gradients) of \(H\) and do not assume specific structural properties such as convexity.

\begin{remark}[Structure-exploiting solvers]
In specific applications where the terminal cost \(H\) possesses favourable structure such as convexity or decomposability, specialized numerical algorithms may be employed to solve the minimization problem~\eqref{eq:proximal_problem_def} efficiently or even exactly. For instance, if \(H\) is the indicator function of a convex set, the problem reduces to a metric projection. However, the development of such problem-specific solvers falls outside the scope of this work.
\end{remark}

\subsubsection{First-order Euclidean approximation (Gradient Descent)}
\label{subsubsec:gradient_descent_derivation}

We begin with the most elementary approximation scheme. As stated in~\Cref{ass:gd_approximation}, this approach relies on a local linearization of the terminal cost function combined with a simplification of the underlying geometry that treats the pullback metric as Euclidean.

\begin{assumption}[First-order Euclidean approximation]
\label{ass:gd_approximation}
We approximate the terminal cost \(H(\bm{z})\) by its first-order Taylor expansion around the current state \(\bm{y}\), and we approximate the pullback metric tensor \(\bm{G}\) by the Euclidean identity matrix:
\begin{align}
H(\bm{z}) &\approx H(\bm{y}) + \nabla_{\bm{y}} H(\bm{y}) \cdot (\bm{z} - \bm{y}), \\
\bm{G}^{-1} &\approx \bm{I}.
\end{align}
\end{assumption}

Under~\Cref{ass:gd_approximation}, the proximal optimization problem defined in~\eqref{eq:proximal_problem_def} simplifies to a quadratic minimization in Euclidean space:
\begin{equation}
\min_{\bm{z} \in \mathbb{R}^d} \Bigl\{ H(\bm{y}) + \nabla_{\bm{y}} H(\bm{y}) \cdot (\bm{z} - \bm{y}) + \frac{1}{2s} \| \bm{z} - \bm{y} \|_2^2 \Bigr\}.
\end{equation}
The unique minimizer \(\bm{y}^+\) (we omit its dependence on \(\bm y\) for conciseness) is obtained by solving the first-order optimality condition with respect to \(\bm{z}\), which yields the explicit update
\begin{equation}
\bm{y}^+ = \bm{y} - s \nabla_{\bm{y}} H(\bm{y}).
\end{equation}
We recover the value function gradient in the terminal frame using the Moreau--Yosida identity~\eqref{eq:moreau_yosida_identity} with \(\bm{G}=\bm{I}\):
\begin{equation}
\nabla_{\bm{y}} \widehat{\mathcal{J}}(\bm{y},s) \approx \frac{1}{s} (\bm{y} - \bm{y}^+) = \nabla_{\bm{y}} H(\bm{y}).
\end{equation}
Mapping this back to the Eulerian frame via the pullback relation~\eqref{eq:grad_x_proximal}, we obtain the approximate optimal control
\begin{equation}
\nabla_{\bm{x}} \mathcal{J}(\bm{x},t) \approx D\Phi^{\bm{b}^\star}_{t \to 1}(\bm{x})^\top \nabla_{\bm{y}} H(\Phi^{\bm{b}^\star}_{t \to 1}(\bm{x})).
\label{eq:gd_guidance_final}
\end{equation}
We refer to this scheme as \textbf{Gradient Descent (GD)} (summarized as \textsc{GDSolver} in~\Cref{alg:gd_solver}), as it performs a steepest descent step in the flat Euclidean metric. This scheme is computationally inexpensive (scaling linearly with dimension \(d\)) but treats the control cost as isotropic. It fails to account for the fact that the reference flow may stretch space significantly in certain directions, potentially leading to unstable behavior when the condition number of \(\bm{G}_t\) is large.

\begin{algorithm}[htb]
\caption{\textsc{GDSolver}}
\small
\label{alg:gd_solver}
\begin{algorithmic}[1]
\STATE \textbf{Input} current state \(\bm x_t\), time \(t\),
reference flow \(\Phi_{t\to 1}^{\bm b^\star}\), constraint vector \(\bm h\)
\STATE \textbf{Output} guidance control \(\nabla_{\bm x}\mathcal J(\bm x_t,t)\)
\STATE Compute
\[
\bm g \gets \nabla_{\bm x}\bigl(H \circ \Phi_{t\to 1}^{\bm b^\star}\bigr)(\bm x_t),\quad\text{where~}H(\bm z)=\frac{1}{2}\|\bm h(\bm z)\|_2^2
\]
using reverse mode differentiation
\STATE \textbf{return} \(\nabla_{\bm x}\mathcal J(\bm x_t,t) \gets \bm g\)
\end{algorithmic}
\end{algorithm}

\begin{remark}[Connection to gradient guidance methods]
The control derived in~\eqref{eq:gd_guidance_final} recovers the gradient guidance methods employed in recent works such as DiffusionPDE~\citep{huang2024diffusionpde}. In these frameworks, the guidance signal is computed via a ``lookahead'' procedure. First, the current state \(\bm{x}_t\) is mapped to an estimated terminal state \(\hat{\bm{x}}_1 \approx \Phi^{\bm{b}^\star}_{t \to 1}(\bm{x_t})\) (often approximated by a single Euler step \(\hat{\bm{x}}_1 = \bm{x}_t + (1-t)\bm{b}^\star(\bm{x}_t,t)\)). Second, the terminal penalty \(H(\hat{\bm{x}}_1)\) is evaluated; finally, the gradient is taken with respect to the current state \(\bm{x}_t\). Mathematically, this operation calculates the gradient of the composition \(\nabla_{\bm{x}} (H \circ \Phi^{\bm{b}^\star}_{t \to 1})(\bm{x}_t)\). By the chain rule, this yields
\begin{equation}
\nabla_{\bm{x}} (H \circ \Phi^{\bm{b}^\star}_{t \to 1})(\bm{x}_t) = D\Phi^{\bm{b}^\star}_{t \to 1}(\bm{x}_t)^\top \nabla H(\hat{\bm{x}}_1),
\end{equation}
which is identical to the control~\eqref{eq:gd_guidance_final} derived under the first-order Euclidean assumption. This analysis reveals that such approaches implicitly assume a flat Riemannian geometry (\(\bm{G} \approx \bm{I}\)), ignoring the distortion introduced by the transport. In~\Cref{sec:representative_applications}, we implement the GD method as a baseline.
\end{remark}

\subsubsection{Second-order geometric approximation (Gauss--Newton)}
\label{subsubsec:gauss_newton_derivation}

To address the limitations of the Euclidean approximation, we construct a scheme that respects both the curvature of the constraint manifold and the Riemannian geometry of the flow.

\begin{assumption}[Second-order geometric approximation]
\label{ass:gn_approximation}
Recall that the terminal cost is \(H(\bm{z}) = \frac{1}{2}\|\bm{h}(\bm{z})\|_2^2\). We approximate the constraint map \(\bm{h}\) by its first-order linearization around \(\bm{y}\), while retaining the exact pullback metric \(\bm{G}\):
\begin{equation}
\bm{h}(\bm{z}) \approx \bm{h}(\bm{y}) + \bm{J}_{\bm h}(\bm{y})(\bm{z} - \bm{y}),
\end{equation}
where \(\bm{J}_{\bm h}(\bm{y}) \in \mathbb{R}^{r \times d}\) is the Jacobian of \(\bm h\).
\end{assumption}

Under~\Cref{ass:gn_approximation}, substituting the linearization into the proximal problem~\eqref{eq:proximal_problem_def} yields a weighted linear least-squares problem for the perturbation \(\delta \bm{y} \coloneqq \bm{z} - \bm{y}\):
\begin{equation}
\min_{\delta \bm{y} \in \mathbb{R}^d} \Bigl\{ \frac{1}{2} \| \bm{h}(\bm{y}) + \bm{J}_{\bm h}(\bm{y}) \delta \bm{y} \|_2^2 + \frac{1}{2s} \| \delta \bm{y} \|_{\bm{G}^{-1}}^2 \Bigr\}.
\end{equation}
Setting the gradient with respect to \(\delta \bm{y}\) to zero leads to the system of normal equations
\begin{equation}
\Bigl( \bm{J}_{\bm h}(\bm{y})^\top \bm{J}_{\bm h}(\bm{y}) + \frac{1}{s} \bm{G}^{-1} \Bigr) \delta \bm{y}^\star = - \bm{J}_{\bm h}(\bm{y})^\top \bm{h}(\bm{y}).
\end{equation}
We compute the value gradient via the Moreau--Yosida identity \(\nabla_{\bm{y}} \widehat{\mathcal{J}} = -\frac{1}{s} \bm{G}^{-1} \delta \bm{y}^\star\). Substituting \(\delta \bm{y}^\star = -s \bm{G} \nabla_{\bm{y}} \widehat{\mathcal{J}}\) into the normal equations allows us to solve directly for the gradient:
\begin{equation}
\Bigl( \bm{J}_{\bm h}(\bm{y})^\top \bm{J}_{\bm h}(\bm{y}) + \frac{1}{s} \bm{G}^{-1} \Bigr) (-s \bm{G} \nabla_{\bm{y}} \widehat{\mathcal{J}}) = - \bm{J}_{\bm h}(\bm{y})^\top \bm{h}(\bm{y}).
\end{equation}
Rearranging terms yields
\begin{equation}
\bigl( \bm{I}_d + s \bm{J}_{\bm h}(\bm{y})^\top \bm{J}_{\bm h}(\bm{y}) \bm{G} \bigr) \nabla_{\bm{y}} \widehat{\mathcal{J}} = \bm{J}_{\bm h}(\bm{y})^\top \bm{h}(\bm{y}).
\end{equation}
Using the matrix inversion identity \((\bm{I} + \bm{U}\bm{V})^{-1}\bm{U} = \bm{U}(\bm{I} + \bm{V}\bm{U})^{-1}\) with \(\bm{U} = \bm{J}_{\bm h}^\top\) and \(\bm{V} = s \bm{J}_{\bm h} \bm{G}\), we can push the inversion into the lower-dimensional constraint space (assuming \(r \leq d\)):
\begin{equation}
\nabla_{\bm{y}} \widehat{\mathcal{J}}(\bm{y},s) = \bm{J}_{\bm h}(\bm{y})^\top \bigl( \bm{I}_k + s \bm{J}_{\bm h}(\bm{y}) \bm{G} \bm{J}_{\bm h}(\bm{y})^\top \bigr)^{-1} \bm{h}(\bm{y}).
\label{eq:gn_update}
\end{equation}
We refer to this scheme as \textbf{Gauss--Newton (GN)} (summarized as \textsc{GNSolver} in~\Cref{alg:gn_solver}), as it corresponds to a damped step of the classical Gauss--Newton algorithm~\citep{nocedal2006numerical} in the Riemannian metric \(\bm{G}\). The term \(s \bm{J}_{\bm h} \bm{G} \bm{J}_{\bm h}^\top\) functions as a geometry-dependent Tikhonov regularizer in the dual constraint space. Analytically, the inverse operator acts as a spectral filter. In directions corresponding to the large eigenvalues of \(\bm{G}\) (where the reference flow exhibits significant expansion), the regularization term dominates the identity, thereby dampening the guidance signal. This ensures that the control effort remains bounded even in stiff directions where deviations are energetically costly.

\begin{algorithm}[htb]
\caption{\textsc{GNSolver}}
\small
\label{alg:gn_solver}
\begin{algorithmic}[1]
\STATE \textbf{Input} current state \(\bm x_t\), time \(t\), remaining budget \(s\),
reference flow \(\Phi_{t\to 1}^{\bm b^\star}\), constraint vector \(\bm h\),
CG tolerance \(\varepsilon\), maximum CG iterations \(K\)
\STATE \textbf{Output} guidance control \(\nabla_{\bm x}\mathcal J(\bm x_t,t)\)

\STATE \(\bm y \gets \Phi_{t\to 1}^{\bm b^\star}(\bm x_t)\)
\STATE \(\bm r \gets \bm h(\bm y)\)

\STATE Solve \((\bm I + s \bm M \bm M^\top)\bm \alpha = \bm r\) for \(\bm \alpha \in \mathbb R^r\) using Conjugate Gradient with tolerance \(\varepsilon\) and at most \(K\) iterations, where
\[
\bm M = \bm J_{\bm h}(\bm y)D\Phi_{t\to 1}^{\bm b^\star}(\bm x_t)
\]
and the matrix vector products required by CG are evaluated matrix free as
\begin{align*}
\bm M^\top \bm u &\gets D\Phi_{t\to 1}^{\bm b^\star}(\bm x_t)^\top \bm J_{\bm h}(\bm y)^\top \bm u, \tag*{// \text{ VJP}}
\\
\bm M \bm v &\gets \bm J_{\bm h}(\bm y)D\Phi_{t\to 1}^{\bm b^\star}(\bm x_t)\bm v \tag*{// \text{ JVP}}
\end{align*}
using PyTorch primitives VJP or JVP
\STATE \textbf{return} \(\nabla_{\bm x}\mathcal J(\bm x_t,t) \gets \bm M^\top \bm \alpha\)
\end{algorithmic}
\end{algorithm}

\begin{remark}[Connection to manifold projection methods]
The GN scheme derived in~\eqref{eq:gn_update} is conceptually similar to manifold projection methods used in recent works such as~\citep{cheng2025gradient,jacobsen2025cocogen,utkarsh2025physics}. Specifically speaking, these methods operate by ``looking ahead'' to the terminal state \(\bm y \approx \Phi_{t \to 1}^{\bm{b}^\star}(\bm{x})\), performing a constrained optimization step (typically a Gauss--Newton projection) to find a point \(\bm{y}^+\) on the manifold \(\mathcal{M}\) closest to \(\bm{y}\), and then pulling this correction vector back to the current time \(t\) via the linearized inverse flow. This strategy is applied continuously for all \(t\in[0, 1)\) in~\citep{cheng2025gradient,utkarsh2025physics}, or specifically near the terminal time \(t\approx1\) in~\citep{jacobsen2025cocogen}. In~\Cref{sec:representative_applications}, we also implement these projection-based methods as baselines.
\end{remark}

\subsection{Computational complexity and the TOCFlow scheme}
\label{subsec:computational_complexity_and_the_tocflow_scheme}

The practical utility of the derived guidance schemes depends critically on the computational cost of evaluating the control in high dimensions. In many generative modeling applications, the state dimension \(d\) is large, while the number of constraints \(r = \dim(\bm{h})\) varies from a single scalar to the order of \(d\).

To formalize the complexity, we define the linear sensitivity operator \(\bm{M} \in \mathbb{R}^{r \times d}\) as the composition of the constraint Jacobian and the flow tangent map
\begin{equation}
\bm{M} \coloneqq \bm{J}_{\bm h}(\bm{y}) D\Phi^{\bm{b}^\star}_{t \to 1}(\bm{x})\Big|_{\bm x = \Phi_{1\to t}^{\bm b^\star}(\bm y)}.
\end{equation}
We assume that these Jacobians are not instantiated as dense matrices but are accessed via automatic differentiation primitives. The application of the adjoint operator \(\bm{u} \mapsto \bm{M}^\top \bm{u}\) corresponds to a vector-Jacobian product (VJP) computed via reverse-mode differentiation, with a cost \(\mathcal{C}_{\mathrm{bwd}}\) comparable to one backward integration of the reference flow. Similarly, the forward application \(\bm{v} \mapsto \bm{M} \bm{v}\) corresponds to a Jacobian-vector product (JVP) computed via forward-mode differentiation, with a cost \(\mathcal{C}_{\mathrm{fwd}}\) comparable to one forward integration.

We compare the cost of evaluating the guidance update \(\nabla_{\bm{x}} \mathcal{J}\) for the GD scheme~\eqref{eq:gd_guidance_final} and GN scheme~\eqref{eq:gn_update} established in the previous subsection.

The first-order Euclidean approximation (i.e, the GD scheme) derived in~\Cref{subsubsec:gradient_descent_derivation} requires computing the pullback of the terminal cost gradient \(\bm{M}^\top \nabla_{\bm{y}} H(\bm{y})\). Since \(\nabla_{\bm{y}} H(\bm{y}) = \bm{J}_{\bm h}(\bm{y})^\top \bm{h}(\bm{y})\), this operation is equivalent to a single VJP applied to the scalar objective \(H \circ \Phi_{t\to 1}^{\bm b^\star}\). The computational cost scales as \(O(1)\) relative to the flow integration.

The second-order geometric approximation (i.e., the GN scheme) derived in~\Cref{subsubsec:gauss_newton_derivation} necessitates solving the linear system \(\bm{A} \bm{\alpha} = \bm{h}(\bm{y})\), where \(\bm{A} \coloneqq \bm{I}_r + s \bm{M}\bm{M}^\top\) is the regularized Gram matrix of the sensitivity operator. Constructing \(\bm{A}\) explicitly requires \(r\) separate VJP evaluations to populate the rows of \(\bm{M}\), which renders the scheme prohibitively expensive for high-dimensional constraints. Alternatively, the system may be solved iteratively using the Conjugate Gradient method~\citep{trefethen2022numerical} without explicit matrix formation. Each iteration involves the operator action \(\bm{v} \mapsto \bm{M}(\bm{M}^\top \bm{v})\), requiring one VJP followed by one JVP. Due to the typically large condition number of the flow-induced metric, the number of iterations required for convergence can be substantial, making the GN scheme significantly more computationally intensive than the GD scheme.

This analysis highlights a dichotomy. The GD scheme is efficient but geometrically isotropic, while the GN scheme is geometrically consistent but computationally demanding. To resolve this trade-off, we propose \textbf{Terminal Optimal Control with Flow-based models (TOCFlow)} (summarized as \textsc{TOCSolver} in~\Cref{alg:toc_solver}). Instead of optimizing over the full space \(\mathbb{R}^d\), we restrict the optimization to a specific one-dimensional subspace. This reduces the problem to a scalar line search, which admits a closed-form solution.

We recall the fundamental proximal optimization problem~\eqref{eq:hopf_lax_formula} defining the value function in the terminal frame
\begin{equation}
\bm{y}^+ \in \underset{\bm{z} \in \mathbb{R}^d}{\arg\min} \Bigl\{ H(\bm{z}) + \frac{1}{2s} \| \bm{z} - \bm{y} \|_{\bm{G}^{-1}}^2 \Bigr\}.
\end{equation}
The natural gradient associated with the metric \(\bm{G}^{-1}\) is given by \(\nabla_{\bm{G}^{-1}} H(\bm{y}) = \bm{G} \nabla_{\bm{y}} H(\bm{y})\). Let \(\bm{g} \coloneqq \nabla_{\bm{y}} H(\bm{y})\) denote the standard Euclidean gradient. We define the search direction \(\bm{d} \coloneqq \bm{G} \bm{g}\) and posit the ansatz
\begin{equation}
\bm{z}(\tau) = \bm{y} - s \tau \bm{d}, \quad \tau \in \mathbb{R}.
\end{equation}
This choice is structurally advantageous because substituting it into the gradient recovery formula \(\nabla_{\bm{y}} \widehat{\mathcal{J}} = \frac{1}{s} \bm{G}^{-1}(\bm{y} - \bm{z})\) yields a scalar scaling of the Euclidean gradient, \(\nabla_{\bm{y}} \widehat{\mathcal{J}}(\bm{y},s) = \tau \bm{g}\). Consequently, the high-dimensional vector optimization reduces to finding the optimal scalar damping factor \(\tau\) for the standard gradient descent update.

Substituting the ansatz \(\bm{z}(\tau)\) into the proximal objective function yields a scalar function. The regularization term simplifies algebraically due to the choice of direction as
\begin{equation}
\frac{1}{2s} \| \bm{z}(\tau) - \bm{y} \|_{\bm{G}^{-1}}^2 = \frac{1}{2s} \langle -s \tau \bm{G}\bm{g}, \bm{G}^{-1}(-s \tau \bm{G}\bm{g}) \rangle = \frac{s \tau^2}{2} \langle \bm{g}, \bm{G}\bm{g} \rangle.
\end{equation}
This term represents the energy cost of moving in the direction \(\bm{d}\). For the target cost \(H(\bm{z}(\tau))\), exact evaluation requires querying the constraint function at a new point, which is expensive inside a derivative calculation. Instead, we introduce a local model. Since \(H(\bm{y}) = \frac{1}{2}\|\bm{h}(\bm{y})\|_2^2\), we model the behavior of the constraint residual \(\bm{h}\) along the descent direction.

\begin{assumption}[Linear constraint decay]
\label{ass:linear_decay}
We assume that along the Riemannian gradient direction, the constraint vector \(\bm{h}\) decays linearly at a rate \(\rho \in \mathbb{R}\). Specifically, we posit the approximation
\begin{equation}
\bm{h}(\bm{z}(\tau)) \approx (1 - s \tau \rho) \bm{h}(\bm{y}).
\end{equation}
\end{assumption}

\begin{remark}[Justification of~\Cref{ass:linear_decay}]
This approximation is motivated by the first-order Taylor expansion along the search direction, which gives
\begin{equation}
\bm h(\bm z(\tau)) \approx \bm h(\bm y) - s\tau\bm J_{\bm h}(\bm y)\bm d.
\end{equation}
\Cref{ass:linear_decay} further approximates the first-order change by a scalar contraction along the current residual, namely \(\bm J_{\bm h}(\bm y)\bm d \approx \rho\bm h(\bm y)\), for example with \(\rho\) chosen as the Rayleigh quotient along \(\bm h(\bm y)\) (as derived below). The relation is exact for a single scalar constraint \(r=1\). For general nonlinear constraints, the assumption serves as a local surrogate as \(s\tau \to 0\), effectively replacing \(\bm h\) by its local linearization.
\end{remark}

Under~\Cref{ass:linear_decay}, the terminal cost is approximated by the quadratic model
\begin{equation}
H(\bm{z}(\tau)) \approx \frac{1}{2} \| (1 - s \tau \rho) \bm{h}(\bm{y}) \|_2^2 = \frac{1}{2} (1 - s \tau \rho)^2 \| \bm{h}(\bm{y}) \|_2^2.
\end{equation}
To determine the decay rate \(\rho\), we require this model to be consistent with the true geometry to first order. We match the derivative of the model at \(\tau=0\) with the true directional derivative of the cost function. The true directional derivative is
\begin{equation}
\frac{\mathrm{d}}{\mathrm{d}\tau} H(\bm{z}(\tau)) \Big|_{\tau=0} = \nabla_{\bm{y}} H(\bm{y})^\top \frac{\mathrm{d}\bm{z}}{\mathrm{d}\tau} = \bm{g}^\top (-s \bm{G}\bm{g}) = -s \bm{g}^\top \bm{G} \bm{g}.
\end{equation}
The derivative of the model is 
\begin{equation}
\frac{\mathrm{d}}{\mathrm{d}\tau} \Bigl( \frac{1}{2} (1 - s \tau \rho)^2 \| \bm{h}(\bm{y}) \|_2^2 \Bigr) \big|_{\tau=0} = -s \rho \| \bm{h}(\bm{y}) \|_2^2.    
\end{equation}
Equating these two expressions yields the consistency condition for \(\rho\)
\begin{equation}
\rho = \frac{\bm{g}^\top \bm{G} \bm{g}}{\| \bm{h}(\bm{y}) \|_2^2}.
\end{equation}
This quantity \(\rho\) can be interpreted as a generalized Rayleigh quotient, measuring the sensitivity of the cost function relative to the magnitude of the constraint violation in the metric space.

We now solve for the optimal \(\tau\) by minimizing the total scalarized objective, which is the sum of the approximated target cost and the transport cost
\begin{equation}
\frac{1}{2} (1 - s \tau \rho)^2 \| \bm{h}(\bm{y}) \|_2^2 + \frac{s \tau^2}{2} \bm{g}^\top \bm{G} \bm{g}.
\end{equation}
Substituting the identity \(\bm{g}^\top \bm{G} \bm{g} = \rho \| \bm{h}(\bm{y}) \|_2^2\) into the second term allows us to factor out the squared norm of the residual, reducing the objective to a form proportional to \(\frac{1}{2}(1 - s \tau \rho)^2 + \frac{1}{2}s \tau^2 \rho\). We find the minimum by setting the derivative with respect to \(\tau\) to zero:
\begin{equation}
-(1 - s \tau \rho)(s \rho) + s \tau \rho = s \rho \bigl( -(1 - s \tau \rho) + \tau \bigr) = 0.
\end{equation}
Assuming \(\rho \neq 0\), we obtain the algebraic equation \(\tau(1 + s \rho) = 1\). The optimal damping factor is thus
\begin{equation}
\tau^\star = \frac{1}{1 + s \rho} = \frac{\| \bm{h}(\bm{y}) \|_2^2}{\| \bm{h}(\bm{y}) \|_2^2 + s\bm{g}^\top \bm{G} \bm{g}}.
\label{eq:dampling_factor}
\end{equation}

The resulting feedback law in the original state space is \(\nabla_{\bm{x}} \mathcal{J}(\bm{x},t) \approx \tau^\star \nabla_{\bm{x}} (H \circ \Phi^{\bm{b}^\star}_{t \to 1})(\bm{x})\). Crucially, the scalar \(\tau^\star\) can be computed efficiently without forming any dense matrices. The quadratic form appearing in the denominator is equivalent to the squared norm of the gradient in the \(\bm{x}\)-frame
\begin{equation}
\begin{aligned}
\bm{g}^\top \bm{G} \bm{g} &= \nabla_{\bm{y}} H(\bm{y})^\top \bigl( D\Phi^{\bm{b}^\star}_{t \to 1}(\bm{x}) D\Phi^{\bm{b}^\star}_{t \to 1}(\bm{x})^\top \bigr) \nabla_{\bm{y}} H(\bm{y}) \\
&= \bigl\| D\Phi^{\bm{b}^\star}_{t \to 1}(\bm{x})^\top \nabla_{\bm{y}} H(\bm{y}) \bigr\|_2^2 \\
&= \bigl\| \nabla_{\bm{x}} ( H \circ \Phi^{\bm{b}^\star}_{t \to 1} )(\bm{x}) \bigr\|_2^2.
\end{aligned}
\end{equation}
This quantity requires only a single VJP to evaluate. Therefore, TOCFlow achieves a geometry-aware correction with a computational cost of \(O(1)\), identical to that of the GD scheme and independent of the number of constraints.

\begin{algorithm}[htb]
\caption{\textsc{TOCSolver}}
\small
\label{alg:toc_solver}
\begin{algorithmic}[1]
\STATE \textbf{Input} current state \(\bm x_t\), time \(t\), remaining budget \(s\), reference flow \(\Phi_{t\to 1}^{\bm b^\star}\), constraint vector \(\bm h\)
\STATE \textbf{Output} guidance control \(\nabla_{\bm x}\mathcal J(\bm x_t,t)\)

\STATE \(\bm y \gets \Phi_{t\to 1}^{\bm b^\star}(\bm x_t)\)
\STATE \(\bm r \gets \bm h(\bm y)\), \(r^2 \gets \|\bm r\|_2^2\)
\STATE Compute
\[
\bm g \gets \nabla_{\bm x}\bigl(H \circ \Phi_{t\to 1}^{\bm b^\star}\bigr)(\bm x_t),\quad\text{where}~H(\bm z)=\dfrac{1}{2}\|\bm h(\bm z)\|_2^2
\]
using reverse mode differentiation

\IF{\(r^2 = 0\)}
\STATE \textbf{return} \(\nabla_{\bm x}\mathcal J(\bm x_t,t) \gets \bm 0\)
\ENDIF
\STATE \(g^2 \gets \|\bm g\|_2^2\), \(\rho \gets g^2 / r^2\)
\STATE \(\tau^\star \gets (1 + s \rho)^{-1}\)
\STATE \textbf{return} \(\nabla_{\bm x}\mathcal J(\bm x_t,t) \gets \tau^\star \bm g\)
\end{algorithmic}
\end{algorithm}

\subsection{Analytical comparison in the Gaussian setting}
\label{subsec:analytical_comparison_in_the_gaussian_setting}

To verify the accuracy of the proposed approximation schemes in~\Cref{subsec:constant_metric_reduction_and_proximal_operators,subsec:computational_complexity_and_the_tocflow_scheme}, we revisit the one-dimensional Gaussian model problem introduced in~\Cref{subsec:exact_solution_for_the_gaussian_model}. In this setting, the control induced by GD, GN, and TOCFlow can be analyzed explicitly, as detailed in~\Cref{prop:approximate_moments}.

\begin{proposition}
[Analytical comparison in the Gaussian setting]
\label{prop:approximate_moments}
Consider the first-coordinate Gaussian subproblem induced by the \(2\)-dimensional setting in~\Cref{prop:gaussian_exact}, with source \(\mathcal N(0,1)\) and first-coordinate unconstrained target \(\mathcal N(\mu,\sigma^2)\). Let \(v(t)^2\coloneqq(1-t)^2 + t^2\sigma^2\), and let \(\gamma_{\bm{\lambda}} \coloneqq \int_0^1 \frac{\sigma^2}{\lambda_t v(t)^2} \mathrm{d}t\). Then terminal densities generated by the Gradient Descent (GD), Gauss--Newton (GN), and TOCFlow schemes are Gaussians \(\mathcal{N}(\tilde\mu, \tilde\sigma^2)\) characterized as follows:
\begin{enumerate}
\item Gradient Descent (GD): The GD scheme yields the moments
\begin{equation}
\tilde\mu_{\mathrm{GD}} = \mu \exp(-\gamma_{\bm{\lambda}}) \quad \text{and} \quad \tilde\sigma_{\mathrm{GD}} = \sigma \exp(-\gamma_{\bm{\lambda}}).
\end{equation}
\item Gauss--Newton (GN) and TOCFlow: In the one-dimensional setting, the GN and TOCFlow schemes coincide. They yield the moments
\begin{equation}
\tilde\mu_{\mathrm{TOC}} = \mu \exp(-\eta_{\bm{\lambda}}) \quad \text{and} \quad \tilde\sigma_{\mathrm{TOC}} = \sigma \exp(-\eta_{\bm{\lambda}}),
\end{equation}
where the effective stiffness \(\eta_{\bm{\lambda}}\) incorporates the geometric regularization via the stretched time \(s(t) = \int_t^1 \lambda_u^{-1} \mathrm{d}u\) as
\begin{equation}
\eta_{\bm{\lambda}} \coloneqq \int_0^1 \frac{\sigma^2}{\lambda_t (v(t)^2 + \sigma^2 s(t))} \mathrm{d}t.
\label{eq:tocflow_contraction}
\end{equation}
\end{enumerate}
\end{proposition}

\begin{proof}[Proof of~\Cref{prop:approximate_moments}]
We analyze the closed-loop dynamics induced by each scheme. Recall from~\Cref{prop:gaussian_exact} that the reference velocity field is affine: \(b^\star(x,t) = \alpha(t)x + \beta(t)\), where \(\alpha(t) = \dot{v}(t)/v(t)\).
The reference flow map \(\Phi_{t \to 1}^{\bm{b}^\star}(x)\) is linear. Using the moment matching property of the optimal transport interpolant, the map scales the centered process by the ratio of standard deviations:
\begin{equation}
y(x) \coloneqq \Phi_{t \to 1}^{\bm{b}^\star}(x) = \frac{\sigma}{v(t)}(x - \mu t) + \mu.
\end{equation}
Consequently, the Jacobian is scalar and spatially constant:
\begin{equation}
J_{t \to 1} \coloneqq D\Phi_{t \to 1}^{\bm{b}^\star}(x) = \frac{\sigma}{v(t)}.
\end{equation}
The terminal cost is \(H(y) = \frac{1}{2}y^2\), with gradient \(\nabla_y H(y) = y\).

\textbf{1. Gradient Descent (GD):}
The GD control law is defined by the pullback of the Euclidean gradient:
\begin{equation}
u_{\mathrm{GD}}(x,t) = b^\star(x,t) - \lambda_t^{-1} J_{t \to 1} \nabla_y H(y(x)).
\end{equation}
Substituting the expressions for \(J_{t \to 1}\) and \(\nabla_y H(y)\) yields
\begin{equation}
u_{\mathrm{GD}}(x,t) = \alpha(t)x + \beta(t) - \lambda_t^{-1} \frac{\sigma}{v(t)} \biggl( \frac{\sigma}{v(t)}(x - \mu t) + \mu \biggr).
\end{equation}
Isolating the terms dependent on \(x\), we define the effective feedback gain \(\kappa_{\mathrm{GD}}(t)\):
\begin{equation}
\kappa_{\mathrm{GD}}(t) \coloneqq \frac{\sigma^2}{\lambda_t v(t)^2}.
\end{equation}
The dynamics of the controlled process \(X_t\) are governed by the linear ODE
\begin{equation}
\dot{X}_t = (\alpha(t) - \kappa_{\mathrm{GD}}(t))X_t + C(t),
\end{equation}
where \(C(t)\) collects the terms independent of the state. The variance \(V_t \coloneqq \operatorname{Var}(X_t)\) evolves according to the homogeneous equation
\begin{equation}
\dot{V}_t = 2(\alpha(t) - \kappa_{\mathrm{GD}}(t))V_t.
\end{equation}
Substituting \(\alpha(t) = \dot{v}(t)/v(t)\) and dividing by \(V_t\), we obtain
\begin{equation}
\frac{\dot{V}_t}{V_t} = 2 \frac{\dot{v}(t)}{v(t)} - 2 \kappa_{\mathrm{GD}}(t).
\end{equation}
Integrating from \(t=0\) to \(1\) yields
\begin{equation}
\ln V_1 - \ln V_0 = 2 \ln \frac{v(1)}{v(0)} - 2 \int_0^1 \kappa_{\mathrm{GD}}(t) \mathrm{d}t.
\end{equation}
Using the boundary values \(V_0 = v(0)^2 = 1\) and \(v(1) = \sigma\), and identifying the integral as \(\gamma_{\bm{\lambda}}\), we obtain the terminal variance
\begin{equation}
\tilde{\sigma}_{\mathrm{GD}}^2 = V_1 = \sigma^2 \exp\biggl(-2 \int_0^1 \frac{\sigma^2}{\lambda_t v(t)^2} \mathrm{d}t \biggr) = \sigma^2 \exp(-2\gamma_{\bm{\lambda}}).
\end{equation}
The mean follows the same linear dynamics relative to the target origin, yielding the same scaling factor.

\textbf{2. TOCFlow / Gauss--Newton:}
In the scalar setting (\(d=r=1\)), the matrix inverse in the GN update becomes scalar division, coinciding with the TOCFlow damping factor \(\tau^\star\). The squared norm of the gradient is
\begin{equation}
\|\nabla_x (H \circ \Phi_{t \to 1}^{\bm{b}^\star})\|_2^2 = (J_{t \to 1} y)^2 = \frac{\sigma^2}{v(t)^2} y^2.
\end{equation}
The squared norm of the residual is \(\|\bm{h}(y)\|_2^2 = y^2\). Substituting these into the formula for \(\tau^\star\) from~\Cref{subsec:computational_complexity_and_the_tocflow_scheme} with the effective step size \(s(t)\):
\begin{equation}
\tau^\star(t) = \frac{y^2}{y^2 + s(t) (\frac{\sigma^2}{v(t)^2} y^2)} = \frac{v(t)^2}{v(t)^2 + s(t)\sigma^2}.
\end{equation}
The feedback gain \(\kappa_{\mathrm{TOC}}(t)\) is simply the GD gain scaled by this damping factor:
\begin{equation}
\kappa_{\mathrm{TOC}}(t) = \tau^\star(t) \kappa_{\mathrm{GD}}(t) = \biggl( \frac{v(t)^2}{v(t)^2 + s(t)\sigma^2} \biggr) \frac{\sigma^2}{\lambda_t v(t)^2} = \frac{\sigma^2}{\lambda_t (v(t)^2 + s(t)\sigma^2)}.
\end{equation}
Following the same integration steps as in the GD case, the log-variance ratio is determined by the integral of \(\kappa_{\mathrm{TOC}}(t)\):
\begin{equation}
\ln \frac{\tilde{\sigma}_{\mathrm{TOC}}^2}{\sigma^2} = -2 \int_0^1 \frac{\sigma^2}{\lambda_t (v(t)^2 + s(t)\sigma^2)} \mathrm{d}t.
\end{equation}
Identifying the integral as \(\eta_{\bm{\lambda}}\) completes the proof.
\end{proof}

\begin{remark}[Asymptotic consistency and over-compression]
We may quantify the accuracy of the approximations via the contraction factor, defined as the ratio of the terminal standard deviation to the reference, \(\tilde{\sigma}/\sigma\). According to~\Cref{prop:gaussian_exact}, the exact solution exhibits an algebraic decay \((1+\gamma_{\bm{\lambda}})^{-1}\), standing in contrast to the exponential decay \(\exp(-\gamma_{\bm{\lambda}})\) of the GD scheme. In the regime of weak penalties (where \(\lambda_t \to \infty\) implies \(\gamma_{\bm{\lambda}} \ll 1\)), these scalings are asymptotically consistent, coinciding up to the first order expansion \(1 - \gamma_{\bm{\lambda}} + \mathcal{O}(\gamma_{\bm{\lambda}}^2)\). However, in the strong penalty regime (\(\gamma_{\bm{\lambda}} \gg 1\)), the exponential factor vanishes significantly faster than the algebraic factor. This divergence indicates that the Euclidean approximation suffers from over-compression. The GN and TOCFlow schemes mitigate this through geometric regularization. The presence of the strictly positive term \(s(t)\sigma^2\) in the denominator of~\eqref{eq:tocflow_contraction} ensures that the effective stiffness satisfies \(\eta_{\bm{\lambda}} < \gamma_{\bm{\lambda}}\). Consequently, the resulting contraction factor \(\exp(-\eta_{\bm{\lambda}})\) decays more slowly than the GD solution, structurally correcting the trajectory to better approximate the exact algebraic behavior. See~\cref{fig:gaussian_comparison} for an illustration.
\end{remark}

\begin{figure}[htb]
\centering
\includegraphics[width=0.6\linewidth]{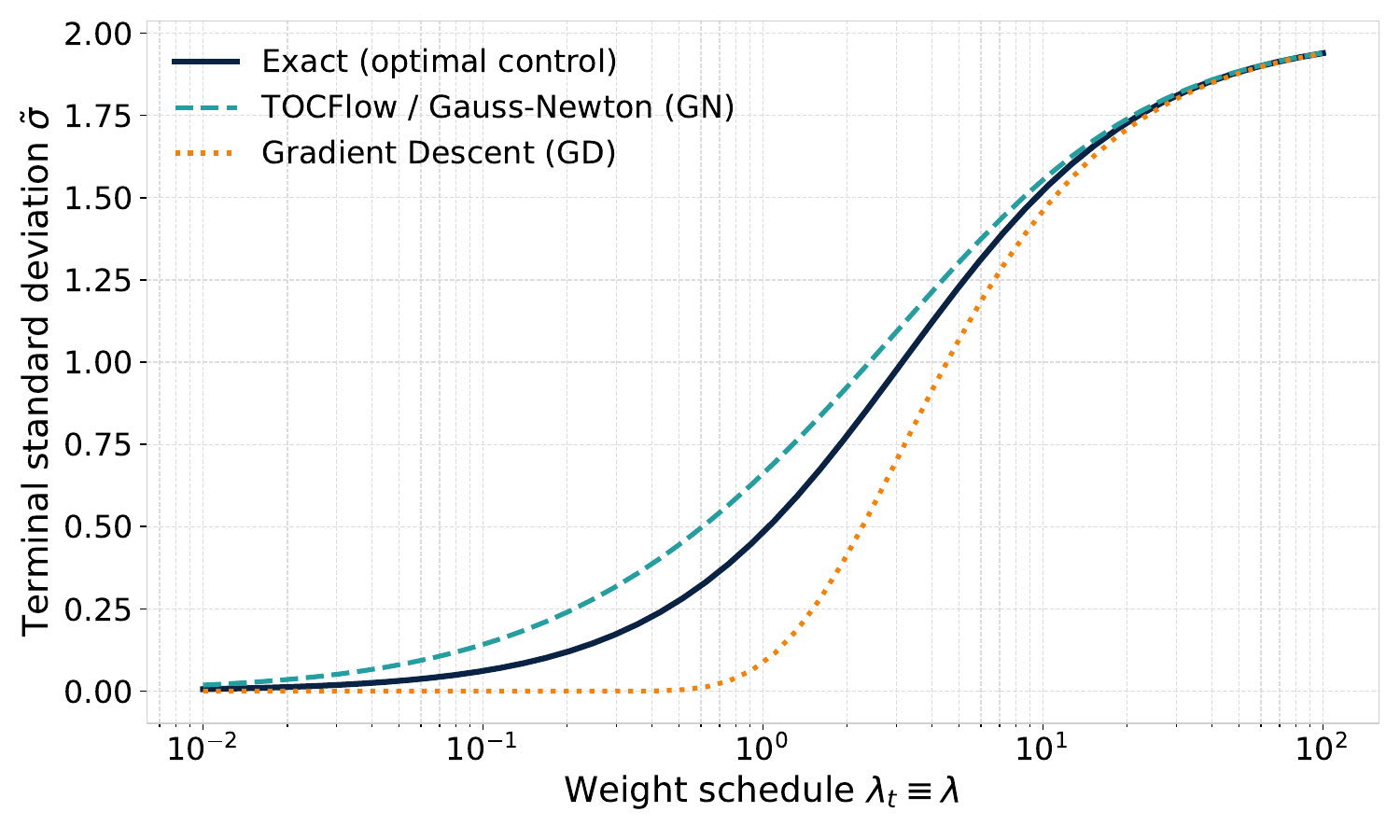}
\caption{\textbf{Terminal standard deviation \(\tilde{\sigma}\) against the weight schedule \(\lambda_t\) under different guidance schemes.} Near \(\lambda=0\), the GD scheme decays too fast compared with the optimal-control solution, whereas GN and TOCFlow exhibit a comparatively more accurate contraction.}
\label{fig:gaussian_comparison}
\end{figure}
\section{Representative applications}
\label{sec:representative_applications}

To demonstrate the versatility and effectiveness of our framework, we evaluate TOCFlow on three distinct high-dimensional generation tasks, each presenting a unique class of constraints and computational challenges. As summarized in~\Cref{tab:experiments_overview}, these applications cover a diverse range of scientific scenarios. First, we address Darcy flow generation (\Cref{subsec:high_dimensional_pde_constraints_darcy_flow}), which imposes high-dimensional equality constraints coupled through an elliptic differential operator, testing the method's ability to handle stiff, implicit manifolds. Second, we consider trajectory planning (\Cref{subsec:geometric_inequality_constraints_trajectory_planning}), where the goal is to satisfy non-convex geometric inequalities (obstacle avoidance), challenging the method to navigate discontinuous energy landscapes. Finally, we tackle turbulence snapshot generation (\Cref{subsec:global_spectral_constraints_turbulence_snapshot}), which enforces a global spectral law (Kolmogorov scaling), requiring the control to act coherently across all spatial scales.

\begin{table}[!ht]
\centering
\caption{Overview of the experimental settings for the three representative applications. The tasks are chosen to cover a diverse range of constraint types (equality vs.~inequality, local vs.~global) and dimensionalities.}
\label{tab:experiments_overview}
\footnotesize
\setlength{\tabcolsep}{4pt}
\renewcommand{\arraystretch}{1.25}
\begin{tabular}{@{}p{27mm}p{50mm}p{44mm}p{39mm}@{}}
\toprule
\footnotesize
& \textbf{Darcy flow} & \textbf{Trajectory planning} & \textbf{Turbulence snapshot} \\
\midrule
\textbf{Data dim \(d\)} &
\(2\times 64\times 64\) &
\(512\) &
\(256\times 256\) \\
\midrule
\textbf{Constraint} &
Enforce the discretized Darcy PDE residual at every grid location &
Obstacle avoiding inequality along the whole path in state space &
Match the Kolmogorov's power-law scaling of the generated field \\
\midrule
\textbf{Constraint dim \(r\)} &
\(64\times 64\) (High-dim) &
\(4\times 41\) (Medium-dim) &
\(1\) (Scalar) \\
\midrule
\textbf{Challenge} &
Constraint dimension comparable to data dimension and tightly coupled through an elliptic operator &
Inequality constraints that define a non convex feasible region for the trajectory &
Global constraint that acts in Fourier space and couples all spatial locations \\
\midrule
\textbf{Generative modeling approaches} & Constrained generation using gradient guidance~\citep{huang2024diffusionpde} or manifold projection~\citep{jacobsen2025cocogen,utkarsh2025physics} & Unconstrained trajectory generation~\citep{kerrigan2023functional} & Unconstrained turbulence snapshot generation~\citep{whittaker2024turbulence}\\
\midrule
\textbf{Architecture} &
Diffusion Transformer~\citep{peebles2023scalable} &
Fourier Neural Operator~\citep{li2021fourier} &
U-Net~\citep{ronneberger2015u}\\
\bottomrule
\end{tabular}
\end{table}

\subsection{Numerical implementation and comparison baselines}
\label{subsec:numerical_implementation_and_baselines}

\paragraph{Discretization of controlled dynamics.}
For all tasks, the reference velocity field \(\bm{b}^\star\) is parameterized by a deep neural network trained via Flow Matching. The sampling process involves integrating the controlled dynamics
\begin{equation}
\dot{\bm{X}}_t = \bm{b}^\star(\bm{X}_t, t) + \bm{a}(\bm{X}_t, t), \quad t\in[0, 1].
\end{equation}
The choice of integrator depends on the computational demands of the domain. For the Darcy flow and trajectory planning experiments, we use Heun's method (explicit trapezoidal rule) with a fixed step size \(\Delta t = 0.005\) (\(N=200\) steps) to ensure high-order accuracy. For the turbulence snapshot experiment, we employ the Forward Euler method to maximize throughput given the high dimensionality of the state space.

\paragraph{Gradient computation via automatic differentiation.}
Both the proposed TOCFlow method and the Gradient Descent (GD) method require the gradient of the terminal cost composed with the flow map, i.e., \(\nabla_{\bm x}(H \circ \Phi_{t \to 1}^{\bm{b}^\star})\). We estimate this spatial gradient using a lookahead estimator. At state \(\bm{X}_t\), we simulate the reference ODE to terminal time \(t=1\) using \(k\) Forward Euler steps (typically \(k=4\), with ablation studies of this parameter provided). We then compute the gradient of the terminal penalty with respect to \(\bm{X}_t\) by differentiating through this unrolled solver sequence. In our implementation, this is performed using PyTorch's standard automatic differentiation engine (\texttt{torch.autograd})~\citep{paszke2019pytorch}, which executes reverse-mode differentiation (the discrete adjoint method) to obtain exact gradients up to machine precision.

\paragraph{Implementation details of TOCFlow.}
Our method computes the control using the closed-form damping factor~\eqref{eq:dampling_factor} derived in~\Cref{subsec:computational_complexity_and_the_tocflow_scheme}:
\begin{equation}
\bm{a}(\bm{x}, t) = -\frac{\lambda_t^{-1}}{1 + s(t)\rho} \nabla_{\bm x} (H \circ \Phi_{t \to 1}^{\bm{b}^\star})(\bm{x}),
\end{equation}
where
\begin{equation}
s(t) = \int_t^1\dfrac{1}{\lambda_u}\mathrm du, \quad \rho = \dfrac{\|\nabla_{\bm x}(H\circ\Phi_{t\to 1}^{\bm b^\star})(\bm x)\|_2^2}{\|\bm h(\Phi_{t\to1}^{\bm b^\star}(\bm x))\|_2^2}. 
\end{equation}
In our implementation, we use the \(k\)-step Forward Euler lookahead described above to estimate the gradient \(\nabla_{\bm x} (H \circ \Phi_{t \to 1}^{\bm{b}^\star})(\bm{x})\). The weight schedule follows a power-law schedule \(\lambda_t = \lambda_0 (1-t)^\gamma\), where the scale \(\lambda_0\) and decay \(\gamma\) are hyperparameters tuned via grid search over powers of \(10\) (i.e., \(\lambda_0, \gamma\in\{10^i\}_{i\in\mathbb{Z}}\)). We will report these hyperparameters in each subsection below. We note a specific adaptation for the turbulence snapshot experiment. Due to the global nature of the spectral constraint (i.e., the constraint dimension \(r=1\)), we perform the exact Gauss--Newton update (solving the linear system without the diagonal approximation) to maximize accuracy. For consistency in the discussion of results, we refer to this method as TOCFlow throughout the text.

\paragraph{Comparison baselines.}
We compare TOCFlow against four widely-used baselines:

\begin{enumerate}
\item \textbf{Vanilla:} Samples are generated by integrating the reference field \(\bm{b}^\star(\bm x, t)\) without control (\(\bm{a}(\bm x, t) \equiv \bm{0}\)). This characterizes the unconstrained data distribution \(\bar{\rho}_1\).

\item \textbf{Gradient Descent (GD):} As derived in~\Cref{subsubsec:gradient_descent_derivation}, the GD method approximates the optimal control by taking a steepest descent step on the composed terminal cost, defined by the control law \(\bm{a}(\bm{x}, t) = -\eta_t \nabla_{\bm x} (H \circ \Phi_{t \to 1}^{\bm{b}^\star})(\bm{x})\). In our implementation, we estimate the gradient using the \(k\)-step Forward Euler lookahead described above. We treat the step size \(\eta_t \equiv \eta\) as a constant hyperparameter, selected via a grid search over powers of \(10\) (i.e., \(\eta \in \{10^i\}_{i \in \mathbb{Z}}\)) to minimize the terminal constraint violation.

\item \textbf{Approximated Gauss--Newton (GN):} As derived in~\Cref{subsubsec:gauss_newton_derivation}, the GN method approximates the optimal control by
\begin{equation}
\bm{a}(\bm{x}, t) = -\lambda_t^{-1} D\Phi_{t \to 1}^\top \bm{J}_{\bm h}^\top \bigl(\bm{I} + s \bm{J}_{\bm h} \bm{G} \bm{J}_{\bm h}^\top \bigr)^{-1} \bm{h}(\Phi_{t \to 1}(\bm{x})).
\end{equation}
As discussed in~\Cref{subsec:computational_complexity_and_the_tocflow_scheme}, evaluating this expression is computationally prohibitive because the linear operator involves the flow Jacobian \(D\Phi_{t\to 1}^{\bm b^\star}\) and inverting it requires differentiating through the ODE solver at every iteration of the Conjugate Gradient loop. To circumvent this, we implement a heuristic based on the linearity of the optimal transport map, similar to strategies employed in ECI~\citep{cheng2025gradient} and PCFM~\citep{utkarsh2025physics}. At each step \(t\), we first estimate the terminal state \(\hat{\bm{X}}_1 \approx \Phi_{t \to 1}^{\bm{b}^\star}(\bm{X}_t)\) using the \(k\)-step Forward Euler lookahead. We then project \(\hat{\bm{X}}_1\) onto the constraint manifold using a standard Gauss--Newton solver to obtain a valid target \(\bm X_1\) (this projection is cheap as it occurs at \(t=1\) and does not involve \(D\Phi_{t\to1}^{\bm b^\star}\)). Finally, we construct the state at the next timestep \(t+\Delta t\) by interpolating between the initial noise \(\bm{X}_0\) and the projected target \(\bm X_1\), effectively pulling the correction back along the optimal transport geodesic:
\begin{equation}
\bm{X}_{t+\Delta t} \gets (1-(t+\Delta t))\bm{X}_0 + (t+\Delta t)\bm{X}_1.
\end{equation}
The above process is repeated for each \(t\in[0,1)\) until a final generated sample \(\bm X_1\) is obtained.

\item \textbf{Terminal projection:} This method generates unconstrained samples \(\bm{X}_1 \sim \bar{\rho}_1\) using the vanilla method and subsequently projects them onto the constraint manifold \(\mathcal{M} = \{\bm{x} \colon \bm{h}(\bm{x}) = \bm{0}\}\). The projection is obtained by solving the constrained minimization problem \(\min_{\bm{z}} \|\bm{z} - \bm{X}_1\|_2^2\) subject to \(\bm{h}(\bm{z})=\bm{0}\). We employ the same Gauss--Newton solver used in the GN baseline. Since the solver performs only an approximate projection, we use a fixed budget of \(1000\) iterations to match the maximum iteration budget in comparable work such as CoCoGen~\citep{jacobsen2025cocogen}. Empirical analysis confirms that this computational budget yields results similar as those obtained using adaptive stopping criteria based on the relative change of the residual norm.
\end{enumerate}

\subsection{High-dimensional PDE constraints: Darcy flow}
\label{subsec:high_dimensional_pde_constraints_darcy_flow}

We first apply our framework to Darcy flow, a fundamental system in fluid dynamics describing flow through porous media. In this setting, the state space consists of high-dimensional grid-based fields representing physical coefficients and solutions, and the objective is to generate consistent \emph{pairs} that satisfy the PDE. We define the constraints by enforcing the PDE residual to be zero at every discretization point. Mathematically, this requires enforcing \emph{high-dimensional equality constraints} coupled through differential operators, making it challenging.

\paragraph{Dataset.}
We use the standard Darcy flow setup~\citep{baldan2025flow,bastek2025physics}. The physical system is governed by the steady-state diffusion equation with mass conservation:
\begin{equation}
\begin{aligned}
\bm{u}(\bm{x}) &= -\bm K(\bm{x})\nabla \bm p(\bm{x}), \\
\nabla \cdot \bm{u}(\bm{x}) &= f(\bm{x}), \quad \bm{x} \in \Omega = [0, 1]^2,
\end{aligned}
\end{equation}
subject to zero-flux Neumann boundary conditions \(\bm{u} \cdot \bm{n} = 0\) on \(\partial \Omega\) and a global zero-mean pressure constraint \(\int_\Omega p(\bm{x}) \mathrm d\bm{x} = 0\). The source term \(f(\bm{x})\) is a fixed, piecewise constant function defined as
\begin{equation}
f(\bm{x}) = \begin{cases} 
r, & \text{if } |x_i - w/2| \le w/2, \quad \text{for } i=1,2, \\ 
-r, & \text{if } |x_i - (1 - w/2)| \le w/2, \quad \text{for } i=1,2, \\ 
0, & \text{otherwise}
\end{cases}
\end{equation}
with magnitude \(r=10\) and width \(w=0.125\). The permeability field \(\bm K(\bm{x})\) is sampled from a log-normal Gaussian Random Field (GRF), defined as \(\bm K(\bm{x}) = \exp(\bm G(\bm{x}))\). The underlying GRF \(\bm G(\bm{x})\) is characterized by a zero mean and an exponential covariance kernel
\begin{equation}
k(\bm{x}, \bm{x}') = \exp(-\|\bm{x} - \bm{x}'\|_2 / \ell)
\end{equation}
with \(\ell=0.1\). To ensure computational tractability, the field is constructed using a truncated Karhunen--Lo\`{e}ve expansion~\citep{wang2008karhunen}
\begin{equation}
\bm G(\bm{x}) = \sum_{i=1}^{s} \sqrt{\lambda_i} z_i \bm {\phi}_i(\bm{x}),
\end{equation}
where \(s=64\) is the number of retained modes, \((\lambda_i, \bm{\phi}_i)\) are the eigenvalues and eigenfunctions of the kernel, and \(z_i \sim \mathcal{N}(0, 1)\) are standard normal coefficients. The corresponding pressure fields \(\bm p(\bm{x})\) are obtained by solving the discretized system using a least-squares solver to handle the over-determined constraints.

We use the dataset publicly available at~\href{https://doi.org/10.3929/ethz-b-000674074}{Darcy Flow dataset}~\citep{bastek2025physics}, which consists of \(10{,}000\) pairs of permeability fields \(\bm K(\bm x)\) and pressure solutions \(\bm p(\bm x)\) defined on the unit square domain \(\bm x \in [0, 1]^2\). The data is discretized on a spatial grid of resolution \(64 \times 64\), resulting in a state dimension of \(d = 2 \times 64^2 = 8{,}192\). Before training, both the permeability and pressure channels are standardized channel-wise using the global mean and standard deviation computed over the training set.

\paragraph{Pre-trained model architecture.}
The time-dependent velocity field \(\bm{b}^\star(\bm x, t)\) is parameterized using a Diffusion Transformer (DiT)~\citep{peebles2023scalable}. The model processes the discretized PDE fields as images, partitioning the \(64 \times 64\) input into non-overlapping patches (size \(4 \times 4\)) which are linearly projected into a latent embedding space. To preserve spatial structure, fixed \(2\)-dimensional sinusoidal (sin-cos) positional embeddings are added to these patch tokens. The diffusion timestep \(t\) is encoded via a sinusoidal frequency embedding followed by a Multi-Layer Perceptron (MLP) to form a global conditioning vector \(\bm{c}\). The network consists of a sequence of DiT blocks that employ an adaptive layer normalization (adaLN-Zero) mechanism. Within each block, the conditioning vector \(\bm{c}\) is projected to regress scale and shift parameters that modulate the normalized features immediately prior to the self-attention and pointwise feedforward layers. Additionally, \(\bm{c}\) determines dimension-wise gating scalars for the residual connections, which are initialized to zero to approximate an identity mapping at the start of training.

\paragraph{Training details.}
The model is trained using the Flow Matching objective with an Optimal Transport (OT) conditional vector field~\citep{lipman2023flow,liu2023flow}. We optimize the network parameters using the AdamW~\citep{loshchilov2019decoupled} optimizer with a constant learning rate of \(3 \times 10^{-5}\) and a batch size of \(128\). and training proceeds for a total of \(8{,}000\) epochs. To stabilize the training dynamics and improve sample quality, we maintain an Exponential Moving Average (EMA)~\citep{morales2024exponential} of the model parameters with a decay rate of \(0.999\).

\paragraph{Constraint formulation.}
We interpret the generative task as a physics-constrained sampling problem where the generated state \((\bm K, \bm p)\) must satisfy the steady-state Darcy flow PDE. To handle the global zero-mean pressure constraint, we explicitly project the generated pressure field onto the valid subspace by subtracting its spatial mean, \(\bm p(\bm{x}) \gets\bm p(\bm{x}) - \int_\Omega \bm p(\bm{x}) \mathrm{d}\bm{x}\), before evaluating the residuals. We then define the constraint function \(\bm{h}(\bm{x})\) by concatenating the discretized interior PDE residual and the boundary flux violations. Using a second-order central finite difference scheme \(\mathcal{D}\), the components at grid index \((i, j)\) are given by
\begin{equation}
h_{\text{int}}(\bm{x})_{i,j} = \bigl( -\nabla \cdot (\bm K \nabla \bm p) - f \bigr)_{i,j} \approx - K_{i,j} \mathcal{D}^2 p_{i,j} - (\mathcal{D} K_{i,j}) \cdot (\mathcal{D} p_{i,j}) - f_{i,j},
\end{equation}
for interior points, and the Neumann boundary residual is
\begin{equation}
h_{\text{bc}}(\bm{x})_{i,j} = (\nabla \bm p \cdot \bm{n})_{i,j} \approx \mathcal{D}_{\bm{n}} p_{i,j},
\end{equation}
for points on the boundary \(\partial \Omega\). The total constraint vector \(\bm{h}(\bm{x})\) collects these values across the \(64 \times 64\) grid. The terminal cost is defined as the squared Euclidean norm of this residual, \(H(\bm{x}) = \frac{1}{2} \|\bm{h}(\bm{x})\|_2^2\). This formulation is particularly challenging because the constraint Jacobian approximates an elliptic operator, which is sparse but ill-conditioned, and the constraint dimension \(r\) scales linearly with the data dimension \(d\) (specifically \(r=d/2\)), severely restricting the manifold of valid solutions.

\paragraph{Implementation details.}
We detail here the specific hyperparameters for the Darcy flow task, complementing the general settings described in~\Cref{subsec:numerical_implementation_and_baselines}. For GD and TOCFlow, the spatial gradient \(\nabla_{\bm{x}} (H \circ \Phi_{t\to 1}^{\bm{b}^\star})\) is estimated using a \(k=4\) step forward Euler lookahead. For GD, we employ a constant step size of \(\eta = 0.01\). For TOCFlow, we set the weight schedule parameters to \(\lambda_0 = 100\) and \(\gamma = 0\) (constant weight schedule). For the terminal projection method, we project samples onto the constraint manifold using a damped Gauss--Newton solver with a maximum budget of \(1{,}000\) iterations. The linearized subproblem at each step is solved inexactly via Conjugate Gradient with \(20\) inner iterations. All methods are implemented on top of the public~\href{https://github.com/tum-pbs/PBFM}{Physics Based Flow Matching} codebase~\citep{baldan2025flow}, which we extend with our GD, GN, and TOCFlow modules.

\paragraph{Results.}
We compare the performance of TOCFlow against the baselines in terms of constraint satisfaction and sample quality. We note that the approximated Gauss--Newton baseline is omitted from this analysis as it exhibited numerical divergence. \Cref{fig:darcy_visualization} visualizes the generated permeability \(\bm{K}\) and pressure \(\bm{p}\) fields alongside their associated PDE residual absolute error \(|\bm h(\bm{K}, \bm{p})|\). Interestingly, the results reveal two distinct performance clusters. The vanilla method (top left) and the terminal projection method (bottom right) yield nearly identical results, exhibiting high residual errors (bright yellow regions). This suggests that the projection solver fails to converge, likely because the unguided samples lie outside the basin of attraction of the constraint manifold, where the ill-conditioning of the discrete elliptic operator makes the local Gauss--Newton projectin ineffective. In contrast, both GD method (top right) and TOCFlow (bottom left) demonstrate significant and similar improvements, effectively reducing the residual magnitude by an order of magnitude (dominated by blue regions).

\begin{figure}[!ht]
\includegraphics[width=0.495\linewidth]{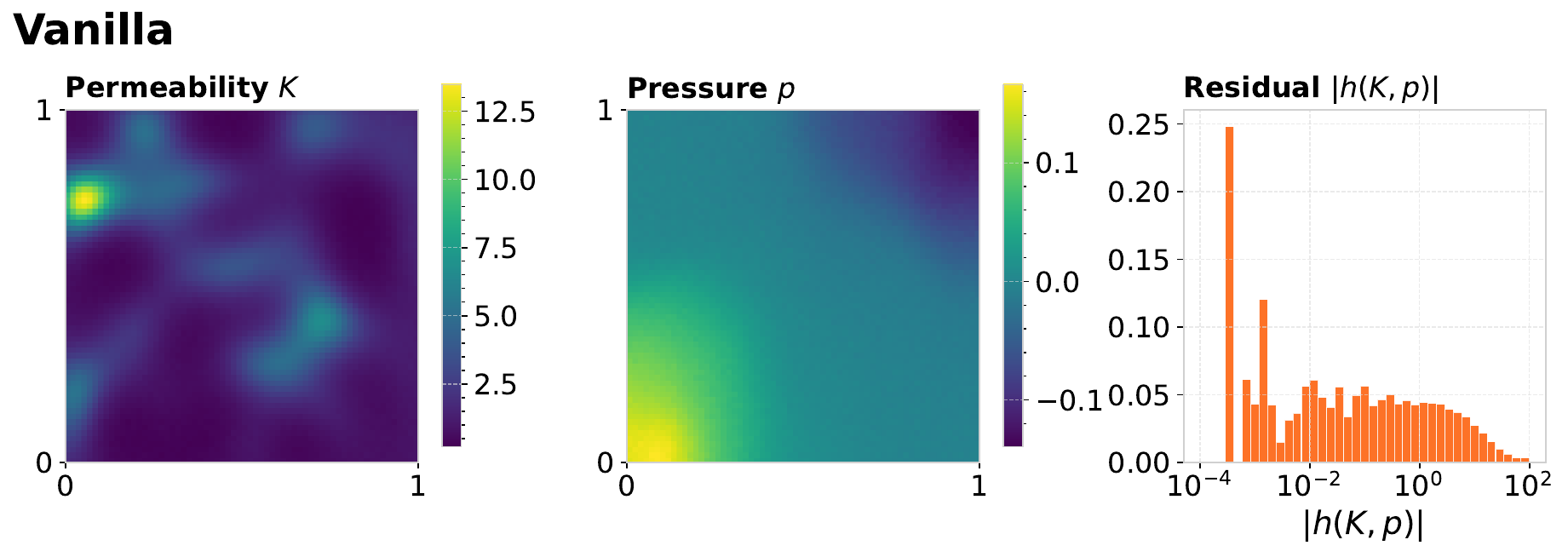}
\includegraphics[width=0.495\linewidth]{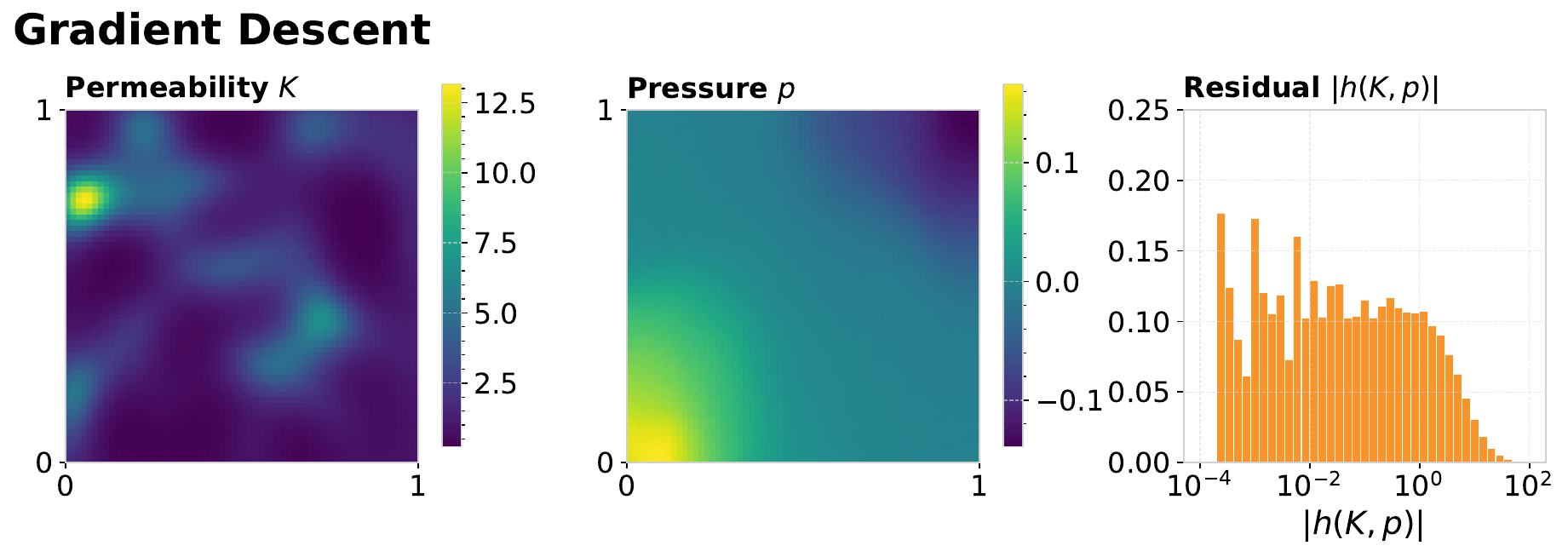}\par
\includegraphics[width=0.495\linewidth]{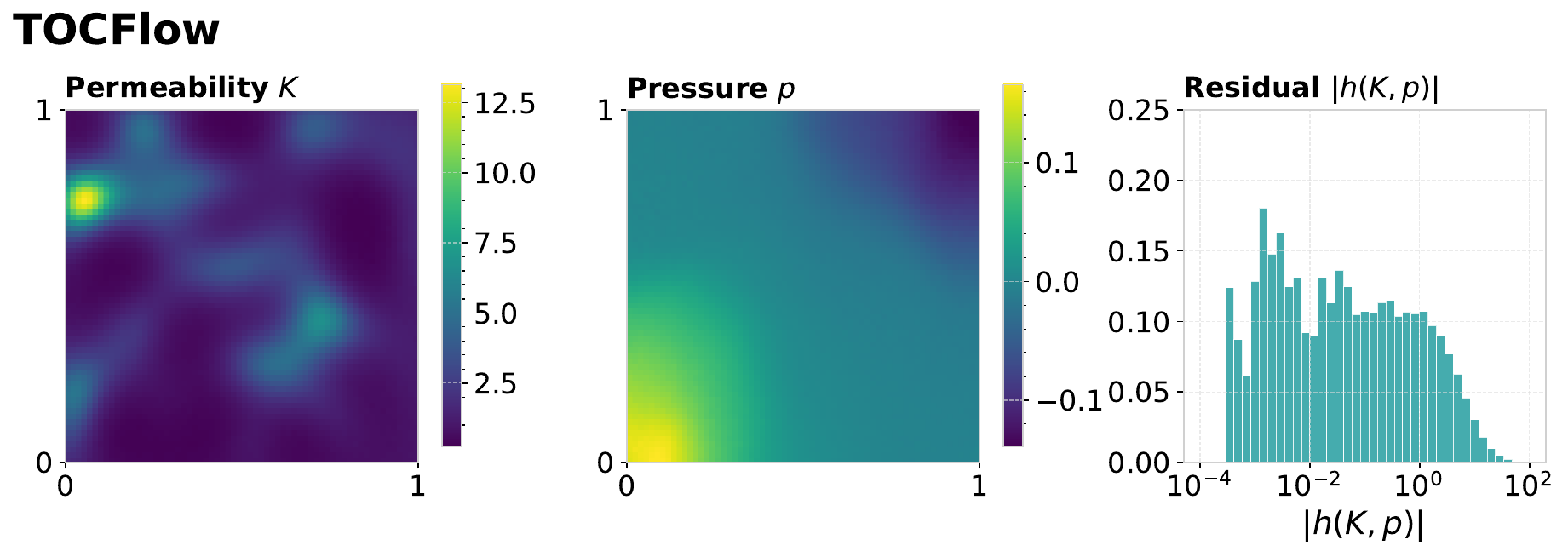}
\includegraphics[width=0.495\linewidth]{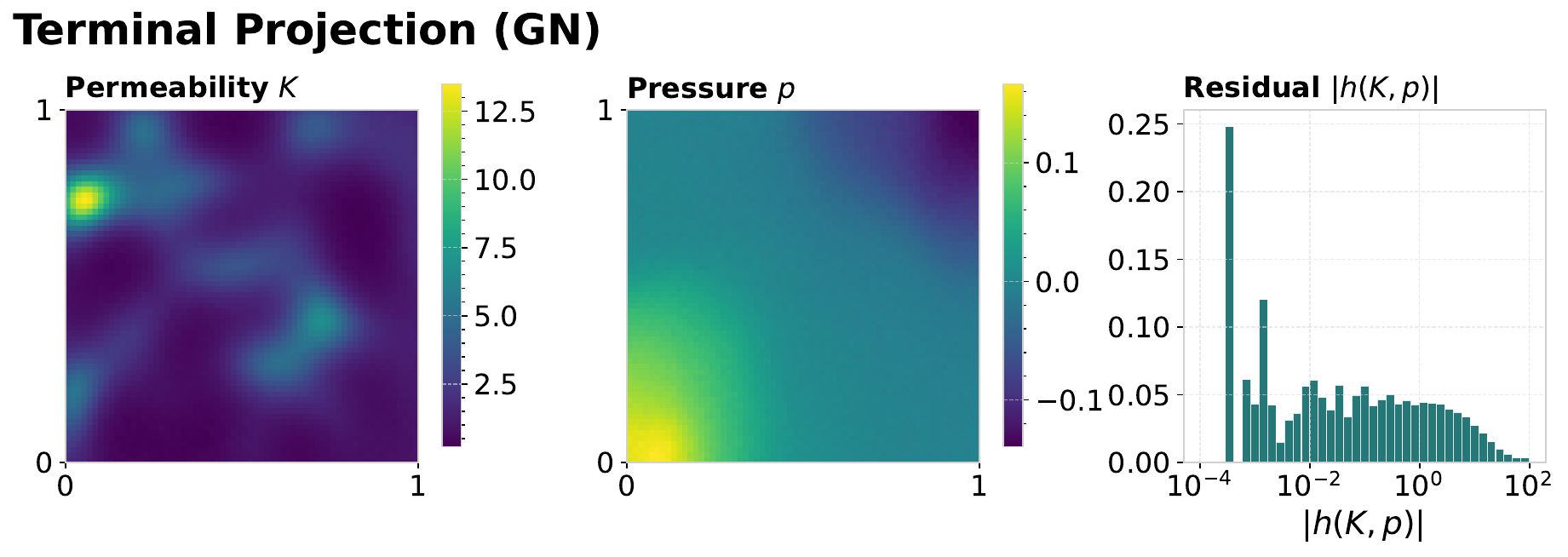}
\caption{\textbf{Qualitative comparison of the generated samples using different methods.}  The panels display the permeability field \(\bm K\) (left) and the pressure field \(\bm p\) (middle) for a representative sample. The rightmost plot in each block shows the histogram of PDE residuals aggregated over \(512\) generated samples, with the \(y\)-axis representing density. \textbf{Top row:} Vanilla (left) and GD (right). \textbf{Bottom row:} TOCFlow (left) and terminal projection (right). Both GD and TOCFlow shift the residual distribution toward lower values, indicating better constraint satisfaction.}
\label{fig:darcy_visualization}
\end{figure}

Quantitative results are summarized in~\cref{fig:darcy_ablation}. The violin plots of the log-terminal cost \(\log_{10} H(\bm{f})\) (top) confirm this grouping: vanilla and terminal projection methods cluster around a high cost of \(10^{2.6}\), while GD and TOCFlow achieve significantly lower costs around \(10^{1.8}\). We further investigate the impact of hyperparameters in the bottom rows of~\cref{fig:darcy_ablation}. Notably, GD and TOCFlow exhibit remarkably similar behavior. Increasing the lookahead steps \(k\) yields a slight increase in the terminal cost for both methods. For the terminal projection baseline (bottom right), increasing the number of iteration steps from \(1\) to \(1{,}000\) has almost no effect on the terminal cost, confirming that the solver remains stuck in a local minimum far from the valid manifold.

\begin{figure}[!ht]
\hfil
\includegraphics[width=0.45\linewidth]{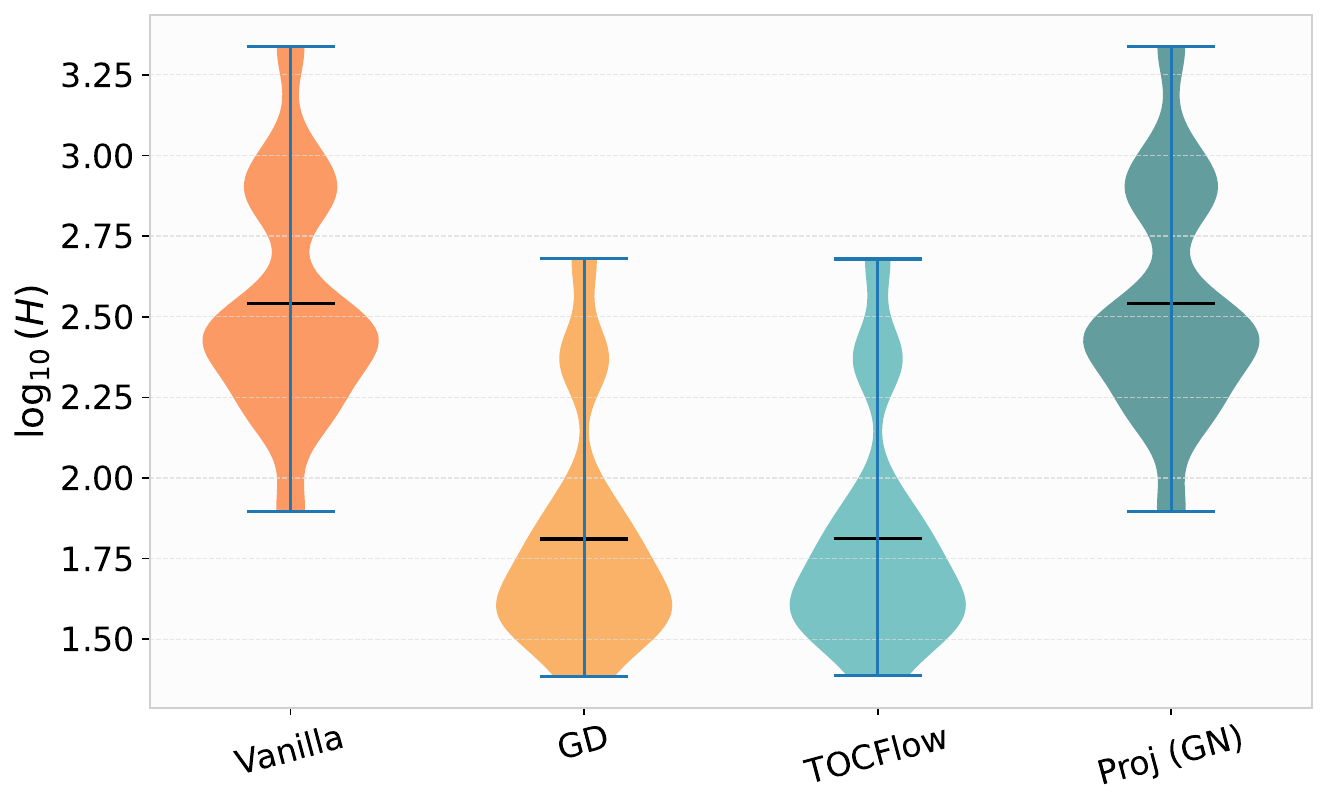}\par
\includegraphics[width=0.337\linewidth]{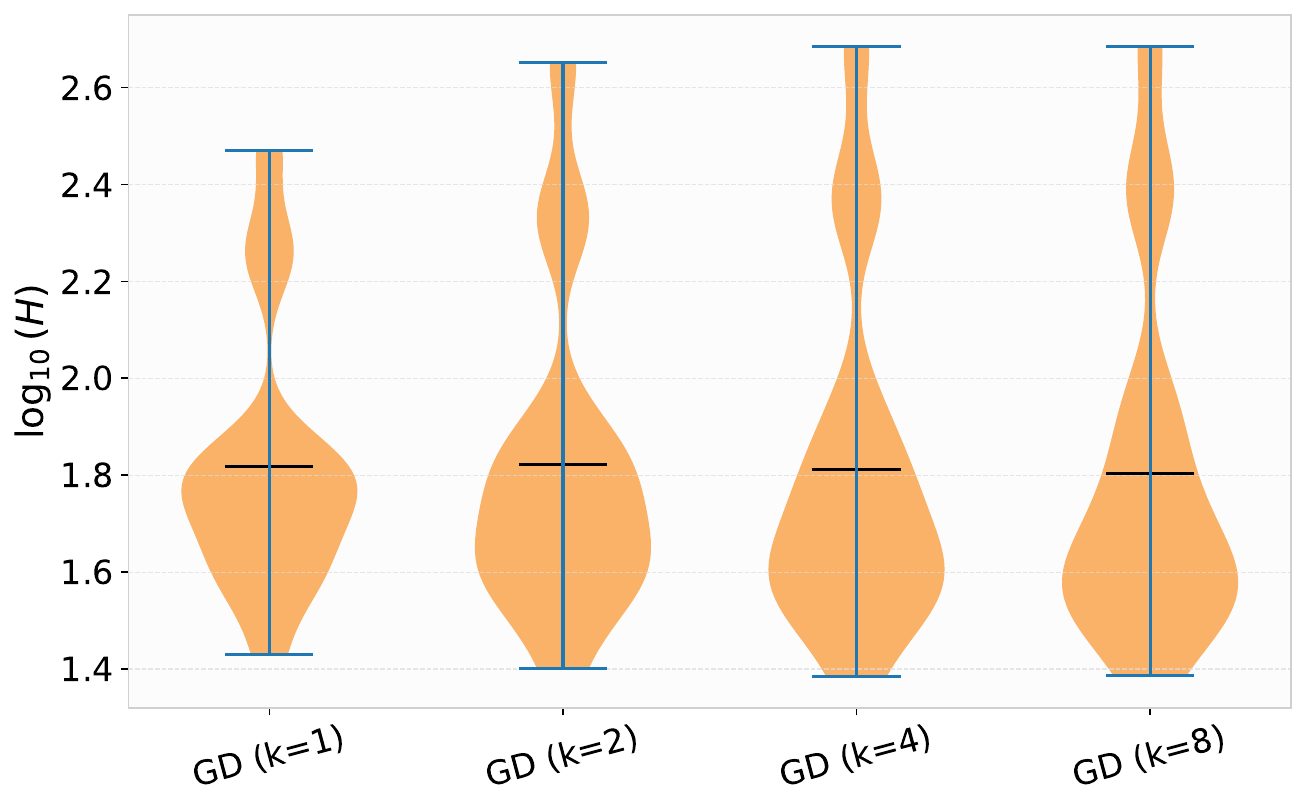}
\includegraphics[width=0.328\linewidth]{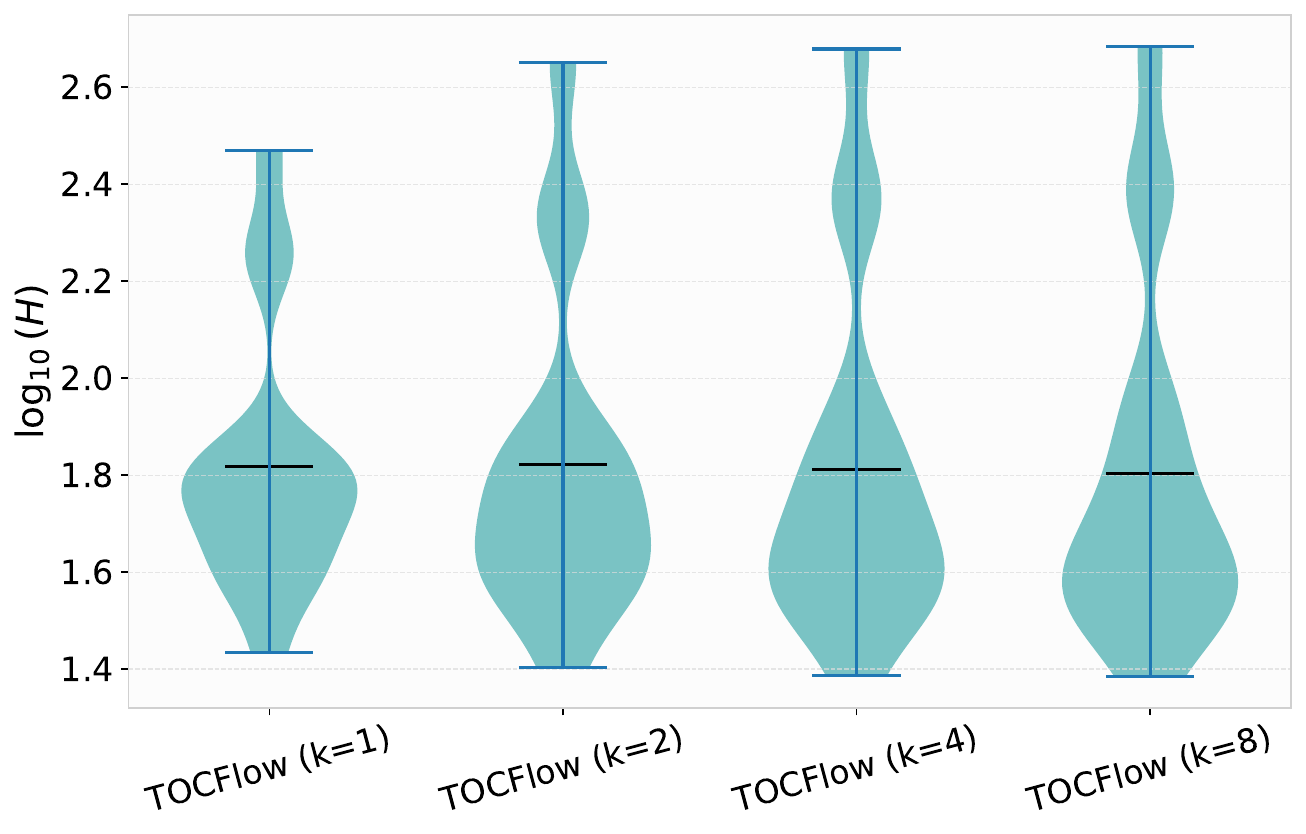}
\includegraphics[width=0.325\linewidth]{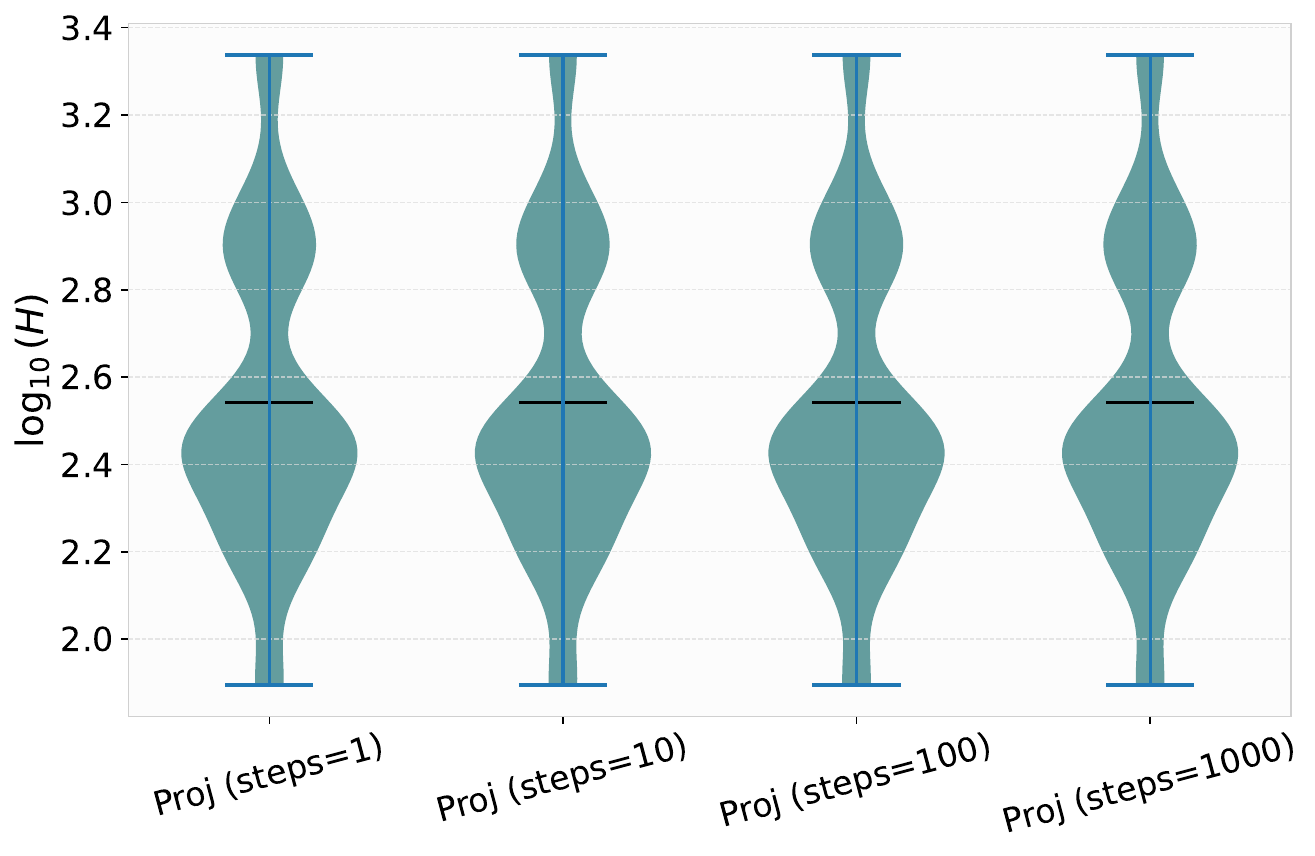}
\caption{\textbf{Quantitative evaluation of constraint violation and hyperparameter sensitivity.} \textbf{Top:} Violin plots of the terminal cost \(H(\bm{x})\) (log scale) for \(512\) generated samples across all methods. The results show two performance tiers: vanilla/terminal projection (high error) vs. GD/TOCFlow (low error). \textbf{Bottom left/center:} Ablation of the number of lookahead explicit Euler steps \(k\) for GD and TOCFlow. Both methods show similar sensitivity, with increased \(k\) slightly increasing the cost. \textbf{Bottom right:} Ablation of the number of iterations for terminal projection. The error remains high and constant regardless of the iteration count, indicating that the solver cannot bridge the gap from the unguided distribution to the constraint manifold.}
\label{fig:darcy_ablation}
\end{figure}

\subsection{Geometric inequality constraints: Trajectory planning}
\label{subsec:geometric_inequality_constraints_trajectory_planning}

We then apply our framework to a trajectory planning problem, a standard task in robotic motion planning and control, where collision avoidance is commonly enforced either via corridor-based constraints~\citep{gao2018online} or via trajectory optimization with sequential convexification~\citep{schulman2014motion}. In this setting, the state space consists of functions, and the objective is to generate smooth trajectories that navigate through a domain populated with obstacles. We define the constraints by confining the trajectories to ``safety corridors.'' Mathematically, this necessitates enforcing \emph{geometric inequality constraints}, distinguishing this application from the equality constraints addressed in the Darcy flow example.

\paragraph{Dataset.}
We construct a synthetic dataset of smooth \(1\)-dimensional trajectories \(\bm{f} = \{f(x_i)\}_{i=1}^{n_x}\) over the longitudinal domain \(x \in [0, 1]\). The unconstrained data distribution \(\bar{\rho}_1\) is defined as a mixture of two Gaussian Processes (GPs), i.e., \(0.5\mathcal{GP}_1 + 0.5\mathcal{GP}_2\). Both components share a Radial Basis Function (RBF) covariance kernel
\begin{equation}
k(x, x') = \sigma^2 \exp\biggl(-\frac{|x-x'|^2}{2\ell^2}\biggr),
\end{equation}
with signal variance \(\sigma^2 = 0.4\) and length scale \(\ell = 0.1\). The components are distinguished by their mean functions, which represent two crossing linear paths:
\begin{equation}
\mu_1(x) = 10x - 5 \quad \text{and} \quad \mu_2(x) = -10x + 5.
\end{equation}
To generate a discrete sample \(\bm{f} \in \mathbb{R}^{n_x}\), we first sample \(n_x=512\) spatial coordinates uniformly from \([0,1]\) and sort them. We then evaluate a function draw from one of the GPs at these coordinates. See~\cref{fig:trajectory_dataset} for representative dataset samples.

\begin{figure}[!ht]
\centering
\includegraphics[width=0.6\linewidth]{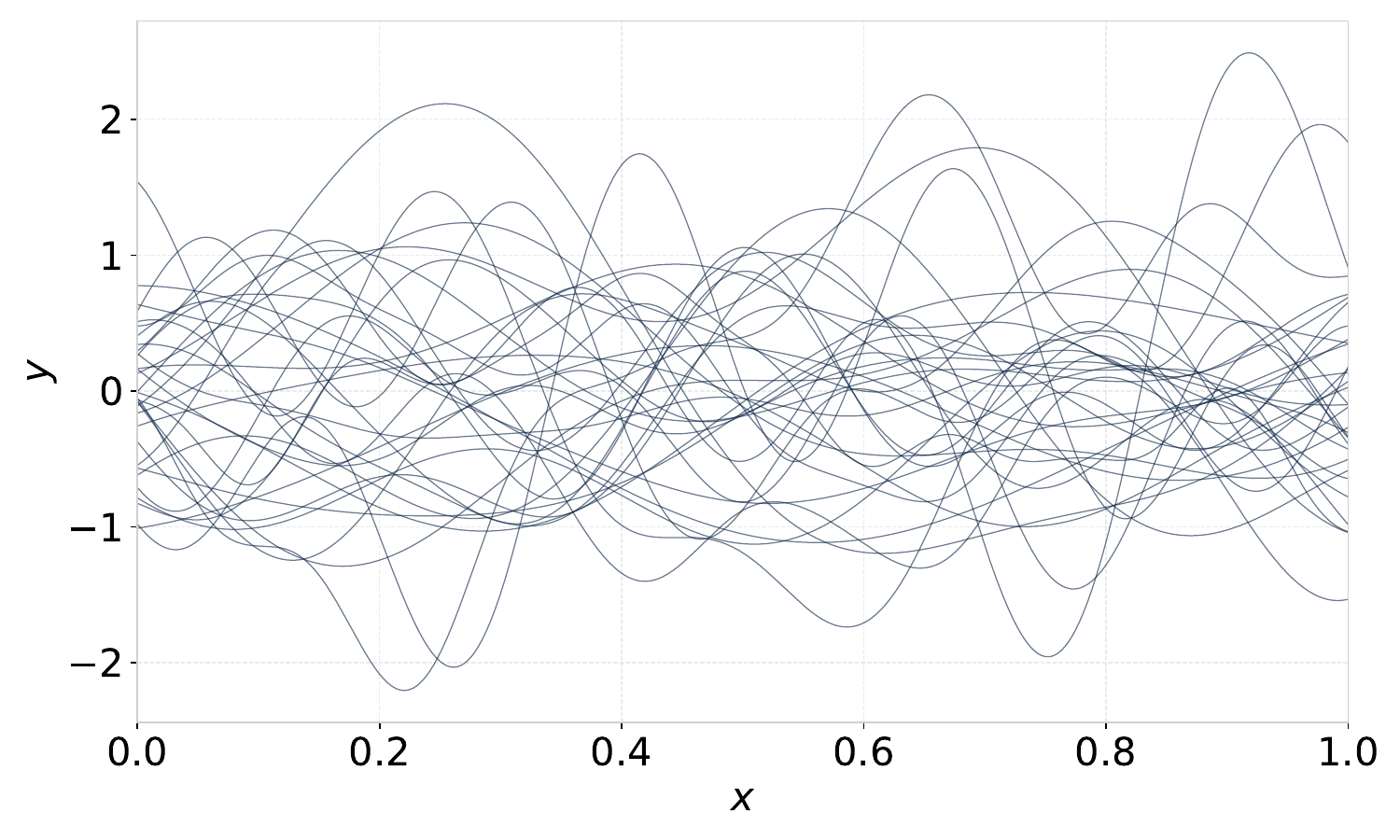}
\caption{Visualization of \(32\) independent realizations from the synthetic trajectory dataset in~\Cref{subsec:geometric_inequality_constraints_trajectory_planning}.}
\label{fig:trajectory_dataset}
\end{figure}

\paragraph{Pre-trained model architecture.}
The time-dependent velocity field \(\bm{b}^\star(\bm f, t)\) is parameterized using a Fourier Neural Operator (FNO)~\citep{li2021fourier}. The model inputs a concatenation of the trajectory state \(\bm{f}(x)\), the spatial coordinate \(x\), and the Flow Matching timestep \(t\) (broadcasted as a constant function over the domain). The input is lifted to a codomain of dimension \(256\) and processed through a sequence of spectral convolution layers. In each layer, the signal is transformed to the frequency domain via the Fast Fourier Transform (FFT). A learnable complex weight matrix then acts linearly on the lowest \(64\) frequency modes, filtering the signal before an inverse FFT recovers the spatial representation. Finally, a MLP with a hidden width of \(128\) projects the features back to the velocity space. By defining the non-linearity in the frequency domain, this architecture approximates a global integral operator, allowing the velocity field to capture long-range spatial correlations independent of the discretization grid.

\paragraph{Training details.}
The model is trained using the Functional Flow Matching (FFM)~\citep{kerrigan2023functional} objective with the Optimal Transport (OT) conditional vector field. To handle the functional nature of the data, the regression loss is defined using a reproducing kernel Hilbert space (RKHS) norm induced by a Gaussian RBF kernel \(k(x, y) = \sigma_k^2 \exp(-|x-y|^2 / 2\ell_k^2)\), with variance \(\sigma_k^2=0.1\) and length scale \(\ell_k=0.01\). The training is performed on a dataset of \(5{,}120\) synthetic curves discretized at \(n_x = 512\) grid points, which are scaled to the interval \([-3, 3]\) prior to processing. We optimize the network parameters using the Adam~\citep{kingma2014adam} optimizer with an initial learning rate of \(10^{-3}\) and a batch size of \(32\). The learning rate is decayed using a step scheduler every \(25\) epochs, and training proceeds for a total of \(1{,}000\) epochs.

\paragraph{Constraint formulation.}
We interpret the generative task as a trajectory planning problem for a robot crossing a domain from longitudinal position \(x=0\) to \(x=1\). In this setting, the grid index represents the longitudinal progress, and the function value \(\bm{f}\) represents the robot's lateral position. We evaluate the methods on a ``safety corridor'' constraint. Specifically, the path \(\bm{f}\) is constrained on a subset of longitudinal coordinates \(\mathcal{T} \subset \{0, \dots, n_x-1\}\) corresponding to \(k_{\text{spans}}=4\) distinct segments. In these segments, the lateral position \(f_k\) is restricted to a randomly generated interval \(\mathcal{I}_k = [\ell_k, u_k]\). The constraints are generated such that the width of the valid corridor varies between \(12\%\) and \(25\%\) of the total lateral range.

We model this inequality constraint via a non-smooth penalty function composed of Rectified Linear Units (ReLU) (i.e., \(x\mapsto\max(0, x)\)). The vector-valued constraint residual \(\bm{h}(\bm{f}) \in \mathbb{R}^{|\mathcal{T}|}\) is defined by the scaled pointwise violations:
\begin{equation}
h_k(\bm{f}) = \frac{1}{\sqrt{|\mathcal{T}|}}\cdot\bigl( \max(0, \ell_k - f_k) + \max(0, f_k - u_k) \bigr), \quad k \in \mathcal{T}.
\end{equation}
The terminal cost is then given by the squared Euclidean norm \(H(\bm{f}) = \frac{1}{2} \|\bm{h}(\bm{f})\|_2^2\). This formulation creates a flat potential with zero gradient strictly inside the safety corridor, and quadratic growth within the obstacles. This poses a challenge for gradient-based sampling, as the guidance signal vanishes entirely when the robot is momentarily in free space, potentially allowing the diffusion noise to push the state back into collision.

\paragraph{Implementation details.}
We detail here the specific hyperparameters for the trajectory planning task, complementing the general settings described in~\Cref{subsec:numerical_implementation_and_baselines}. For GD and TOCFlow, the spatial gradient \(\nabla_{\bm{x}} (H \circ \Phi_{t\to 1}^{\bm{b}^\star})\) is estimated using a \(k=4\) step forward Euler lookahead. For GD, we employ a constant step size of \(\eta = 0.1\). For TOCFlow, we set the weight schedule parameters to \(\lambda_0 = 0.1\) and \(\gamma = 0\) (constant weight schedule). For the terminal projection method, we project samples onto the constraint manifold using a damped Gauss--Newton algorithm. Since the constraint penalty is piecewise quadratic, the Gauss--Newton quadratic approximation is exact. Consequently, we employ only \(1\) outer iteration. The linearized subproblem is solved inexactly via Conjugate Gradient with \(20\) inner iterations. All methods are implemented on top of the public~\href{https://github.com/GavinKerrigan/functional_flow_matching}{Functional Flow Matching} codebase~\citep{kerrigan2023functional}, which we extend with our GD, GN, and TOCFlow modules.

\paragraph{Results.}
We compare the performance of TOCFlow against the baselines in terms of constraint satisfaction and sample quality. \Cref{fig:trajectory_visualization} visualizes the generated trajectories relative to the safety corridors (shaded grey regions). The vanilla method (top left), being unguided, frequently generates trajectories that collide with the obstacles, confirming that the reference distribution has support outside the safety corridors. Gradient Descent (GD) (top middle) improves compliance but exhibits residual collisions near the boundaries due to the heuristic nature of the step size schedule. In contrast, TOCFlow (top right) successfully steers the robot into the feasible region while preserving the smooth nature of the reference distribution.

\begin{figure}[!ht]
\hfil
\includegraphics[width=0.33\linewidth]{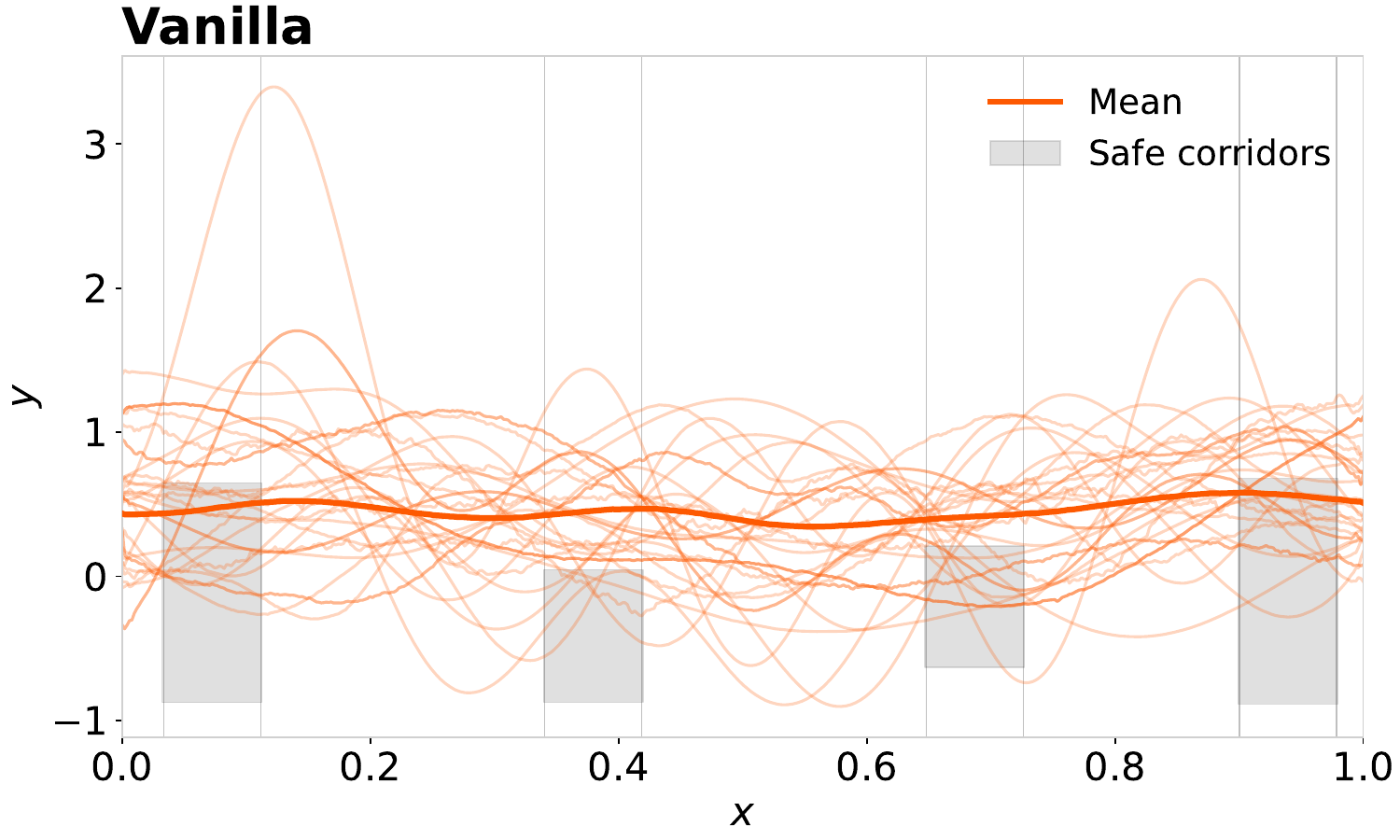}
\includegraphics[width=0.33\linewidth]{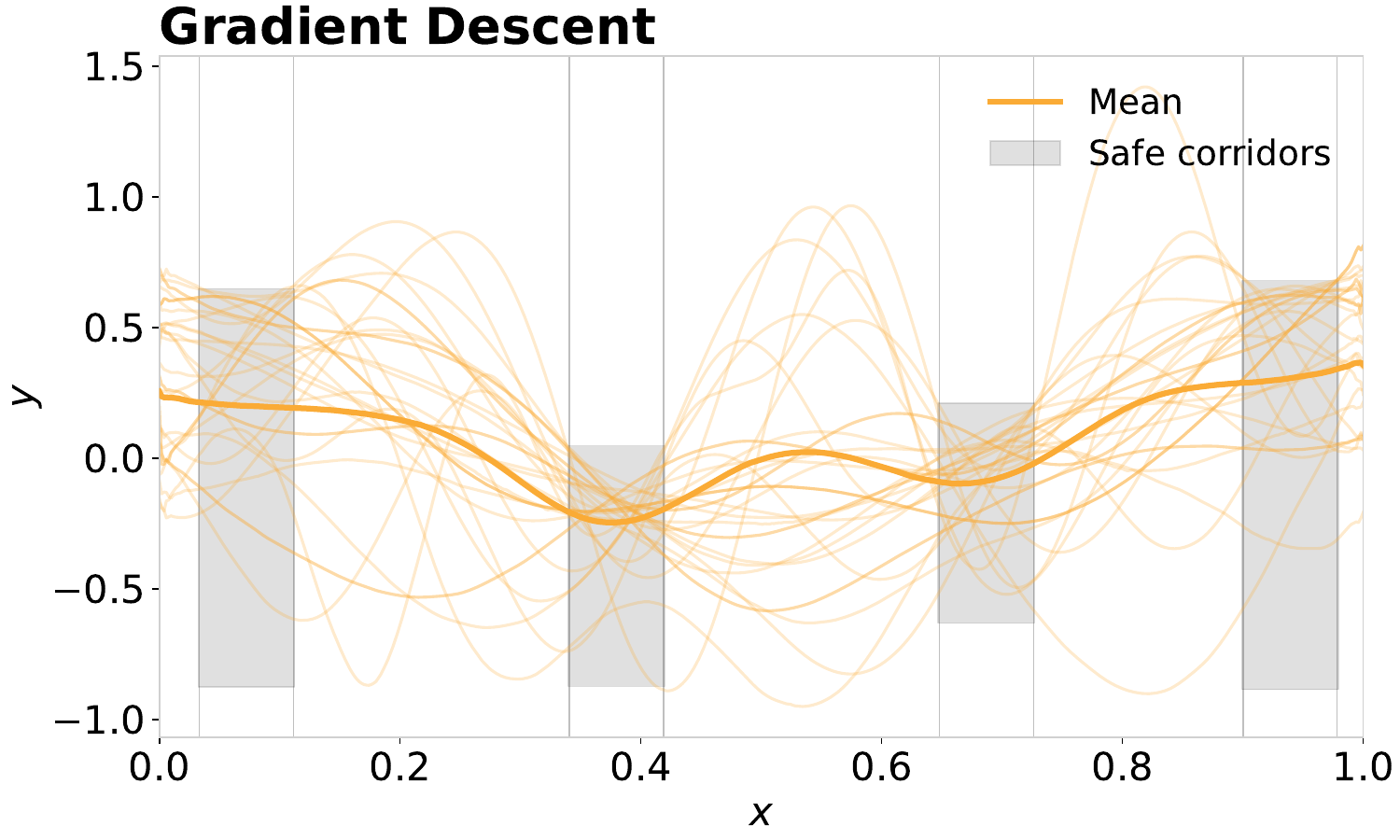}\par
\includegraphics[width=0.33\linewidth]{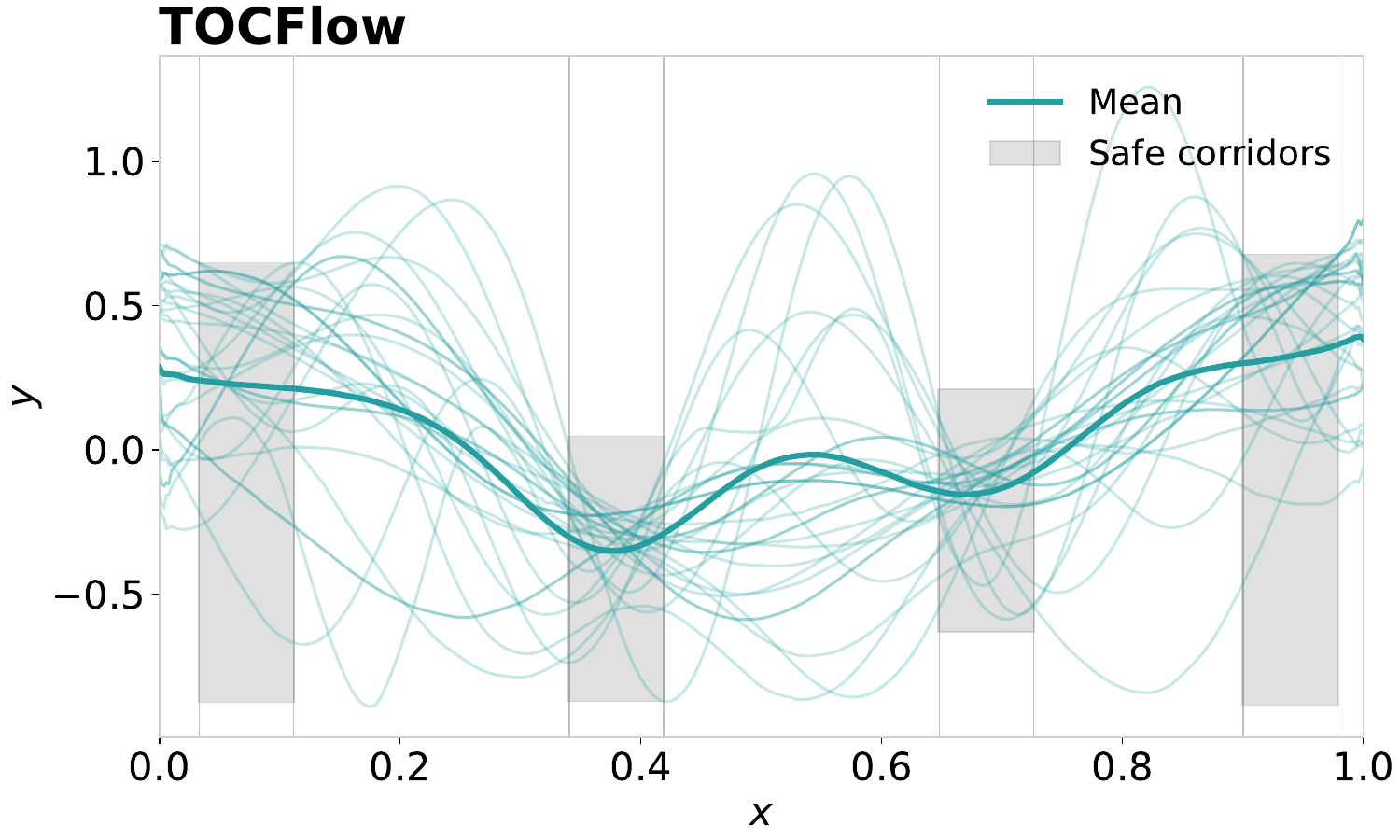}
\includegraphics[width=0.33\linewidth]{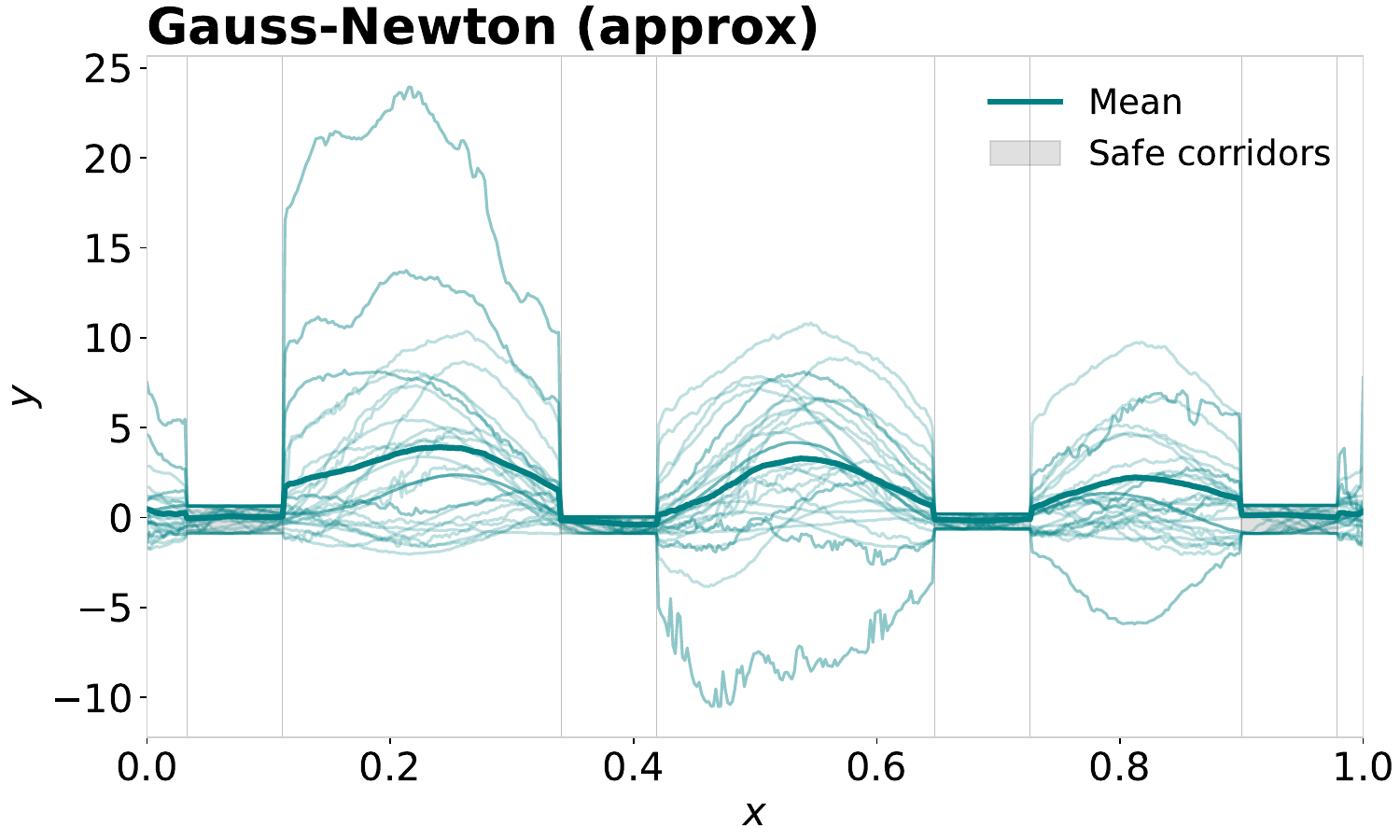}
\includegraphics[width=0.33\linewidth]{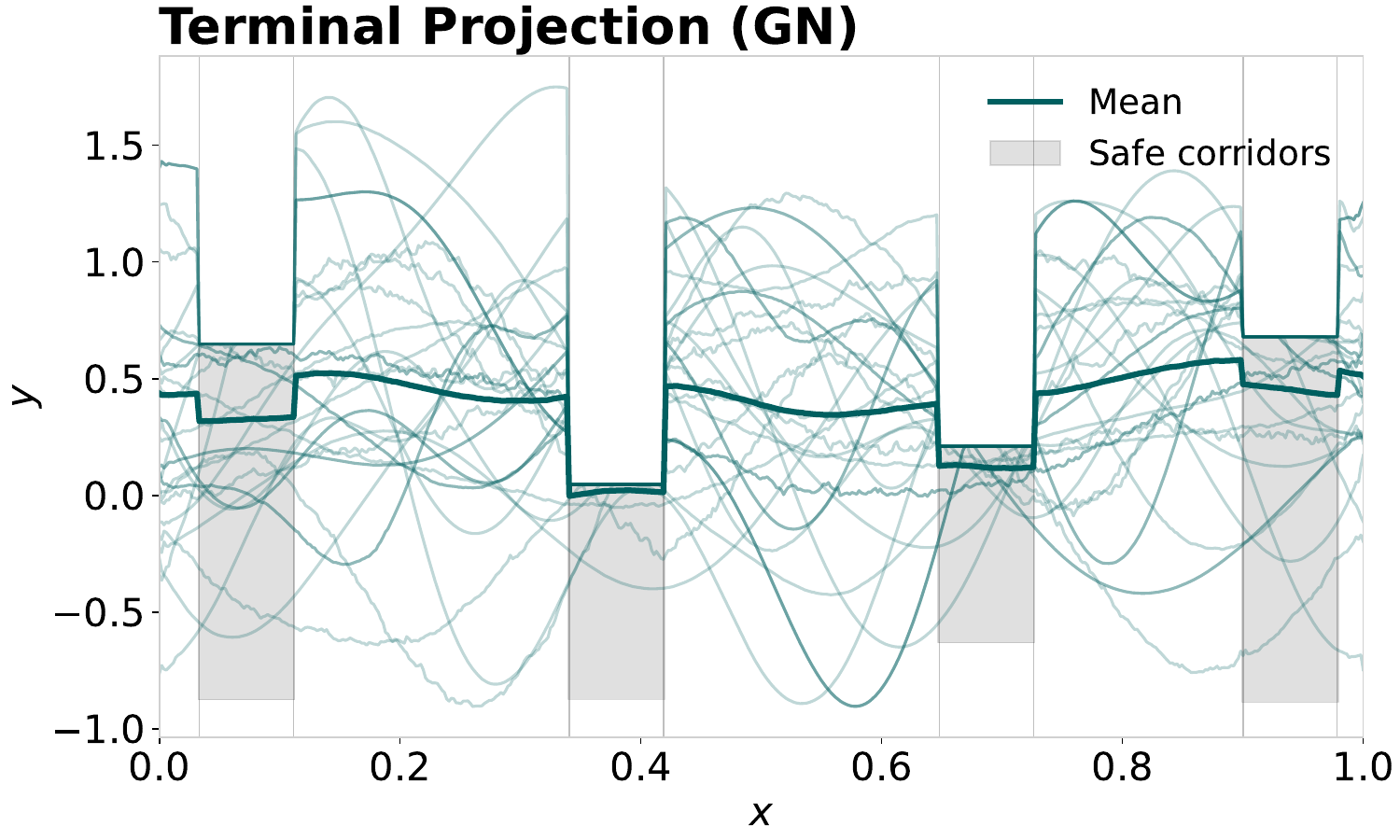}
\caption{\textbf{Qualitative comparison of the generated trajectories using different methods.} Thin lines display \(32\) individual samples for clarity. \textbf{Top row:} Samples generated by the vanilla and GD methods. \textbf{Bottom row:} Samples from the TOCFlow, approximated GN and terminal projection methods. The grey shaded regions denote the permissible safety corridors. The vanilla method frequently violates the boundaries. Both GD and TOCFlow successfully steer the trajectory into the valid region while preserving the geometric regularity (smoothness) of the reference distribution. In contrast, the projection-based methods (GN and terminal projection), while achieving strict constraint satisfaction, exhibit non-physical artifacts characterized by sharp kinks at the boundaries where the trajectory is ``snapped'' to the safety corridors.}
\label{fig:trajectory_visualization}
\end{figure}

Crucially, the comparison with projection-based methods highlights the importance of the optimal control formulation. While approximated Gauss--Newton (GN) and terminal projection (bottom row) achieve zero constraint violation, they introduce non-physical artifacts. As seen in~\cref{fig:trajectory_visualization}, these methods produce paths with sharp kinks and discontinuities at the constraint boundaries, effectively snapping the robot to the wall orthogonal to the Euclidean metric rather than respecting the smooth geometry of the reference distribution.

Quantitative results are summarized in~\cref{fig:trajectory_ablation}. The violin plots of the log-terminal cost \(\log_{10} H(\bm{f})\) (left) confirm that TOCFlow achieves a constraint satisfaction orders of magnitude better than the GD baseline, approaching the numerical precision of the projection methods. We also investigate the sensitivity of the guidance methods to the accuracy of the gradient estimation. We perform two ablation studies on the effect of the number of explicit Euler lookahead steps \(k\) used to estimate \(\Phi_{t\to1}^{\bm b^\star}(\bm x_t)\)), as described in~\Cref{subsec:numerical_implementation_and_baselines}. As shown in the center and right panels of~\cref{fig:trajectory_ablation}, increasing \(k\) reduces the (geometric) mean terminal cost for the GD method. As for TOCFlow, it achieves a terminal cost in the range of \(10^{-12}\) even when \(k=1\), significantly outperforming the GD baseline even when the latter used \(k=8\) steps (reaching only approximately \(10^{-9}\)).

\begin{figure}[!ht]
\hfil
\includegraphics[width=0.45\linewidth]{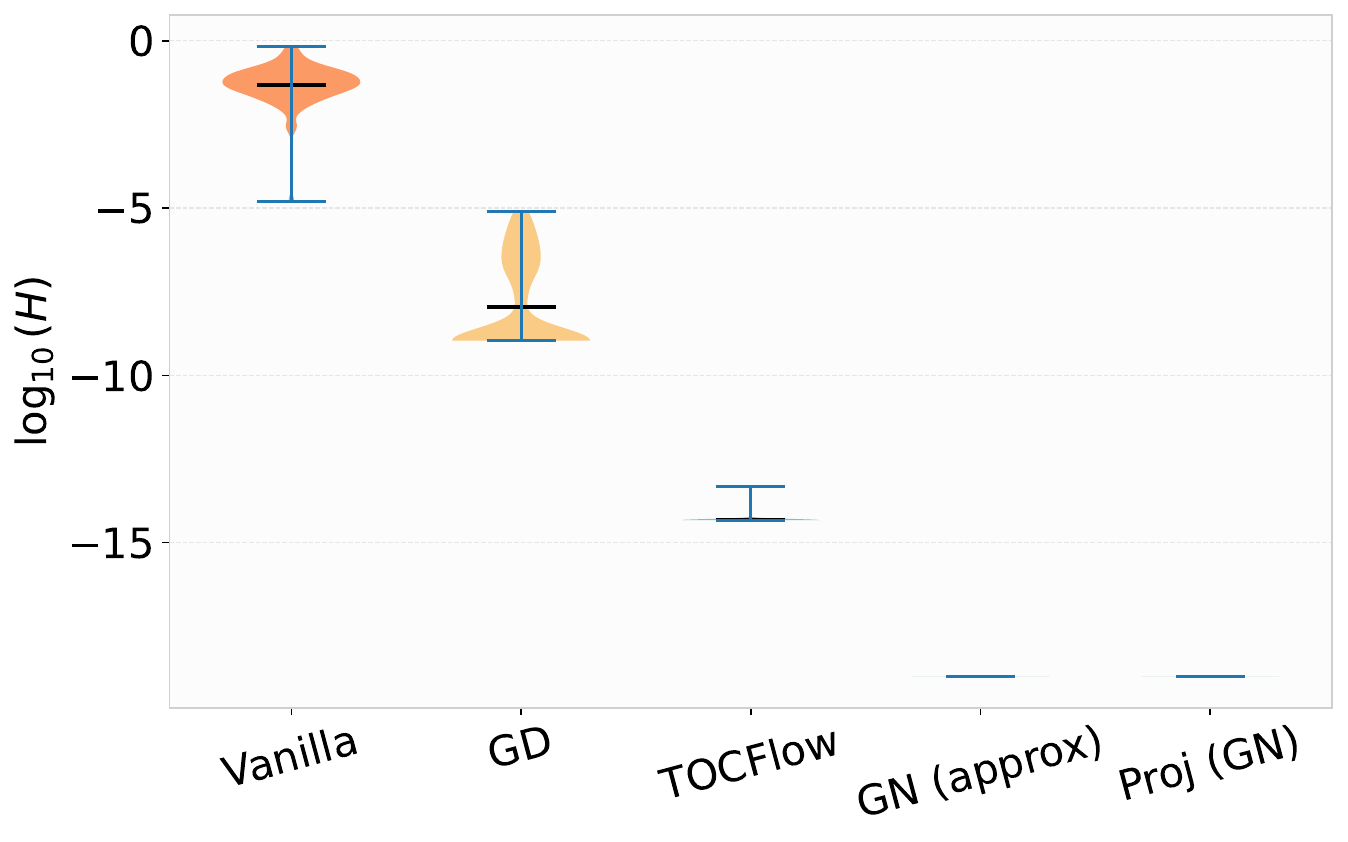}\par\hfil
\includegraphics[width=0.33\linewidth]{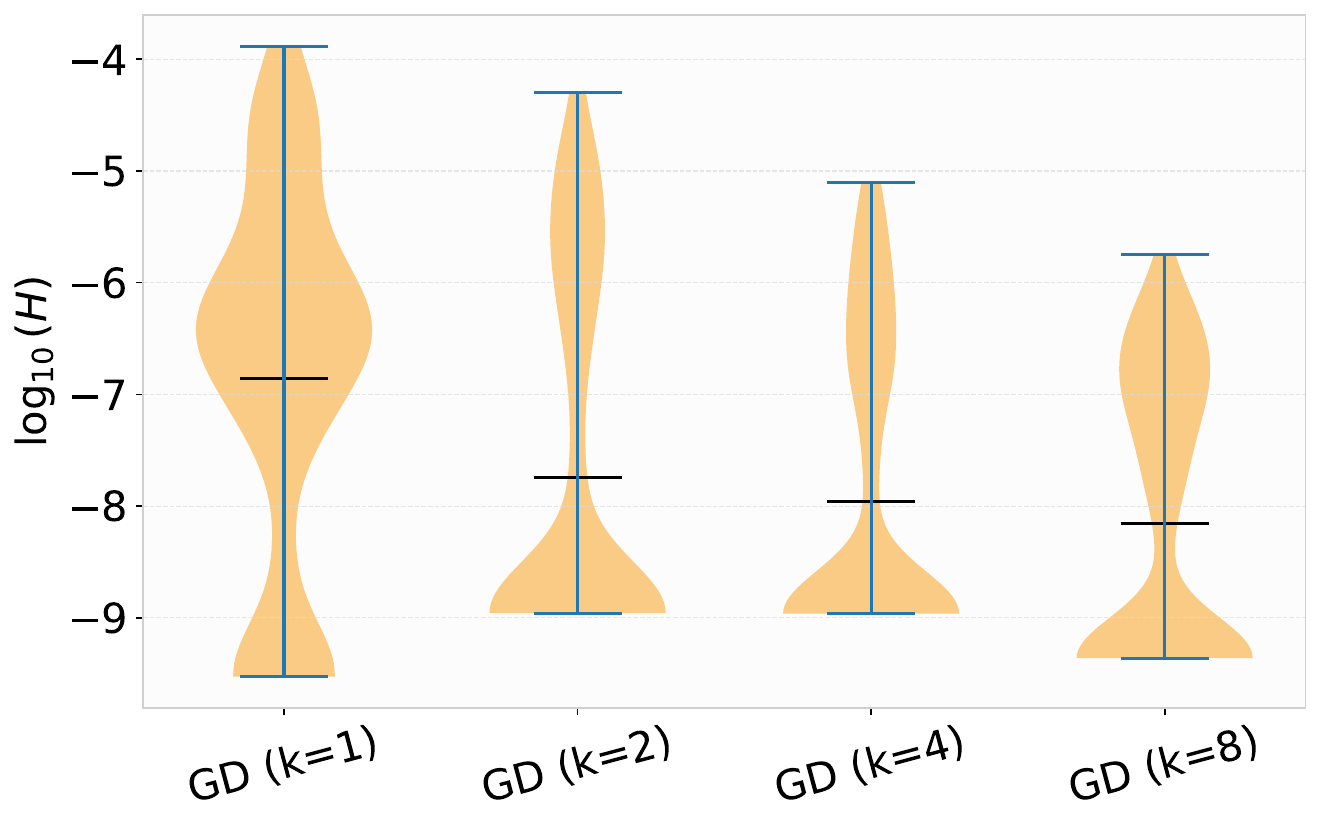}
\includegraphics[width=0.33\linewidth]{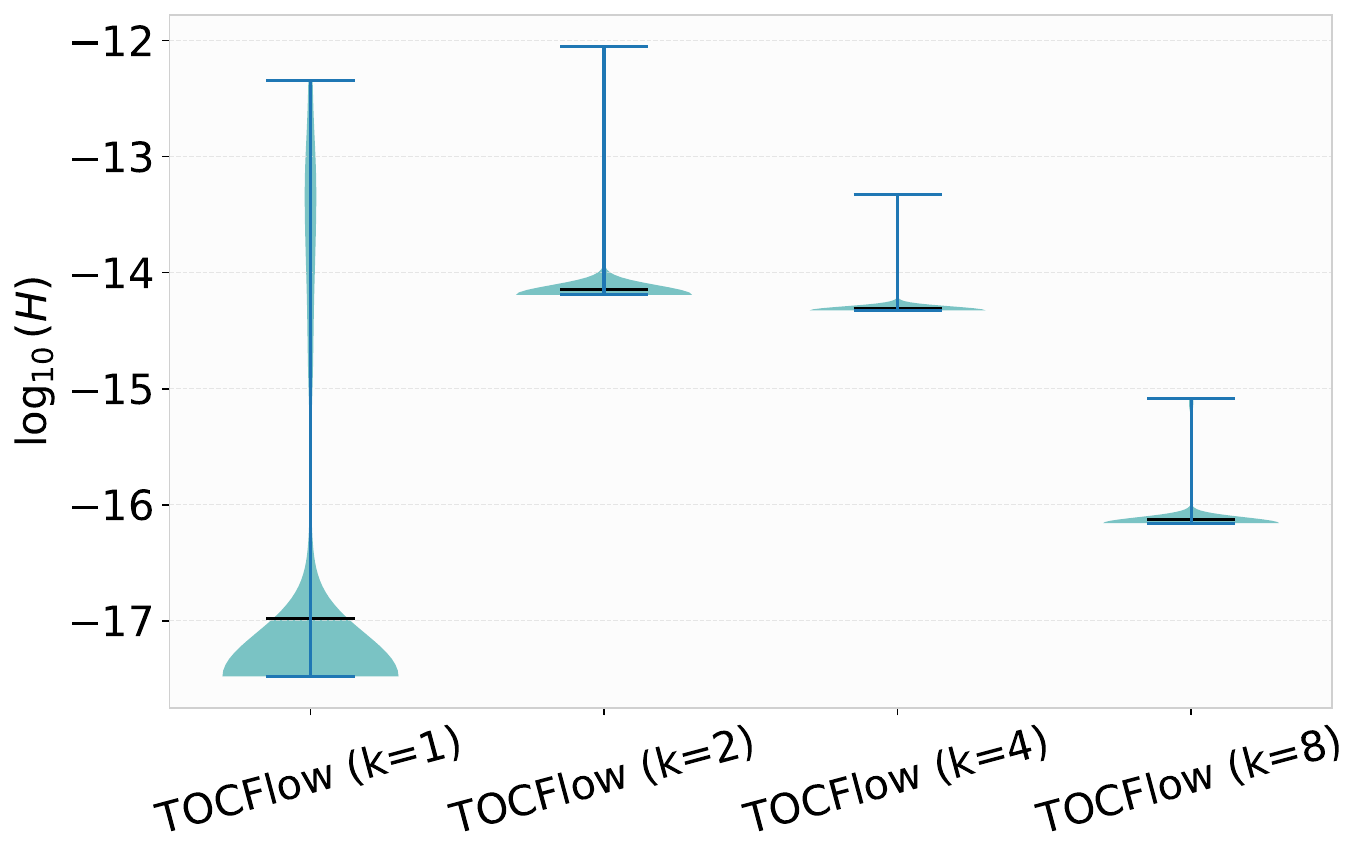}
\caption{\textbf{Quantitative evaluation of constraint violation and hyperparameter sensitivity.} \textbf{Top:} Violin plots of the terminal cost \(H(\bm{f})\) (log scale) for \(512\) generated samples across all methods. TOCFlow achieves significantly lower violations than the vanilla and GD baselines, matching the performance of GN and terminal projection. \textbf{Bottom left/right:} Ablation of the number of lookahead explicit Euler steps \(k\) for the GD and TOCFlow methods. Notably, TOCFlow achieves a terminal cost of approximately \(10^{-12}\) even with a single step (\(k=1\)), significantly outperforming the GD baseline even when the latter uses \(k=8\) steps (approximately \(10^{-9}\)).}
\label{fig:trajectory_ablation}
\end{figure}

\subsection{Global spectral constraints: Turbulence snapshot}
\label{subsec:global_spectral_constraints_turbulence_snapshot}

Finally, we consider the generation of turbulence snapshots, a task critical for applications in climate modeling and aerodynamic design. In this setting, the objective is to synthesize physically realistic snapshots of \(2\)-dimensional turbulence that adhere to the statistical laws of fluid dynamics. Unlike the previous examples where constraints were defined locally (avoiding obstacles or satisfying PDEs at grid points), here we enforce a \emph{global spectral constraint}. The generated fields must reproduce the Kolmogorov's power-law scaling~\citep{kolmogorov1941local}. This requires the control to act coherently across all spatial scales, coupling every pixel in the domain through the Fourier transform.

\paragraph{Dataset.}
We generate a dataset of forced \(2\)-dimensional turbulence using a high-resolution pseudo-spectral solver~\citep{whittaker2024turbulence}. The generating script is publicly available in the~\href{https://github.com/timwhittaker/TurbulenceDiagnostics/tree/main}{Turbulence Dataset} repository~\citep{whittaker2024turbulence}. The system evolves the vorticity field \(\bm{\omega}(\bm{x}, t) = \nabla \times \bm{u}(\bm{x}, t)\) according to the incompressible Navier--Stokes equations with linear drag and stochastic forcing:
\begin{equation}
\partial_t \bm{\omega} + (\bm{u} \cdot \nabla) \bm{\omega} = \nu \nabla^2 \bm{\omega} - \alpha \bm{\omega} + \bm{f},
\end{equation}
where \(\nu = 10^{-4}\) is the kinematic viscosity and \(\alpha = 5 \times 10^{-6}\) is a linear drag coefficient representing large-scale damping. The system is defined on a doubly periodic domain \(\Omega = [0, 2\pi]^2\) discretized on a \(512 \times 512\) grid. The forcing term \(\hat{f}(\bm{k}, t)\) is defined in the spectral domain as a time-correlated stochastic process acting on the wavenumber band \(k \in [10, 15]\). Its discrete time evolution follows the autoregressive rule
\begin{equation}
\hat{f}_{t+\Delta t}(\bm{k}) = r \hat{f}_t(\bm{k}) + \sqrt{1-r^2} A \xi_t(\bm{k}) \bm{1}_{10 \le \|\bm{k}\| \le 15},
\end{equation}
where \(\bm{k}=(k_x, k_y)\) is the wave vector, \(r=0.9\) sets the temporal correlation, \(A\) is a fixed amplitude constant, and \(\xi_t(\bm{k}) = \exp(2\pi i \theta_{\bm{k}})\) represents random phase noise with \(\theta_{\bm{k}} \sim \mathcal{U}[0, 1]\). The simulation is time-stepped using a third-order Adams--Bashforth scheme for the nonlinear advection terms and a Crank--Nicolson scheme for the linear terms, with a time step of \(\Delta t = 10^{-3}\). We collect \(500\) snapshots from the statistically stationary regime (time \(t \in [51, 550]\)). These snapshots are subsequently downsampled to \(256 \times 256\), and are standardized instance-wise and clipped to the range \([-1, 1]\) before training.

\paragraph{Pre-trained model architecture.}
The time-dependent velocity field \(\bm{b}^\star(\bm x, t)\) is parameterized using a 2D U-Net~\citep{ronneberger2015u} implemented via the Diffusers library~\citep{von2022diffusers}. The network operates on single-channel vorticity fields and features a deep, symmetric encoder-decoder structure with seven resolution levels. The channel width progressively doubles, following the sequence \((128, 128, 256, 256, 512, 512, 1024)\), resulting in a highly compressed bottleneck representation. Each resolution level is processed by a stack of \(2\) ResNet-style residual blocks, which employ Group Normalization (\(32\) groups) and SiLU activations. Continuous time conditioning is achieved by mapping the scalar timestep \(t\) to a high-dimensional embedding via Gaussian Fourier projections followed by a two-layer MLP. This embedding is injected into every residual block to modulate the features via scale and shift parameters. To capture long-range spatial dependencies while maintaining computational efficiency, self-attention mechanisms are selectively integrated only at the sixth resolution level (corresponding to the \(512\)-channel feature maps), where the spatial dimension is sufficiently reduced (\(8 \times 8\)) to render the quadratic attention cost negligible.

\paragraph{Training details.}
The model is trained using the Flow Matching objective with an Optimal Transport (OT) conditional vector field~\citep{lipman2023flow,liu2023flow}. We optimize the network parameters using the AdamW~\citep{loshchilov2019decoupled} optimizer with a peak learning rate of \(1 \times 10^{-4}\), scheduled via a cosine decay with \(500\) warmup steps. The training is performed for \(30{,}000\) iterations using a batch size of \(16\). To ensure optimization stability, we apply gradient clipping with a maximum norm of \(1.0\).

\paragraph{Constraint formulation.}
The hallmark of developed turbulence is the distribution of kinetic energy across length scales, known as the energy spectrum \(E(k)\). We define the constraint to ensure that the generated samples exhibit the correct inertial sub-range scaling predicted by Kolmogorov's 1941 theory~\citep{kolmogorov1941local}, i.e., \(E(k) \propto k^{-5/3}\). We define the scalar constraint residual \(h(\bm \omega)\) as the mean squared deviation from this power law within the inertial band \(\mathcal{K} = \{k \in \mathbb{Z} \colon 2 \le k \le 9\}\). Specifically, we compute the radially averaged energy spectrum \(E(k)\) via the discrete Fourier transform and define the residual as the variance of the compensated spectrum:
\begin{equation}
h(\bm \omega) = \frac{1}{|\mathcal{K}|} \sum_{k \in \mathcal{K}} \biggl( \Bigl( \log E(k) + \frac{5}{3} \log k \Bigr) - \mu_{\mathcal{K}} \biggr)^2,
\end{equation}
where \(\mu_{\mathcal{K}}\) is the arithmetic mean of the compensated spectral term over the inertial band, i.e.,
\begin{equation}
\mu_{\mathcal{K}} = \frac{1}{|\mathcal{K}|} \sum_{j \in \mathcal{K}} \Bigl( \log E(j) + \frac{5}{3} \log j \Bigr).
\end{equation}
By subtracting the mean \(\mu_{\mathcal{K}}\), this formulation strictly enforces the slope (scaling exponent) of the spectrum while being invariant to the total energy magnitude. The terminal cost is defined as the square of this residual, i.e., \(H(\bm \omega) = \frac{1}{2} (h(\bm \omega))^2\).

We remark that this is not the only way to enforce Kolmogorov scaling. For example, one may fit a line \(\log E(k)\approx \alpha+\beta \log k\) over \(k\in\mathcal K\) and penalize \((\beta+5/3)^2\), which enforces the exponent directly. We instead use the variance of the compensated log-spectrum because it is simpler to compute and, by subtracting the mean, it enforces the slope while remaining invariant to the intercept, avoiding an unintended coupling to the total kinetic energy.

\paragraph{Implementation details.}
We detail here the specific hyperparameters for the turbulence snapshot task, complementing the general settings described in~\Cref{subsec:numerical_implementation_and_baselines}. We generate samples using Forward Euler method with \(N=200\) steps. For GD and TOCFlow, we estimate the gradient of the composition \(\nabla_{\bm{x}} (H \circ \Phi_{t\to 1}^{\bm{b}^\star})\) using a \(k=4\) step forward Euler lookahead. For TOCFlow, we employ a strict, rapidly decaying weight schedule with \(\lambda_0 = 10^{10}\) and \(\gamma = 10.0\), reflecting the high sensitivity of the spectral constraint. For GD, we employ a constant step size of \(\eta=1000\). For the terminal projection method, we project samples onto the constraint manifold using a damped Gauss--Newton algorithm with a maximum budget of \(1{,}000\) iterations. The linearized subproblem at each step is solved inexactly via Conjugate Gradient with \(20\) inner iterations. All methods are implemented on top of the public~\href{https://github.com/timwhittaker/TurbulenceDiagnostics/tree/main}{Turbulence Diagnostics} codebase~\citep{whittaker2024turbulence}, which we extend with our GD, GN, and TOCFlow modules.

\paragraph{Results.} We compare the performance of TOCFlow against the baselines in terms of constraint satisfaction and sample quality. \Cref{fig:turbulence_sample} presents representative vorticity field snapshots. While the vanilla method produces visually plausible turbulence, the spectral analysis in~\cref{fig:turbulence_spectrum} reveals critical distinctions. We observe that the ground truth dataset spectrum (top left) does not strictly adhere to the theoretical \(k^{-5/3}\) slope. This deviation arises from the finite Reynolds number and the specific localized forcing mechanism used in the underlying Direct Numerical Simulation (DNS). The vanilla method (top center), being purely data-driven, inherits this bias, producing a spectrum that mirrors the dataset's deviation. The GD method (top right) improves adherence to the theoretical slope but exhibits large variance at high frequencies. In contrast, TOCFlow and projection-based methods (bottom row) successfully enforce the constraints. Notably, TOCFlow (bottom left) achieves the most precise alignment with the target scaling law, effectively correcting the physical bias present in the training distribution to satisfy the idealized Kolmogorov theory.

\begin{figure}[!ht]
\includegraphics[width=0.195\linewidth]{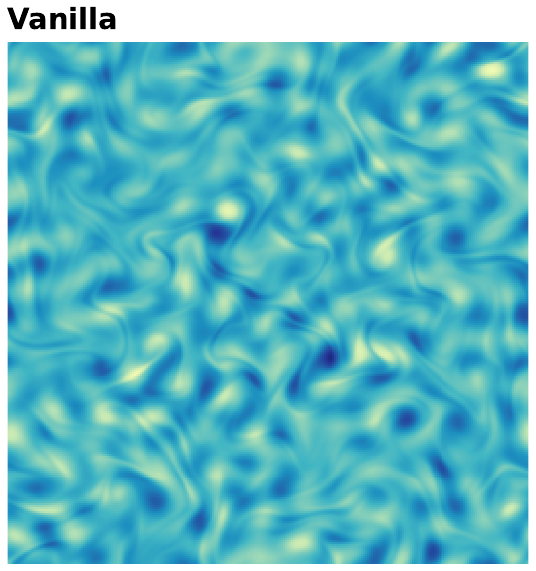}
\includegraphics[width=0.195\linewidth]{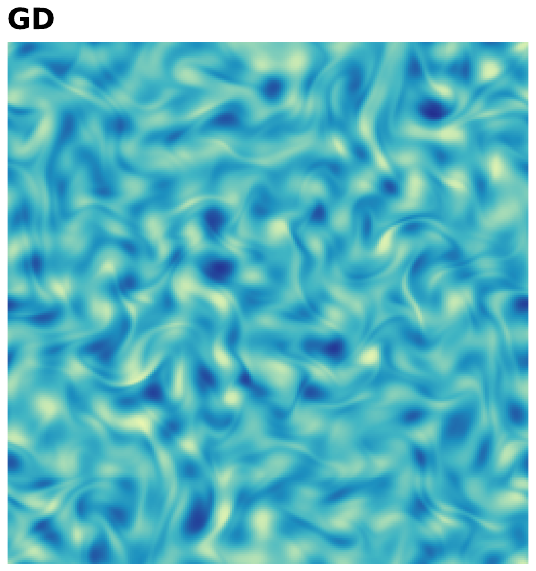}
\includegraphics[width=0.195\linewidth]{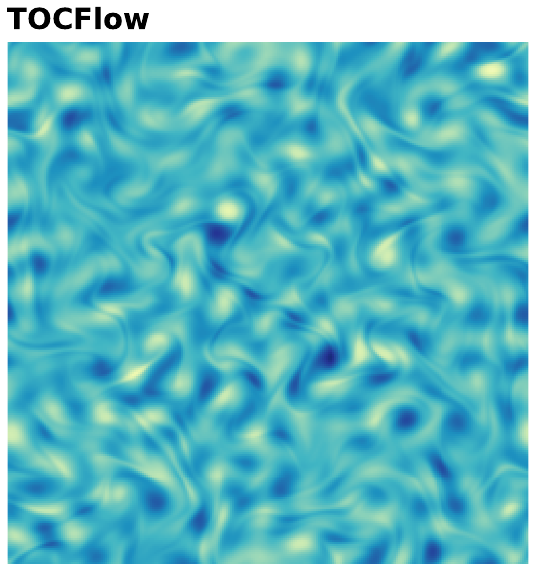}
\includegraphics[width=0.195\linewidth]{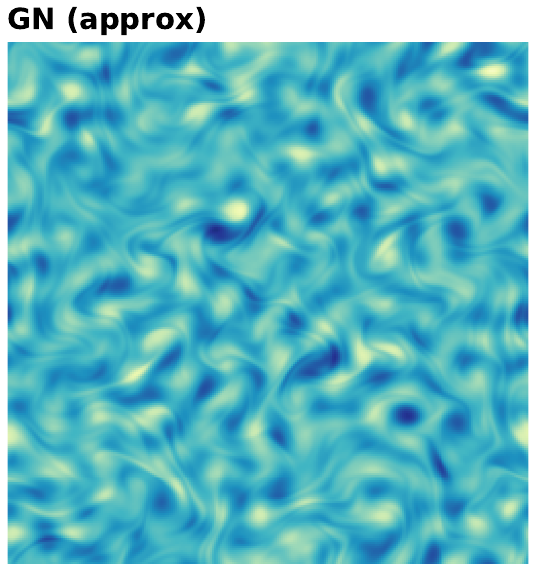}
\includegraphics[width=0.195\linewidth]{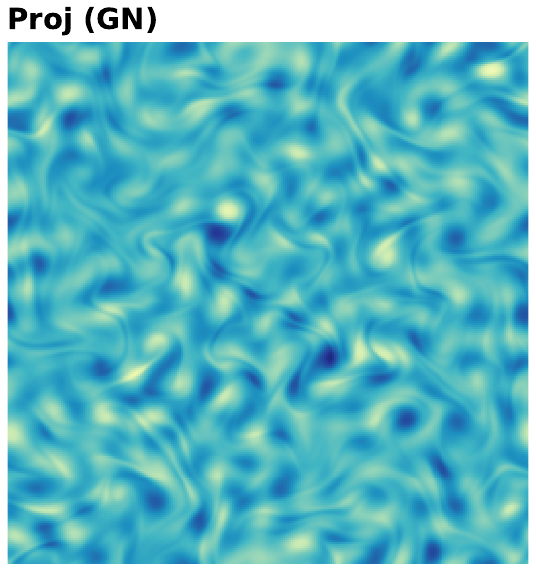}
\caption{\textbf{Qualitative comparison of generated turbulence snapshots using different methods.} From left to right: vanilla, GD, TOCFlow, approximated GN, and terminal projection.}
\label{fig:turbulence_sample}
\end{figure}

\begin{figure}[!ht]
\includegraphics[width=0.33\linewidth]{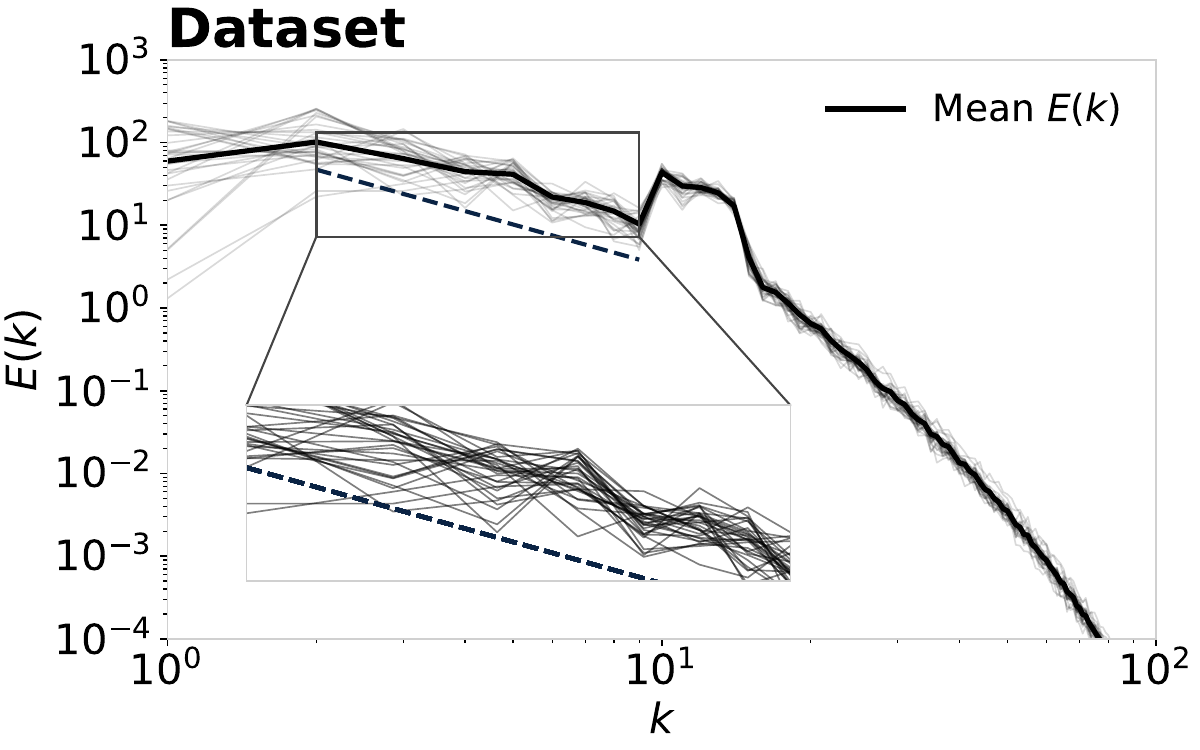}
\includegraphics[width=0.33\linewidth]{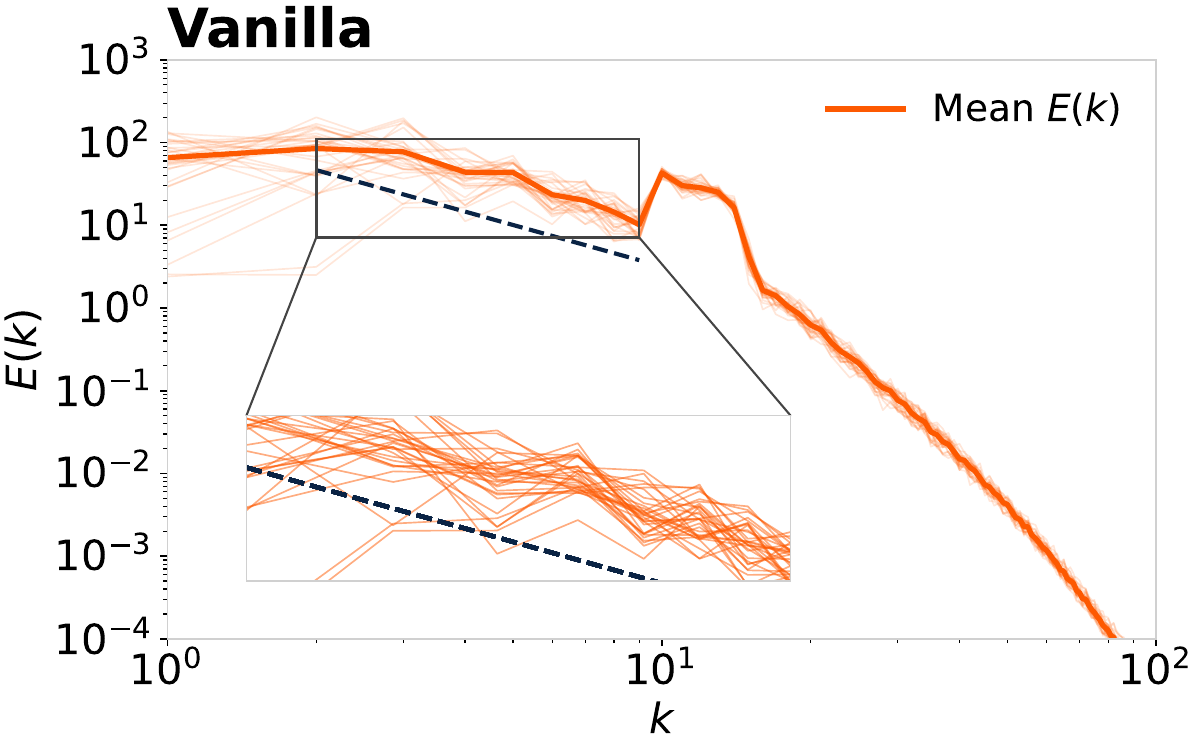}
\includegraphics[width=0.33\linewidth]{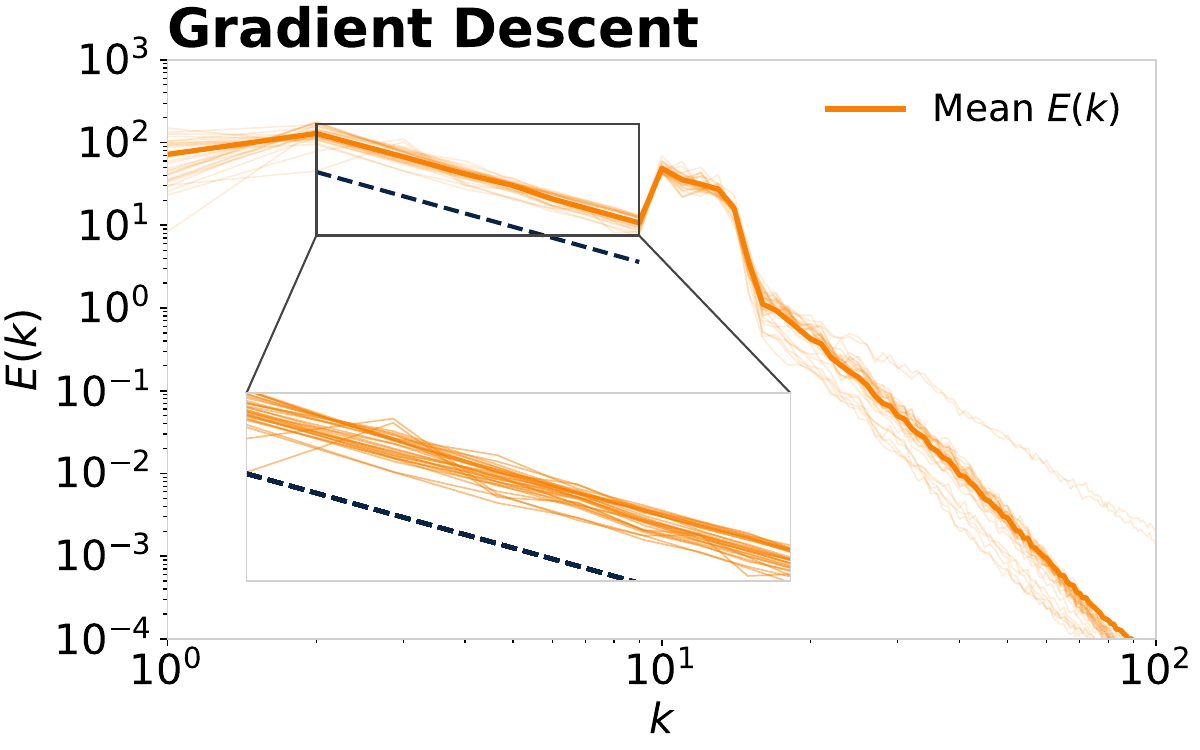}\par
\includegraphics[width=0.33\linewidth]{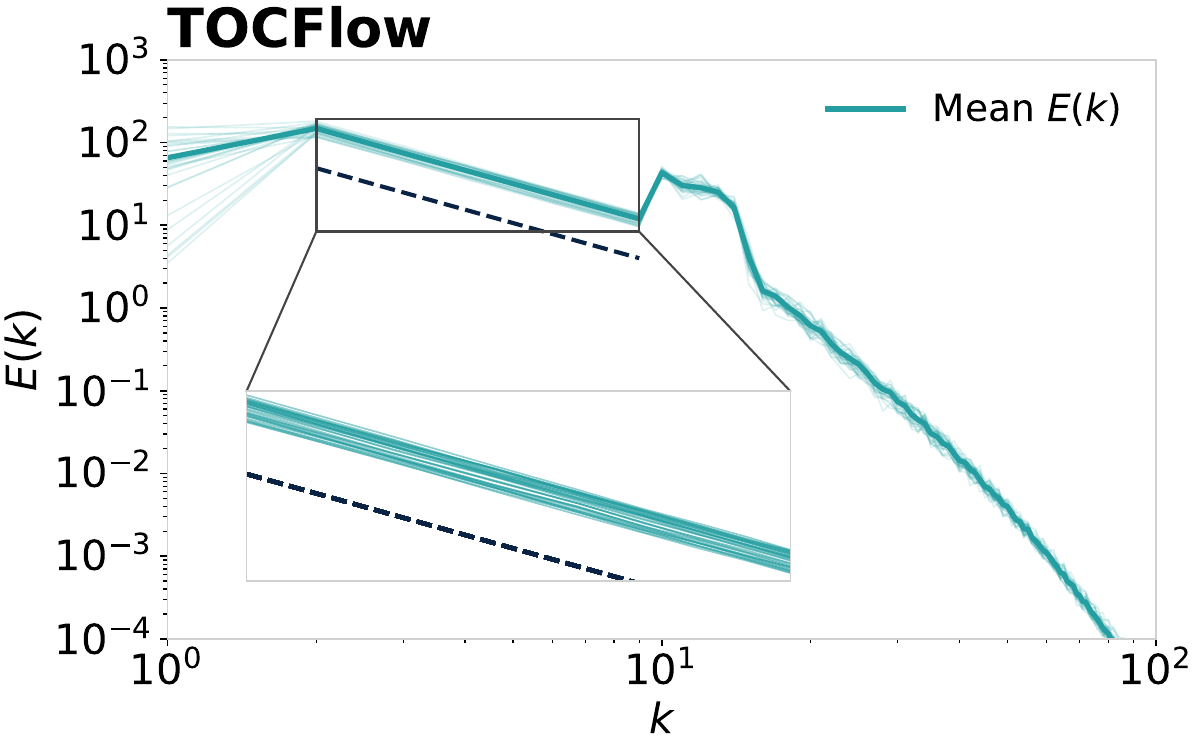}
\includegraphics[width=0.33\linewidth]{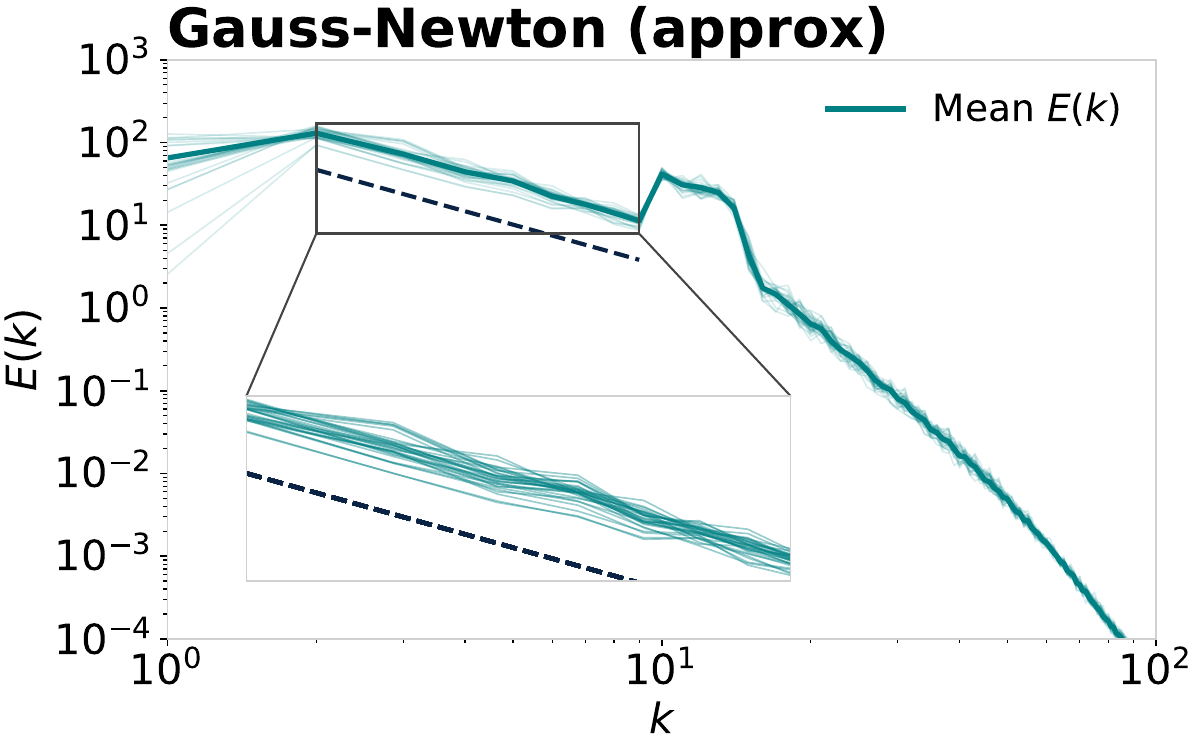}
\includegraphics[width=0.33\linewidth]{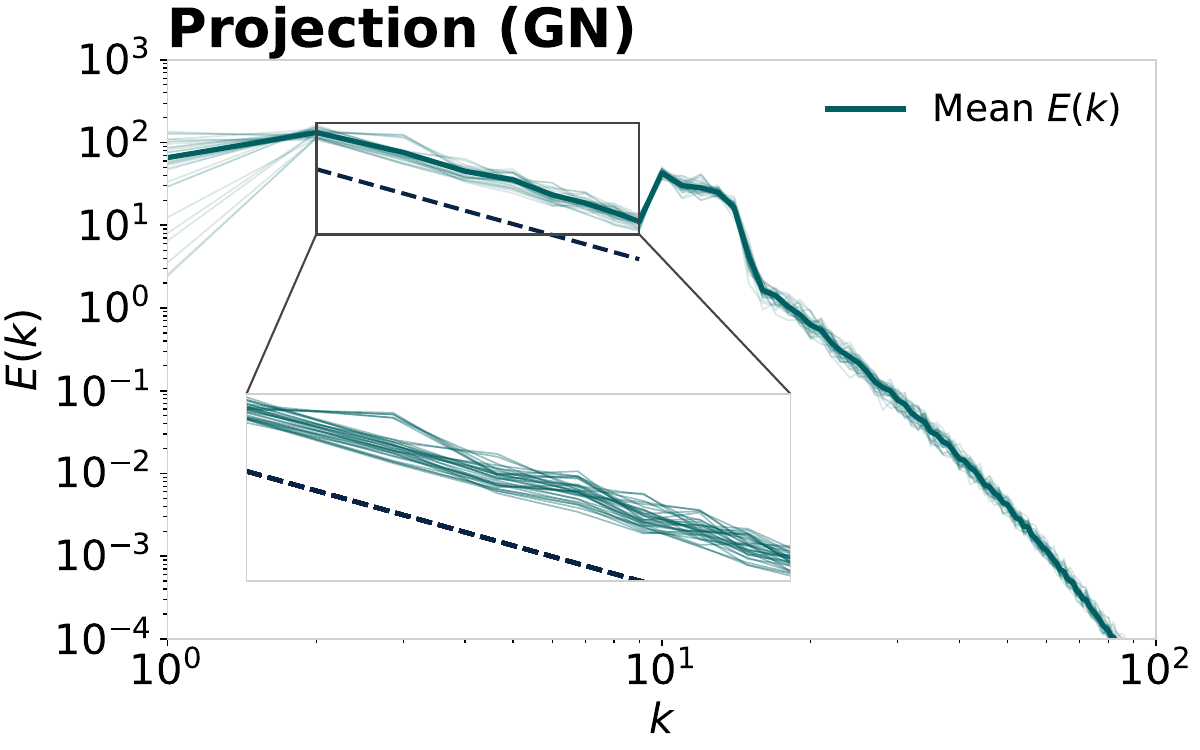}
\caption{\textbf{Qualitative comparison of the kinetic energy spectrum \(E(k)\).} Thin lines display \(32\) individual samples for clarity. \textbf{Top row:} kinetic energy spectrum for the dataset, vanilla, and GD methods. \textbf{Bottom row:} TOCFlow, approximated GN, and terminal projection methods. The dashed lines indicate the inertial range where Kolmogorov's theory predicts a \(k^{-5/3}\) power-law scaling. Both the dataset (top left) and the vanilla method (top center) exhibit similar deviations from the ideal slope due to finite simulation effects. The GD method (top right) enforces the slope better but introduces high-frequency variance. In contrast, TOCFlow (bottom left) aligns almost perfectly with the theoretical scaling law, outperforming the baseline methods.}
\label{fig:turbulence_spectrum}
\end{figure}

Quantitative results are summarized in~\cref{fig:turbulence_ablation}. The violin plots of the spectral residual cost \(\log_{10}H(\bm \omega)\) (top panel) demonstrate the superior constraint satisfaction of TOCFlow. The error distribution for TOCFlow is tightly concentrated at a magnitude significantly lower than both the baseline methods. We perform three ablation studies on the effect of the number of explicit Euler lookahead steps \(k\) used to estimate \(\Phi_{t\to1}^{\bm b^\star}(\bm x_t)\), as described in~\Cref{subsec:numerical_implementation_and_baselines}. For the GD and approximated GN methods, increasing \(k\) yields diminishing returns or slight degradation. In contrast, TOCFlow benefits substantially from a larger lookahead. For \(k>1\), the error drops sharply and remains stable. Finally, while the terminal projection method (bottom right) naturally improves as the iteration steps increase from \(1\) to \(1{,}000\), it fails to match the precision of TOCFlow even at high computational costs.

\begin{figure}[!ht]
\hfil
\includegraphics[width=0.45\linewidth]{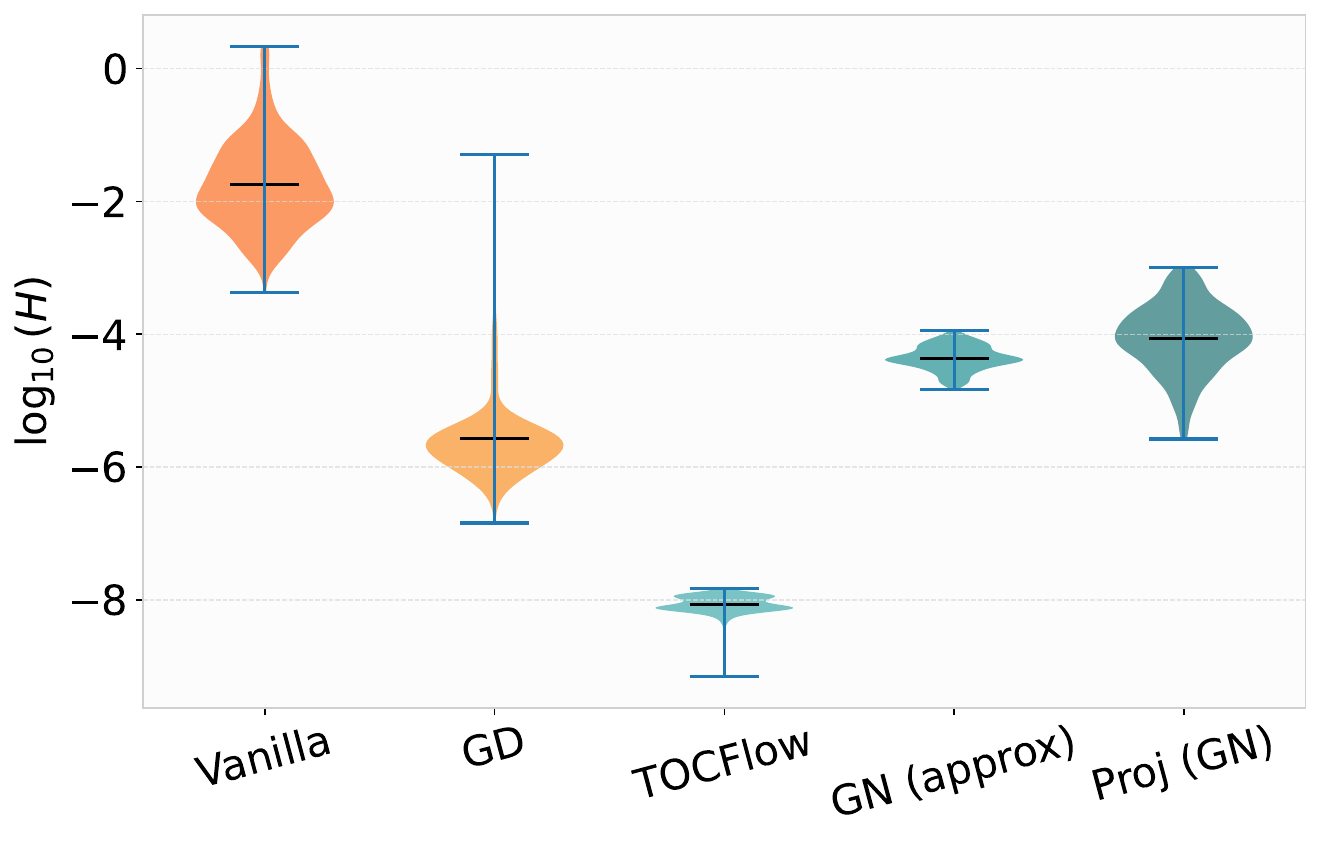}\par\hfil
\includegraphics[width=0.33\linewidth]{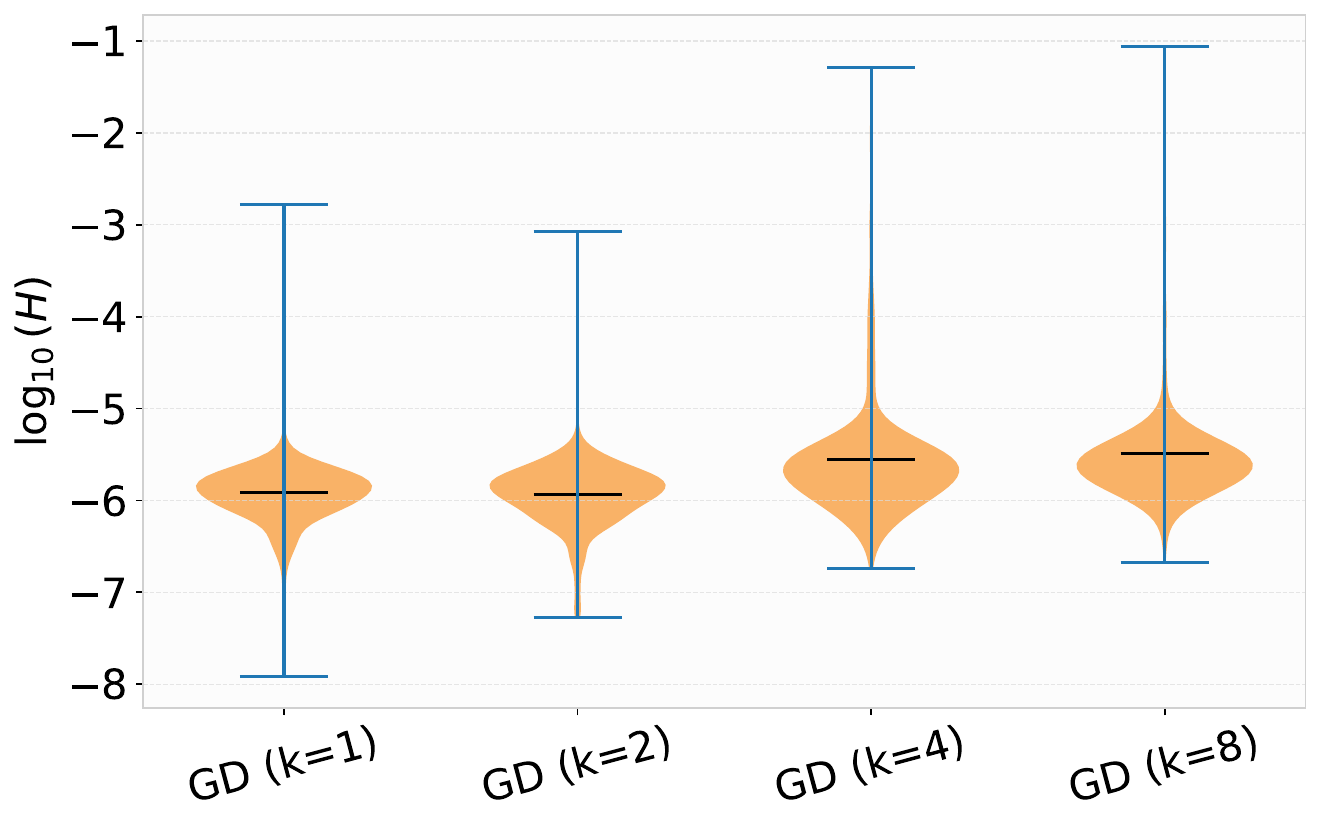}
\includegraphics[width=0.32\linewidth]{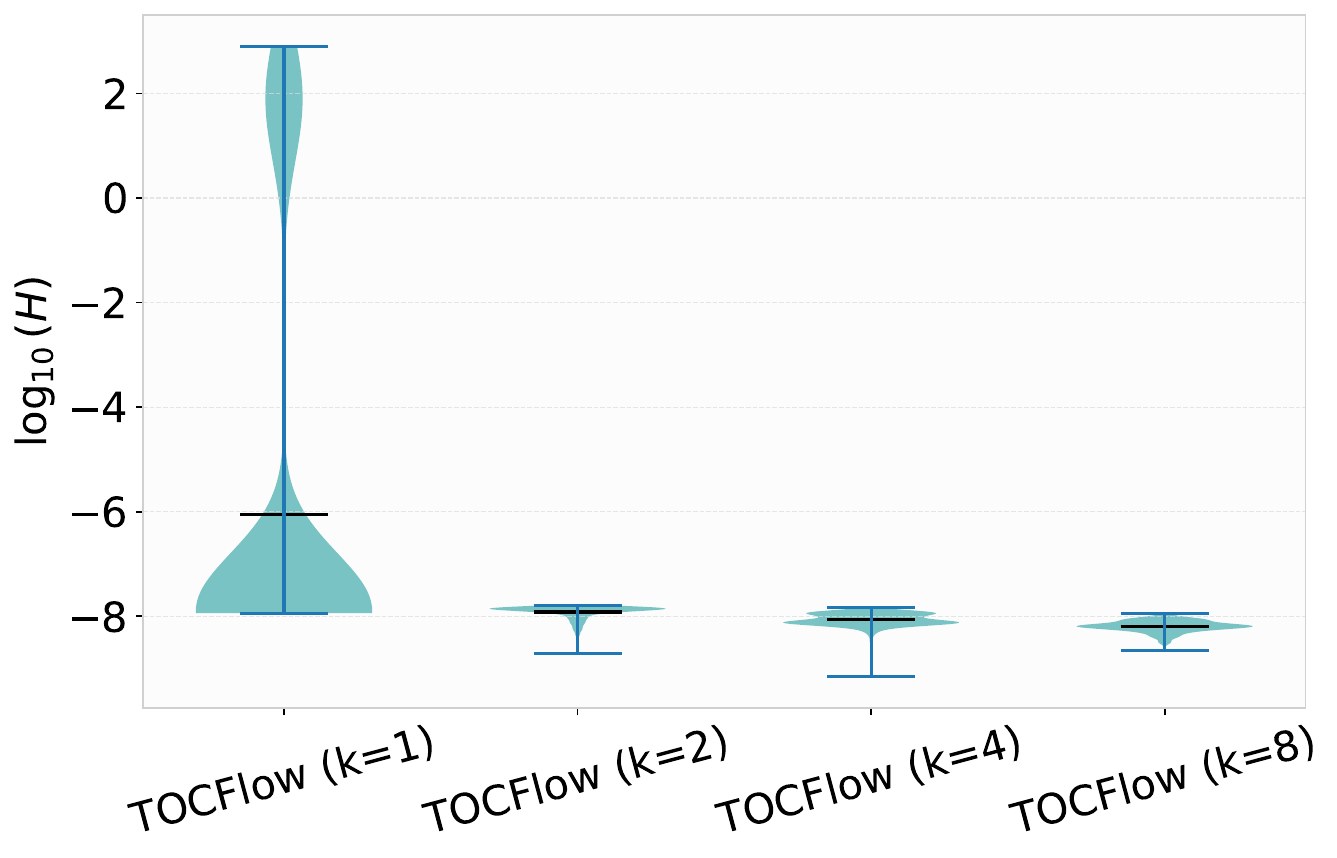}\par\hfil
\includegraphics[width=0.33\linewidth]{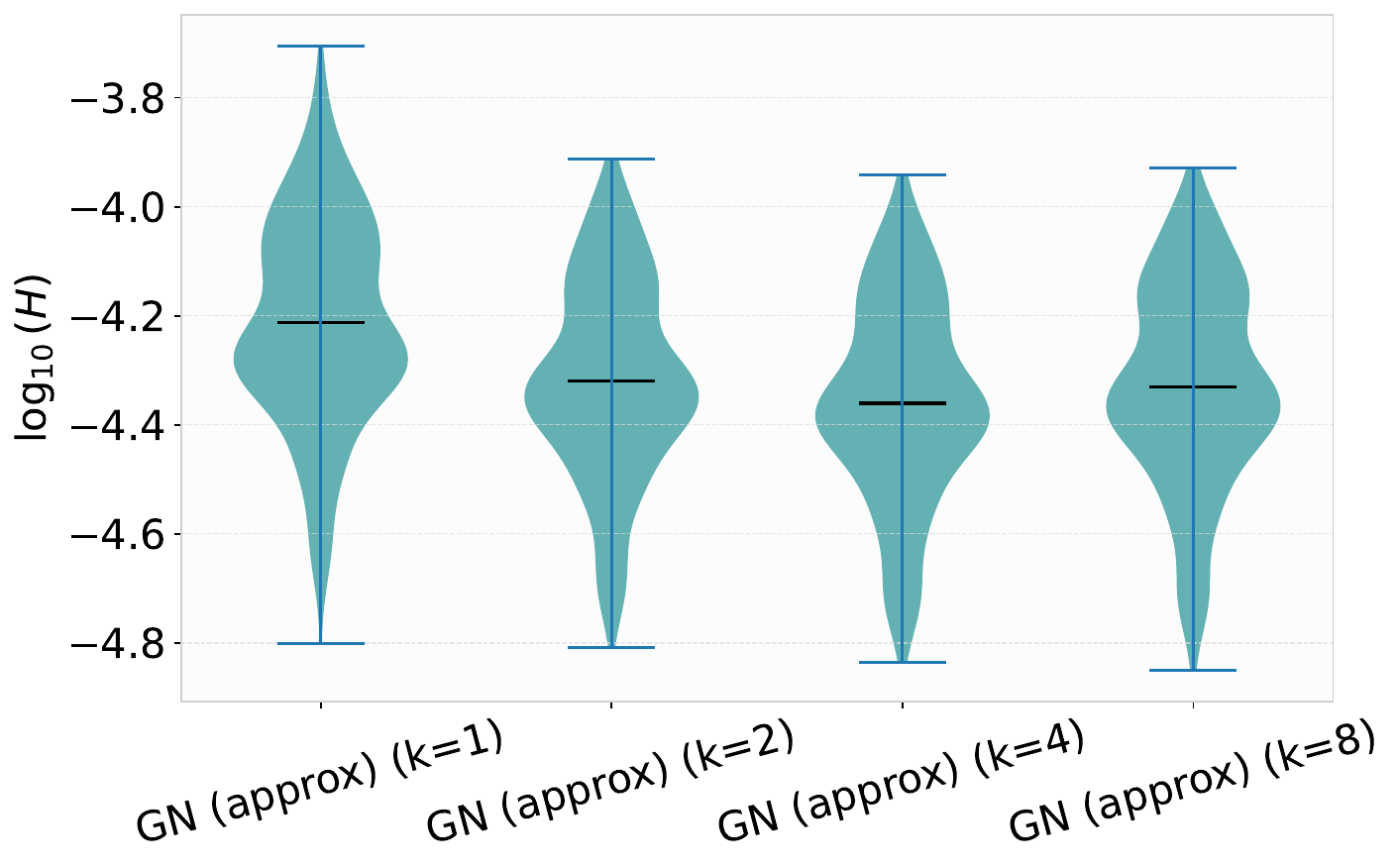}
\includegraphics[width=0.32\linewidth]{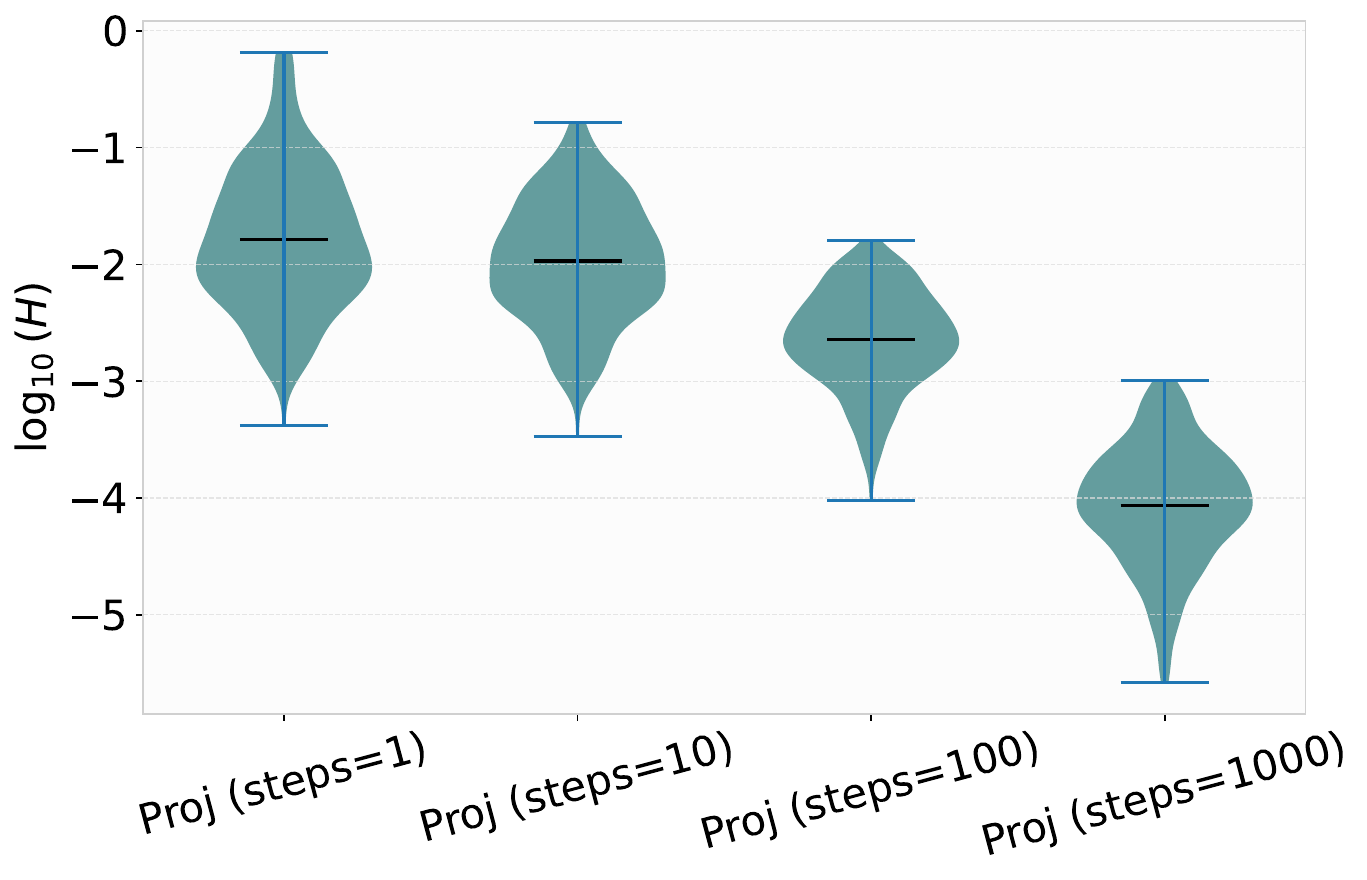}
\caption{\textbf{Quantitative evaluation of constraint violation and hyperparameter sensitivity.} \textbf{Top:} Violin plots of the terminal cost \(H(\bm \omega)\) (log scale) for \(512\) generated samples across all methods. TOCFlow achieves significantly lower violations and variance compared to all baselines. \textbf{Middle and bottom:} Ablation of the number of lookahead explicit Euler steps \(k\) for the GD, TOCFlow, and approximated GN methods. While GD degrades with larger \(k\) and approximated GN shows negligible gain, TOCFlow exhibits a sharp performance gain for \(k>1\) before stabilizing. The bottom right panel shows the ablation of the number of iterations for terminal projection. Although the error decreases with more iterations, it remains orders of magnitude higher than that of TOCFlow.}
\label{fig:turbulence_ablation}
\end{figure}

\bibliography{references}
\bibliographystyle{plain}

\end{document}